\documentclass{article}

\PassOptionsToPackage{numbers}{natbib}

\usepackage[final]{neurips_2020}

\usepackage[utf8]{inputenc} 
\usepackage[T1]{fontenc}    
\usepackage{hyperref}       
\usepackage{url}            
\usepackage{booktabs}       
\usepackage{amsfonts}       
\usepackage{nicefrac}       
\usepackage{microtype}      

\usepackage{microtype}
\usepackage{graphicx}
\usepackage{comment}
\newcommand\numberthis{\addtocounter{equation}{1}\tag{\theequation}}
\usepackage{mathrsfs}
\usepackage{multirow}

\usepackage{amsmath,amsfonts,bm}

\def\1{\bm{1}}

\DeclareMathAlphabet{\mathsfit}{\encodingdefault}{\sfdefault}{m}{sl}
\SetMathAlphabet{\mathsfit}{bold}{\encodingdefault}{\sfdefault}{bx}{n}

\def\sR{{\mathbb{R}}}

\newcommand{\E}{\mathbb{E}}

\newcommand{\R}{\mathbb{R}}

\DeclareMathOperator*{\argmax}{arg\,max}
\DeclareMathOperator*{\argmin}{arg\,min}

\makeatletter
\newcommand*{\newreptext}[1]{%
  \begingroup 
    \csname @safe@actives@true\endcsname
  \expandafter\endgroup
  \expandafter\newcommand\csname reptext@#1\endcsname
}

\newcommand*{\reptext}[1]{%
  \begingroup
  \csname @safe@actives@true\endcsname 
  \@ifundefined{reptext@#1}{%
    \@latex@error{\string\reptext{#1} is undefined}\@ehc
    \endgroup
    \textbf{??}%
  }{%
    \endgroup
    \@nameuse{reptext@#1}%
  }%
}
\makeatother
\usepackage{subcaption}
\usepackage{adjustbox}
\usepackage{booktabs} 
\usepackage{amsthm}
\usepackage{algorithm}
\usepackage{algorithmic}
\usepackage{enumitem}

\usepackage{wrapfig}
\usepackage{xcolor}
\usepackage{colortbl}
\usepackage{makecell}

\usepackage{titlesec}
\usepackage{mathtools}

\usepackage{hyperref}

\newtheorem{lemma}{Lemma}
\newtheorem{theorem}{Theorem}
\newtheorem{assumption}{Assumption}
\newtheorem{definition}{Definition}

\makeatletter
\def\mathcolor#1#{\@mathcolor{#1}}
\def\@mathcolor#1#2#3{%
  \protect\leavevmode
  \begingroup
    \color#1{#2}#3%
  \endgroup
}
\makeatother

\title{Robust Deep Reinforcement Learning against Adversarial Perturbations on State Observations}

\author{%
 Huan Zhang\textsuperscript{*,1} \quad  Hongge Chen\textsuperscript{*,2} \quad Chaowei Xiao\textsuperscript{3} \\ \textbf{Bo Li\textsuperscript{4} \quad Mingyan Liu\textsuperscript{5}\quad Duane Boning\textsuperscript{2} \quad Cho-Jui Hsieh\textsuperscript{1}}\\
  \textsuperscript{1}UCLA\quad
  \textsuperscript{2} MIT\quad
  \textsuperscript{3}NVIDIA \quad
  \textsuperscript{4}UIUC \quad
  \textsuperscript{5}University of Michigan\\
  \texttt{ huan@huan-zhang.com, chenhg@mit.edu, chaoweix@nvidia.com,}\\
  \texttt{lbo@illinois.edu,mingyan@umich.edu,boning@mtl.mit.edu,chohsieh@cs.ucla.edu}\\
  \vspace*{-0.3cm}\\
  \textsuperscript{*}Huan Zhang and Hongge Chen contributed equally. \\
}

\begin{document}
\maketitle
\begin{abstract}
A deep reinforcement learning (DRL) agent observes its states through observations, which may contain natural measurement errors or adversarial noises. Since the observations deviate from the true states, they can mislead the agent into making suboptimal actions. Several works have shown this vulnerability via adversarial attacks, but existing approaches on improving the robustness of DRL under this setting have limited success and lack for theoretical principles. We show that naively applying existing techniques on improving robustness for classification tasks, like adversarial training, is ineffective for many RL tasks. We propose the state-adversarial Markov decision process (SA-MDP) to study the fundamental properties of this problem, and develop a theoretically principled policy regularization which can be applied to a large family of DRL algorithms, including proximal policy optimization (PPO), deep deterministic policy gradient (DDPG) and deep Q networks (DQN), for both discrete and continuous action control problems. We significantly improve the robustness of PPO, DDPG and DQN agents under a suite of strong white box adversarial attacks, including new attacks of our own. Additionally, we find that a robust policy noticeably improves DRL performance even without an adversary in a number of environments. Our code is available at \textcolor{blue}{\url{https://github.com/chenhongge/StateAdvDRL}}.

\end{abstract}

\setlength{\abovedisplayskip}{3.0pt plus 0.0pt minus 2.0pt}
\setlength{\belowdisplayskip}{3.0pt plus 0.0pt minus 2.0pt}
\setlength{\abovedisplayshortskip}{0.0pt plus 1.0pt}
\setlength{\belowdisplayshortskip}{2.0pt plus 0.0pt minus 1.0pt}
\setlength{\topsep}{2.0pt plus 0.0pt minus 0.0pt}

\titlespacing*{\subsection}{0pt}{2pt}{2pt}

\vspace{-0.25cm}
\section{Introduction}
\vspace{-0.6em}
With deep neural networks (DNNs) as powerful function approximators, deep reinforcement learning (DRL) has achieved great success on many complex tasks~\citep{mnih2015human,lillicrap2015continuous,schulman2015trust,silver2016mastering,gu2016continuous} and even on some safety-critical applications (e.g., autonomous driving~\citep{voyage,sallab2017deep,pan2017virtual}).
Despite achieving super-human level performance on many tasks, the existence of adversarial examples~\citep{szegedy2013intriguing} in DNNs and many successful attacks to DRL~\citep{huang2017adversarial,behzadan2017vulnerability,lin2017tactics,pattanaik2018robust,xiao2019characterizing} motivate us to study robust DRL algorithms.

When an RL agent obtains its current state via observations, the observations may contain uncertainty that naturally originates from unavoidable sensor errors or equipment inaccuracy. A policy not robust to such uncertainty can lead to catastrophic failures (e.g., the navigation setting in  Figure~\ref{fig:example_uncertainty}). 
To ensure safety under the \emph{worst case} uncertainty, we consider the adversarial setting where the state observation is adversarially perturbed from $s$ to $\nu(s)$, yet the underlying true environment state $s$ is unchanged. 
This setting is aligned with many adversarial attacks on state observations (e.g., \citep{huang2017adversarial,lin2017tactics}) and
cannot be characterized by existing tools such as partially observable Markov decision process (POMDP), because the conditional observation probabilities in POMDP cannot capture the adversarial (worst case) scenario. Studying the fundamental principles in this setting is crucial.

Before basic principles were developed, several early approaches~\citep{behzadan2017whatever,mandlekar2017adversarially,pattanaik2018robust} extended existing adversarial defenses for supervised learning,
e.g., adversarial training~\citep{kurakin2016adversarial,madry2017towards,zhang2019theoretically} to improve robustness
\begin{wrapfigure}[14]{r}{0.41\textwidth}
    \centering
    \includegraphics[width=\linewidth]{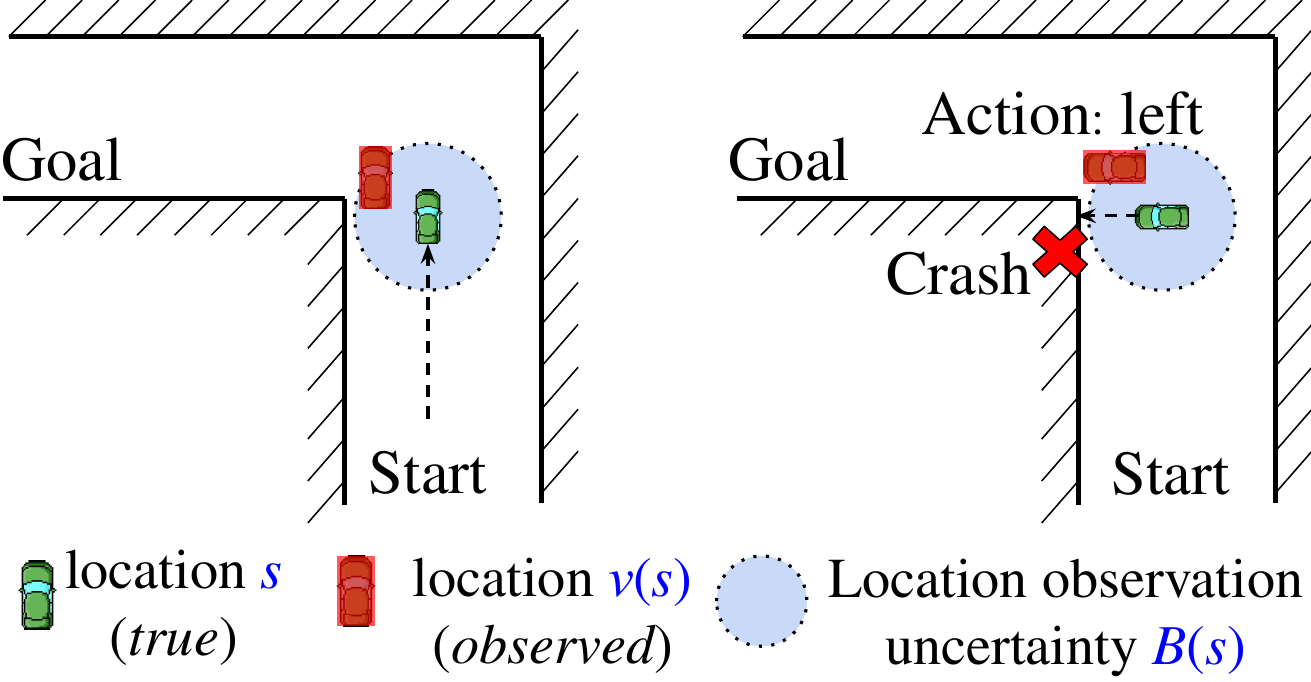}
    \vspace{-1.5em}
    \caption{A car observes its location through sensors (e.g., GPS) and plans its route to the goal. Without considering the uncertainty in observed location (e.g., error of GPS coordinates), an unsafe policy may crash into the wall because $s \neq \nu(s)$.
    }
    \label{fig:example_uncertainty}
\end{wrapfigure}
under this setting.
Specifically, we can attack the agent and generate trajectories adversarially during training time, and apply any existing DRL algorithm to hopefully obtain a robust policy. Unfortunately, we show that for most environments, naive adversarial training (e.g., putting adversarial states into the replay buffer) leads to unstable training and deteriorates agent performance~\citep{behzadan2017whatever,fischer2019online}, or does not significantly improve robustness under strong attacks. Since RL and supervised learning are quite different problems, naively applying techniques from supervised learning to RL without a proper theoretical justification can be unsuccessful.
To summarize, we study the theory and practice of robust RL against perturbations on state observations:

\begin{itemize}[wide,itemsep=0pt]
    \item We formulate the perturbation on state observations as a modified Markov decision process (MDP), which we call state-adversarial MDP (SA-MDP), and study its fundamental properties. We show that under an optimal adversary, a stationary and Markovian optimal policy may not exist for SA-MDP.
    \item Based on our theory of SA-MDP, we propose a theoretically principled robust policy regularizer which is related to the total variation distance or KL-divergence on perturbed policies. It can be practically and efficiently applied to a wide range of RL algorithms, including PPO, DDPG and DQN.
    \item We conduct experiments on 10 environments ranging from Atari games with discrete actions to complex control tasks in continuous action space. Our proposed method significantly improves robustness under strong white-box attacks on state observations, including two \emph{strong} attacks we design, the robust Sarsa attack (RS attack) and maximal action difference attack (MAD attack). 
\end{itemize}


\setlength{\floatsep}{5pt}
\setlength{\textfloatsep}{5pt}

\vspace{-0.3cm}

\section{Related Work}
\vspace{-3pt}
\paragraph{Robust Reinforcement Learning}
Since each element of RL (observations, actions, transition dynamics and rewards) can contain uncertainty, robust RL has been studied from different perspectives. Robust Markov decision process (RMDP)~\citep{iyengar2005robust,nilim2004robustness} considers the worst case perturbation from transition probabilities, and has been extended to distributional settings~\citep{xu2010distributionally} and partially observed MDPs~\citep{osogami2015robust}. The agent observes the original true state from the environment and acts accordingly, but the environment can choose from a set of transition probabilities that minimizes rewards. 
Compared to our SA-MDP where the adversary changes only observations, in RMDP the ground-truth states are changed so RMDP is more suitable for modeling \emph{environment parameter changes} (e.g., changes in physical parameters like mass and length, etc). 
RMDP theory has inspired robust deep Q-learning~\citep{shashua2017deep} and policy gradient algorithms~\citep{mankowitz2018learning,derman2018soft,mankowitz2019robust} that are robust against small environmental changes.

Another line of works~\citep{pinto2017robust,li2019robust} consider the adversarial setting of multi-agent reinforcement learning~\citep{tan1993multi,bu2008comprehensive}. 
In the simplest two-player setting (referred to as minimax games~\citep{littman1994markov}), each agent chooses an action at each step, and the environment transits based on both actions. The regular $Q$ function $Q(s, a)$ can be extended to $Q(S, a, o)$ where $o$ is the opponent's action and Q-learning is still convergent. This setting can be extended to deep Q learning and policy gradient algorithms~\citep{li2019robust,pinto2017robust}. \citet{pinto2017robust} show that learning an opponent simultaneously can improve the agent's performance as well as its robustness against environment turbulence and test conditions (e.g., change in mass or friction). \citet{gu2019adversary} carried out real-world experiments on the two-player adversarial learning game. 
\citet{tessler2019action} considered adversarial perturbations on the action space. \citet{fu2017learning} investigated how to learn a robust reward.
All these settings are different from ours: we manipulate only the state observations but do not change the underlying environment (the true states) directly. 
\vspace{-4pt}
\paragraph{Adversarial Attacks on State Observations in DRL}
\citet{huang2017adversarial} evaluated the robustness of deep reinforcement learning policies through an FGSM based attack on Atari games with discrete actions. \citet{kos2017delving} proposed to use the value function to guide adversarial perturbation search. \citet{lin2017tactics} considered a more complicated case where the adversary is allowed to attack only a subset of time steps, and used a generative model to generate attack plans luring the agent to a designated target state. \citet{behzadan2017vulnerability} studied black-box attacks on DQNs with discrete actions via transferability of adversarial examples. \citet{pattanaik2018robust} further enhanced adversarial attacks to DRL with multi-step gradient descent and better engineered loss functions. They require a critic or $Q$ function to perform attacks. Typically, the critic learned during agent training is used. We find that using this critic can be sub-optimal or impractical in many cases, and propose our two \emph{critic-independent} and strong attacks (RS and MAD attacks) in Section~\ref{sec:attack}.
We refer the reader to recent surveys~\citep{xiao2019characterizing,ilahi2020challenges} for a taxonomy and a comprehensive list of adversarial attacks in DRL setting.

\vspace{-4pt}
\paragraph{Improving Robustness for State Observations in DRL}
For discrete action RL tasks, \citet{kos2017delving} first presented preliminary results of adversarial training on Pong (one of the simplest Atari environments) using weak FGSM attacks on pixel space. \citet{behzadan2017whatever} applied adversarial training to several Atari games with DQN, and found it challenging for the agent to adapt to the attacks during training time.
These early approaches achieved much worse results than ours: for Pong, \citet{behzadan2017whatever} can improve reward under attack from $-21$ (lowest) to $-5$, yet is still far away from the optimal reward ($+21$).
Recently, \citet{mirman2018distilled,fischer2019online} treat the {\em discrete action} outputs of DQN as labels, and apply existing certified defense for classification~\citep{mirman2018differentiable} to robustly predict actions using imitation learning. 
This approach outperforms~\citep{behzadan2017whatever}, but it is unclear how to apply it to environments with continuous action spaces. Compared to their approach, our SA-DQN does not use imitation learning and achieves better performance on most environments.

For continuous action RL tasks (e.g., MuJoCo environments in OpenAI Gym), \citet{mandlekar2017adversarially} used a weak FGSM based attack with policy gradient to adversarially train a few simple RL tasks. \citet{pattanaik2018robust} used stronger multi-step gradient based attacks; however, their evaluation focused on robustness against environmental changes rather than state perturbations. Unlike our work which first develops principles and then applies to different DRL algorithms, these works directly extend adversarial training in supervised learning to the DRL setting and do not reliably improve test time performance under strong attacks in Section~\ref{sec:exp}. A few concurrent works \citep{russo2019optimal,shen2020deep} consider a smoothness regularizer similar to ours: \citep{russo2019optimal} studied an attack setting to MDP similar to ours and proposed Lipschitz regularization, but it was applied to DQN with discrete actions only. 
\citep{shen2020deep} adopted virtual adversarial training also for the continuous-action settings but focused on improving generalization instead of robustness.
In our paper, we provide theoretical justifications for our robustness regularizer from the perspective of constrained policy optimization~\citep{achiam2017constrained}, systematically 
apply our approach to multiple RL algorithms (PPO, DDPG and DQN), 
propose more effective adversarial attacks and conduct comprehensive empirical evaluations under a suit of strong adversaries.

Other related works include \citep{havens2018online}, which proposed a meta online learning procedure with a master agent detecting the presence of the adversary and switching between a few sub-policies, but did not discuss how to train a single agent robustly. \citep{chen2018gradient} applied adversarial training specifically for RL-based path-finding algorithms. \citep{lutjens2019certified} considered the worst-case scenario during rollouts for existing DQN agents to ensure safety, but it relies on an existing policy and does not include a training procedure. 



\vspace{-0.2cm}
\vspace{-0.2cm}
\section{Methodology}
\vspace{-0.9em}
\subsection{State-Adversarial Markov Decision Process (SA-MDP)}

\paragraph{Notations} A Markov decision process (MDP) is defined as $(\mathcal{S}, \mathcal{A}, R,  p, \gamma)$, where $\mathcal{S}$ is the state space, $\mathcal{A}$ is the action space, $R: \mathcal{S} \times \mathcal{A} \times \mathcal{S} \rightarrow \sR$ is the reward function, and $p: \mathcal{S} \times \mathcal{A} \rightarrow  \mathcal{P}(\mathcal{S})$ is the transition probability of environment, where $\mathcal{P}(\mathcal{S})$ defines the set of all possible probability measures on $\mathcal{S}$. The transition probability $p(s' | s, a)\!=\!\mathrm{Pr}(s_{t+1}\!=\!s' | s_t\!=\!s, a_t\!=\!a)$, where $t$ is the time step. We denote a stationary policy as $\pi: \mathcal{S} \rightarrow \mathcal{P}(\mathcal{A})$, the set of all stochastic and Markovian policies as $\Pi_\text{MR}$, the set of all deterministic and Markovian policies as $\Pi_\text{MD}$. Discount factor $0 < \gamma < 1$.

\begin{wrapfigure}[12]{r}{0.4\textwidth}
    \centering
    \vspace{-1.45em}
    \includegraphics[width= \linewidth]{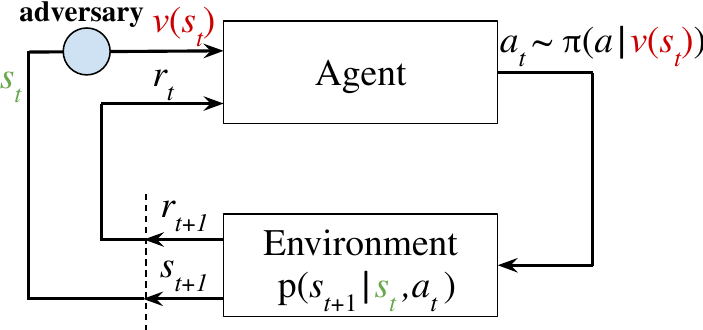}
    \caption{Reinforcement learning with perturbed state observations. The agent observes a perturbed state $\nu(s_t)$ rather than the true environment state $s_t$.}
    \label{fig:state_adversarial_rl}
\end{wrapfigure}
In state-adversarial MDP (SA-MDP), we introduce an adversary $\nu(s): \mathcal{S} \rightarrow \mathcal{S}$~\footnote{Our analysis also holds for a stochastic adversary. The optimal adversary is deterministic (see Lemma \ref{lemma:equivalence_optimal}).}. The adversary perturbs only the state observations of the agent, such that the action is taken as $\pi(a|\nu(s))$; the environment still transits from the true state $s$ rather than $\nu(s)$ to the next state. Since $\nu(s)$ can be different from $s$, the agent's action from $\pi(a|\nu(s))$ may be sub-optimal, and thus the adversary is able to reduce the reward. In  real world RL problems, the adversary can be reflected as the worst case noise in measurement or state estimation uncertainty. Note that this scenario is different from the two-player Markov game~\citep{littman1994markov} where both players see unperturbed true environment states and interact with the environment directly; the opponent's action can change the true state of the game.

To allow a formal analysis, we first make the assumption for the adversary $\nu$:
\begin{assumption}[Stationary, Deterministic and Markovian Adversary]
\label{assumpt:stationary_markovian}
$\nu(s)$ is a deterministic function $\nu: \mathcal{S} \rightarrow \mathcal{S}$ which only depends on the current state $s$, and $\nu$ does not change over time.
\end{assumption}
This assumption holds for many adversarial attacks~\citep{huang2017adversarial,lin2017tactics,kos2017delving,pattanaik2018robust}. These attacks only depend on the current state input and the policy or Q network so they are Markovian; the network parameters are frozen at test time, so given the same $s$ the adversary will generate the same (stationary) perturbation. We leave the formal analysis of non-Markovian, non-stationary adversaries as future work.

If the adversary can perturb a state $s$ arbitrarily without bounds, the problem can become trivial. To fit our analysis to the most realistic settings, we need to restrict the power of an adversary. We define perturbation set $B(s)$, to restrict the adversary to perturb a state $s$ only to a predefined set of states:
\begin{definition}[Adversary Perturbation Set]
\label{assumpt:bounded_power} We define a set $B(s)$ which contains all allowed perturbations of the adversary. Formally, $\nu(s) \in B(s)$ where $B(s)$ is a set of states and $s \in \mathcal{S}$.
\end{definition}
$B(s)$ is usually a set of task-specific ``neighboring'' states of $s$ (e.g., bounded sensor measurement errors), which makes the observation still meaningful (yet not accurate) even with perturbations. After defining $B$, an SA-MDP can be represented as a 6-tuple $(\mathcal{S}, \mathcal{A}, B, R,  p, \gamma)$.

\paragraph{Analysis of SA-MDP} We first derive Bellman Equations and a basic policy evaluation procedure, then we discuss the possibility of obtaining an optimal policy for SA-MDP.
The adversarial value and action-value functions under $\nu$ in an SA-MDP are similar to those of a regular MDP:
\begin{align*}
\tilde{V}_{\pi\circ\nu}(s) = \E_{\pi\circ\nu} \left [ \sum_{k=0}^{\infty} \gamma^k r_{t+k+1} | s_t = s \right ], \quad
\tilde{Q}_{\pi\circ\nu}(s,a) = \E_{\pi\circ\nu} \left [ \sum_{k=0}^{\infty} \gamma^k r_{t+k+1} | s_t = s, a_t=a \right ],
\end{align*}
where the reward at step-$t$ is defined as $r_t$ and $\pi\circ\nu$ denotes the policy under observation perturbations: $\pi(a|\nu(s))$. Based on these two definitions, we first consider the simplest case with \emph{fixed} $\pi$ and $\nu$:
\newreptext{theorem_fixed_pi_nu}{
\begin{theorem}[Bellman equations for fixed $\pi$ and $\nu$]
Given $\pi: \mathcal{S} \rightarrow \mathcal{P}(\mathcal{A})$ and $\nu: \mathcal{S} \rightarrow \mathcal{S}$, we have
\begin{align*}
\tilde{V}_{\pi\circ\nu}(s) &= \sum_{a \in \mathcal{A}} \pi(a|\nu(s)) \sum_{s' \in \mathcal{S}} p(s'|s,a) \left [R(s, a, s') + \gamma \tilde{V}_{\pi\circ\nu}(s^\prime) \right ]\\
\tilde{Q}_{\pi\circ\nu}(s,a) &= \sum_{s' \in \mathcal{S}} p(s'|s,a) \left [R(s, a, s') + \gamma \sum_{a' \in \mathcal{A}} \pi(a'|\nu(s')) \tilde{Q}_{\pi\circ\nu}(s', a') \right ].
\end{align*}
\vspace{-1em}
\label{thm:bellman_nu}
\end{theorem}
}
\reptext{theorem_fixed_pi_nu}
The proof of Theorem~\ref{thm:bellman_nu} is simple, as when $\pi, \nu$ are fixed, they can be ``merged'' as a single policy, and existing results from MDP can be directly applied. Now we consider a more complicated case, where we want to find the value functions under \emph{optimal adversary} $\nu^*(\pi)$, minimizing the total expected reward for a \emph{fixed} $\pi$. The optimal adversarial value and action-value functions are defined as:
\begin{align*}
\tilde{V}_{\pi \circ \nu^*}(s) = \min_{\nu} \tilde{V}_{\pi\circ\nu}(s), \quad
\tilde{Q}_{\pi \circ \nu^*}(s,a) = \min_{\nu} \tilde{Q}_{\pi\circ\nu}(s,a).
\end{align*}
\newreptext{theorem_bellman_optimal}{
\begin{theorem}[Bellman contraction for optimal adversary]
\label{thm:optimal_adversary}
Define Bellman operator $\mathscr{L}: \sR^{|\mathcal{S}|} \rightarrow \sR^{|\mathcal{S}|}$,
\begin{align}
(\mathscr{L_pi}\tilde{V})(s)=\min_{s_\nu \in B(s)}\sum_{a \in \mathcal{A}} \pi(a|s_\nu) \sum_{s' \in \mathcal{S}} p(s'|s,a) \left [R(s, a, s') + \gamma \tilde{V}(s') \right ].
\end{align}
\normalsize
The Bellman equation for optimal adversary $\nu^*$ can then be written as:
$\tilde{V}_{\pi \circ \nu^*}=\mathscr{L_pi}\tilde{V}_{\pi \circ \nu^*}$.
Additionally, $\mathscr{L}$ is a contraction that converges to $\tilde{V}_{\pi \circ \nu^*}$.
\end{theorem}
}
\reptext{theorem_bellman_optimal}
\begin{wrapfigure}[13]{r}{0.29\textwidth}
    \vspace{-18pt}
    \centering
    \includegraphics[width=1.0\linewidth]{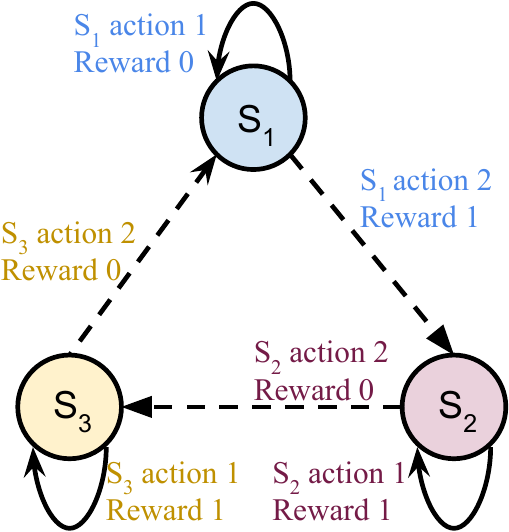}
    \caption{A toy environment.}
    \label{fig:mdp_graph_main}
\end{wrapfigure}
Theorem~\ref{thm:optimal_adversary} says that given a \emph{fixed} policy $\pi$, we can evaluate its performance (value functions) under the optimal (strongest) adversary, through a Bellman contraction. It is functionally similar to the ``policy evaluation'' procedure in regular MDP. 
The proof of Theorem~\ref{thm:optimal_adversary} is in the same spirit as the proof of Bellman optimality equations for solving the optimal policy for an MDP; the important difference here is that we solve the optimal adversary, for a \emph{fixed} policy $\pi$.
Given $\pi$, value functions for MDP and SA-MDP can be vastly different. Here we show a 3-state toy environment in Figure~\ref{fig:mdp_graph_main}; an optimal MDP policy is to take action 2 in $S_1$, action 1 in $S_2$ and $S_3$. 
Under the presence of an adversary $\nu(S_1)=S_2$, $\nu(S_2)=S_1$, $\nu(S_3)=S_1$, 
this policy receives 
zero total reward as the adversary can make the action $\pi(a|\nu(s))$ totally wrong regardless of the states. 
On the other hand, a policy taking random actions on all three states (which is a non-optimal policy for MDP) is unaffected by the adversary and obtains non-zero rewards in SA-MDP. Details are given in Appendix~\ref{sec:example_mdp}.

Finally, we discuss our ultimate quest of finding an \emph{optimal} policy $\pi^*$ under the strongest adversary $\nu^* (\pi)$ in the SA-MDP setting (we use the notation $\nu^*(\pi)$ to explicit indicate that $\nu^*$ is the optimal adversary for a given $\pi$). An optimal policy should be the best among all policies on every state:
\begin{align}
\label{eq:optimal_v}
\tilde{V}_{{\pi^*} \circ \nu^*(\pi^*)}(s) \geq \tilde{V}_{\pi \circ \nu^*(\pi)}(s)\quad \text{for }\forall s\in \mathcal{S}\text{ and }\forall\pi, 
\end{align}
where both $\pi$ and $\nu$ are not fixed. The first question is, what policy classes we need to consider for $\pi^*$. In MDPs, deterministic policies are sufficient. We show that this does not hold anymore in SA-MDP:

\newreptext{theorem_stochastic_better}{
\begin{theorem}
\label{eq:stochastic_can_be_better}
There exists an SA-MDP and some stochastic policy $\pi \in \Pi_\text{MR}$ such that we cannot find a better deterministic policy $\pi' \in \Pi_\text{MD}$ satisfying $\tilde{V}_{\pi' \circ \nu^*(\pi')}(s) \geq \tilde{V}_{\pi \circ \nu^*(\pi)}(s)$ for all $s \in \mathcal{S}$.
\end{theorem}
}
\reptext{theorem_stochastic_better}
The proof is done by constructing a counterexample where some stochastic policies are better than \emph{any} other deterministic policies in SA-MDP (see Appendix~\ref{sec:example_mdp}). Contrarily, in MDP, for any stochastic policy we can find a deterministic policy that is at least as good as the stochastic one. 
Unfortunately, even looking for both deterministic and stochastic policies still cannot always find an optimal one:

\newreptext{theorem_non_optimal}{
\begin{theorem}
\label{eq:non_exist_optimal}
Under the optimal $\nu^*$, an optimal policy $\pi^* \in \Pi_\text{MR}$ does not always exist for SA-MDP.
\end{theorem}
}
\reptext{theorem_non_optimal}
The proof follows the same counterexample as in Theorem~\ref{eq:stochastic_can_be_better}. The optimal policy $\pi^*$ requires to have $\tilde{V}_{{\pi^*} \circ \nu^*(\pi^*)}(s) \geq \tilde{V}_{\pi \circ \nu^*(\pi)}(s)$ for all $s$ and any $\pi$. In an SA-MDP, sometimes we have to make a trade-off between the value of states and no policy can maximize the values of all states.

Despite the difficulty of finding an optimal policy under the optimal adversary, we show that under certain assumptions, the loss in performance due to an optimal adversary can be bounded:
\begin{theorem}
\label{thm:optimal_distance}
Given a policy $\pi$ for a non-adversarial MDP and its value function is $V_\pi(s)$. Under the optimal adversary $\nu$ in SA-MDP, for all $s \in \mathcal{S}$ we have
\begin{equation}
\label{eq:difference_tv}
\max_{s\in\mathcal{S}}\big\{V_\pi(s) - \tilde{V}_{\pi \circ \nu^* (\pi)}(s) \big\}\leq \alpha \max_{s\in\mathcal{S}}\max_{\hat{s} \in B(s)} \mathrm{D}_{\mathrm{TV}}(\pi(\cdot|s),\pi(\cdot|\hat{s}))
\end{equation}
where $\mathrm{D}_{\mathrm{TV}}(\pi(\cdot|s),\pi(\cdot|\hat{s}))$ is the total variation distance between $\pi(\cdot|s)$ and $\pi(\cdot|\hat{s})$, and $\alpha:=2[1+\frac{\gamma}{(1-\gamma)^2}]\max_{(s,a,s^\prime)\in \mathcal{S}\times\mathcal{A}\times\mathcal{S}}|R(s,a,s^\prime)|$ is a constant that does not depend on $\pi$.
\end{theorem}
Theorem~\ref{thm:optimal_distance} says that as long as differences between the action distributions under state perturbations (the term $\mathrm{D}_{\mathrm{TV}}(\pi(\cdot|s),\pi(\cdot|\hat{s}))$) are not too large, the performance gap between $\tilde{V}_{\pi \circ \nu^*}(s)$ (state value of SA-MDP) and $V_\pi(s)$ (state value of regular MDP) can be bounded. An important consequence is the motivation of regularizing $\mathrm{D}_{\mathrm{TV}}(\pi(\cdot|s),\pi(\cdot|\hat{s}))$ during training to obtain a policy robust to strong adversaries. The proof is based on tools developed in constrained policy optimization~\citep{achiam2017constrained}, which gives an upper bound on value functions given two policies with bounded divergence. In our case, we desire that a bounded state perturbation $\hat{s}$ produces bounded divergence between $\pi(\cdot|s)$ and $\pi(\cdot|\hat{s})$. 

We now study a few practical DRL algorithms, including both deep Q-learning (DQN) for discrete actions and actor-critic based policy gradient methods (DDPG and PPO) for continuous actions.

\subsection{State-Adversarial DRL for Stochastic Policies: A Case Study on PPO}
\label{sec:sa_ppo}
We start with the most general case where the policy $\pi(a|s)$ is stochastic (e.g., in PPO~\cite{schulman2017proximal}). The total variation distance is not easy to compute for most distributions, so we upper bound it again by KL divergence: 
$\mathrm{D}_{\mathrm{TV}}(\pi(a|s),\pi(a|\hat{s})) \leq \sqrt{\frac{1}{2}\mathrm{D}_{\mathrm{KL}}(\pi(a|s)\|\pi(a|\hat{s}))}$. 
When Gaussian policies are used, we denote $\pi(a|s) \sim \mathcal{N}(\mu_s, \Sigma_s)$ and $\pi(a|\hat{s}) \sim \mathcal{N}(\mu_{\hat{s}}, \Sigma_{\hat{s}})$. The KL-divergence can be given as:
\begin{equation}
\label{eq:gaussian_kl_divergence}
    \mathrm{D}_{\mathrm{KL}}(\pi(a|s)\|\pi(a|\hat{s})) = \frac{1}{2} \left ( \log | \Sigma_{\hat{s}} \Sigma_s^{-1}| + \mathrm{tr}(\Sigma_{\hat{s}}^{-1} \Sigma_s) + (\mu_{\hat{s}} - \mu_s)^\top \Sigma_{\hat{s}}^{-1} (\mu_{\hat{s}} - \mu_s) - |\mathcal{A}| \right).
\end{equation}
Regularizing KL distance~\eqref{eq:gaussian_kl_divergence} for all $\hat{s} \in B(s)$ will lead to a smaller upper bound in~\eqref{eq:difference_tv}, which is directly related to agent performance under optimal adversary. In PPO, the mean terms $\mu_s$, $\mu_{\hat{s}}$ are produced by neural networks: $\mu_{\theta_\mu} (s)$ and $\mu_{\theta_\mu} (\hat{s})$, and we assume $\Sigma$ is a diagonal matrix independent of state $s$ ($\Sigma_{\hat{s}} = \Sigma_s = \Sigma$). Regularizing the above KL-divergence over all $s$ from sampled trajectories and all $\hat{s} \in B(s)$ leads to the following state-adversarial regularizer for PPO, ignoring constant terms:
\begin{equation}
\label{eq:ppo_regularizer}
    \mathcal{R_\text{PPO}(\theta_\mu)}\!=\!\frac{1}{2}\sum_s \max_{\hat{s} \in B(s)}\left (\mu_{\theta_\mu}(\hat{s}) - \mu_{\theta_\mu} (s) \right)^\top \Sigma^{-1} \left ( \mu_{\theta_\mu}(\hat{s})- \mu_{\theta_\mu} (s) \right ) := \frac{1}{2}\sum_s \max_{\hat{s} \in B(s)} \mathcal{R}_s(\hat{s},\theta_\mu).
\end{equation}
We replace $\max_{s\in \mathcal{S}}$ term in Theorem~\ref{thm:optimal_distance} with a more practical and optimizer-friendly summation over all states in sampled trajectory. A similar treatment was used in TRPO~\cite{schulman2015trust} which was also derived as a KL-based regularizer, albeit on $\theta_\mu$ space rather than on state space. However, minimizing~\eqref{eq:ppo_regularizer} is challenging as it is a minimax objective, and we also have $\nabla_{\hat{s}} \mathcal{R}(\hat{s}, \theta_\mu) \rvert_{\hat{s}=s}=0$ so using gradient descent directly cannot solve the inner maximization problem to a local maximum. Instead of using the more expensive second order methods, we propose two first order approaches to solve~\eqref{eq:ppo_regularizer}: convex relaxations of neural networks, and Stochastic Gradient Langevin Dynamics (SGLD). Here we focus on discussing convex relaxation based method, and we defer SGLD based solver to Section~\ref{sec:optimizing_sgld}.

Convex relaxation of non-linear units in neural networks enables an efficient analysis of the outer bounds for a neural network~\citep{wong2018provable,zhang2018efficient,singh2018fast,dvijotham2018dual,weng2018towards,wang2018efficient,salman2019convex,singh2019abstract}. Several works have used it for certified adversarial defenses~\citep{wong2018scaling,mirman2018differentiable,wang2018mixtrain,gowal2018effectiveness,zhang2019towards}, but here we leverage it as a generic optimization tool for solving minimax functions involving neural networks. Using this technique, we can obtain an upper bound for $\mathcal{R}_s(\hat{s},\theta_\mu)$: $\overline{\mathcal{R}}_s(\theta_\mu) \geq \mathcal{R}_s(\hat{s},\theta_\mu)$ for all $\hat{s} \in B(s)$. $\overline{\mathcal{R}}_s(\theta_\mu)$ is also a function of $\theta_\mu$ and can be seen as a transformed neural network (e.g., the dual network in~\citet{wong2018provable}), and computing $\overline{\mathcal{R}}_s(\theta_\mu)$ is only a constant factor slower than computing $\mathcal{R}_s(s, \theta_\mu)$ (for a fixed $s$) when an efficient relaxation~\citep{mirman2018differentiable,gowal2018effectiveness,zhang2019towards} is used. We can then solve the following minimization problem:
\[
\min_{\theta_\mu} \frac{1}{2}\sum_s \overline{\mathcal{R}}_s(\theta_\mu) \geq \min_{\theta_\mu} \frac{1}{2}\sum_s \max_{\hat{s} \in B(s)} \mathcal{R}_s(\hat{s},\theta_\mu) = \min_{\theta_\mu} \mathcal{R_\text{PPO}(\theta_\mu)}.
\]
Since we minimize an \emph{upper bound} of the inner max, the original objective~\eqref{eq:ppo_regularizer} is guaranteed to be minimized. Using convex relaxations can also provide certain \emph{robustness certificates} for DRL as a bonus (e.g., we can guarantee an action has bounded changes under bounded perturbations), discussed in Appendix~\ref{sec:certificate}. We use \texttt{auto\_LiRPA}, a recently developed tool~\citep{xu2020automatic}, to give $\overline{\mathcal{R}}_s(\theta_\mu)$ efficiently and automatically. Once the inner maximization problem is solved, we can add $\mathcal{R}_\text{PPO}$ as part of the policy optimization objective, and solve PPO using stochastic gradient descent (SGD) as usual.

Although Eq~\eqref{eq:ppo_regularizer} looks similar to smoothness based regularizers in (semi-)supervised learning settings to avoid overfitting~\citep{miyato2015distributional} and improve robustness~\citep{zhang2019theoretically}, our regularizer is based on the foundations of SA-MDP. Our theory justifies the use of such a regularizer in reinforcement learning setting, while \citep{miyato2015distributional,zhang2019theoretically} are developed for quite different settings not related to reinforcement learning.

\subsection{State-Adversarial DRL for Deterministic Policies: A Case Study on DDPG}
\label{sec:sa_ddpg}

DDPG learns a deterministic policy $\pi(s): \mathcal{S} \rightarrow \mathcal{A}$, and in this situation, the total variation distance $\mathrm{D}_{TV}(\pi(\cdot|s),\pi(\cdot|\hat{s}))$ is malformed, as the densities at different states $s$ and $\hat{s}$ are very likely to be completely non-overlapping. To address this issue, we define a smoothed version of policy, $\bar{\pi}(a|s)$ in DDPG, where we add independent Gaussian noise with variance $\sigma^2$ to each action: $\bar{\pi}(a|s) \sim \mathcal{N}(\pi(s), \sigma^2 I_{|\mathcal{A}|})$.
Then we can compute $\mathrm{D}_{TV}(\bar{\pi}(\cdot|s),\bar{\pi}(\cdot|\hat{s}))$ using the following theorem:
\newreptext{theorem_policy_distance}{
\begin{theorem}
\label{thm:policy_distance}
$\mathrm{D}_{TV}(\bar{\pi}(\cdot|s), \bar{\pi}(\cdot|\hat{s})) = \sqrt{2/\pi}\frac{d}{\sigma} + O(d^3)$, where $d = \|\pi(s) - \pi(\hat{s})\|_2$.
\end{theorem}
}
\reptext{theorem_policy_distance}
Thus, as long as we can penalize $\sqrt{2/\pi}\frac{d}{\sigma}$, the total variation distance between the two smoothed distributions can be bounded. In DDPG, we parameterize the policy as a policy network $\pi_{\theta_\pi}$. Based on Theorem~\ref{thm:optimal_distance}, the robust policy regularizer for DDPG is:
\begin{equation}
    \mathcal{R}_\text{DDPG}(\theta_\pi)=\sqrt{2/\pi}(1/\sigma) \sum_s \max_{\hat{s} \in B(s)} \| \pi_{\theta_\pi}(s) - \pi_{\theta_\pi}(\hat{s}) \|_2
\label{eq:ddpg_reg}
\end{equation}
for each state $s$ in a sampled batch of states, we need to solve a maximization problem, which can be done using SGLD or convex relaxations similarly as we have shown in Section~\ref{sec:sa_ppo}. Note that the smoothing procedure can be done completely at test time, and during training time our goal is to keep $\max_{\hat{s} \in B(s)} \| \pi_{\theta_\pi} (s) - \pi_{\theta_\pi} (\hat{s}) \|_2$ small. We show the full SA-DDPG algorithm in Appendix~\ref{sec:ddpg_details}.

\subsection{State-Adversarial DRL for Q Learning: A Case Study on DQN}\label{sec:rob_dqn}

The action space for DQN is finite, and the deterministic action is determined by the max $Q$ value: $\pi(a|s)=1$ when $a=\argmax_{a'}Q(s,a')$ and 0 otherwise. The total variation distance in this case is
\begin{align*}
    \mathrm{D}_{TV}(\pi(\cdot|s),\pi(\cdot|\hat{s}))=
    \begin{cases}
    0&\argmax_a\pi(a|s)=\argmax_a\pi(a|\hat{s})\\1&\text{otherwise}.
    \end{cases}
\end{align*}
Thus, we want to make the top-1 action stay unchanged after perturbation, and we can use a hinge-like robust policy regularizer, where $a^*(s)=\argmax_a Q_\theta(s,a)$ and $c$ is a small positive constant:
\begin{align}
\begin{gathered}
    \mathcal{R}_\text{DQN}(\theta):=\sum_s \max\{\max_{\hat{s}\in B(s)}\max_{a\neq a^*} Q_\theta(\hat{s},a)-Q_\theta(\hat{s},a^*(s)),-c\}.
\end{gathered}
\label{eq:dqn_loss}
\end{align}
The sum is over all $s$ in a sampled batch. Other loss functions (e.g., cross-entropy) are also possible as long as the aim is to keep the top-1 action to stay unchanged after perturbation. This setting is similar to the robustness of classification tasks, if we treat $a^*(s)$ as the ``correct'' label, thus many robust classification techniques can be applied as in~\citep{mirman2018distilled,fischer2019online}. The maximization can be solved using projected gradient descent (PGD) or convex relaxation of neural networks. Due to its similarity to classification, we defer the details on solving $\mathcal{R}_\text{DQN}(\theta)$ and full SA-DQN algorithm to Appendix~\ref{sec:dqn_details}. 

\subsection{Robust Sarsa (RS) and Maximal Action Difference (MAD) Attacks}

\label{sec:attack}
In this section we propose two strong adversarial attacks under Assumption~\ref{assumpt:stationary_markovian} for continuous action tasks trained using PPO or DDPG.
For this setting, \citet{pattanaik2018robust} and many follow-on works use the gradient of $Q(s,a)$ to provide the direction to update states adversarially in $K$ steps:
\begin{equation}
\label{eq:critic_attack}
\begin{gathered}
s^{k+1} = s^{k} - \eta \cdot\text{proj}\left[\nabla_{s^k} Q(s^0, \pi(s^{k}))\right], \quad k=0,\dots,K-1, \text{and define } \hat{s}:=s^{K}.
\end{gathered}
\end{equation}
Here $\text{proj}[\cdot]$ is a projection to $B(s)$, $\eta$ is the learning rate, and $s^0$ is the state under attack. It attempts to find a state $\hat{s}$ triggering an action $\pi(\hat{s})$ minimizing the action-value at state $s^0$. The formulation in~\citep{pattanaik2018robust} has a glitch that the gradient is evaluated as $\nabla_{s^k} Q(s^k, \pi(s^{k}))$ rather than $\nabla_{s^k} Q(s^0, \pi(s^{k}))$. We found that the corrected form~\eqref{eq:critic_attack} is more successful. If $Q$ is a perfect action-value function, $\hat{s}$ leads to the worst action that minimizes the value at $s^0$. However, this attack has a few drawbacks:
\begin{itemize}[label={},leftmargin=0pt,noitemsep,topsep=0pt,parsep=0pt,partopsep=0pt]
\item \textbullet \enskip Attack strength strongly depends on critic quality; if $Q$ is poorly learned, is not robust against small perturbations or has obfuscated gradients, the attack fails as no correct update direction is given.
\item \textbullet \enskip It relies on the $Q$ function which is specific to the training process, but not used during roll-out.
\item \textbullet \enskip Not applicable to many actor-critic methods (e.g., TRPO and PPO) using a learned value function $V(s)$ instead of $Q(s,a)$. Finding $\hat{s}\in B(s)$ minimizing $V(s)$ does not correctly reflect the setting of perturbing observations, as $V(\hat{s})$ represents the value of $\hat{s}$ rather than the value of taking $\pi(\hat{s})$ at $s^0$.
\end{itemize}

When we evaluate the robustness of a policy, we desire it to be independent of a specific critic network to avoid these problems. We thus propose two novel \emph{critic independent} attacks for DDPG and PPO.

\textbf{Robust Sarsa (RS) attack.}
Since $\pi$ is fixed during evaluation, we can learn its corresponding $Q^\pi(s,a)$ using on-policy temporal-difference (TD) algorithms similar to Sarsa~\citep{rummery1994line} without knowing the critic network used during training. Additionally, we find that the robustness of $Q^\pi(s, a)$ is very important; if $Q^\pi(s,a)$ is not robust against small perturbations (e.g., given a state $s_0$, a small change in $a$ will significantly reduce $Q^\pi(s_0,a)$ which does not reflect the true action-value), it cannot provide a good direction for attacks. Based on these, we learn $Q^\pi(s,a)$ (parameterized as an NN with parameters $\theta$) with a TD loss as in Sarsa and an additional robustness objective to minimize:
\begin{equation*}
L_{RS}(\theta)\!=\!\sum_{i \in [N]} \left [ r_i + \gamma Q^\pi_{RS}(s_i^\prime, a_i^\prime) - Q^\pi_{RS}(s_i, a_i) \right ]^2
+ \lambda_{RS}\!\sum_{i \in [N]} \max_{\hat{a} \in B(a_i)} (Q^\pi_{RS}(s_i, \hat{a}) -  Q^\pi_{RS}(s_i, a_i))^2
\end{equation*}
$N$ is the batch size and each batch contains $N$ tuples of transitions $(s, a, r, s^\prime, a^\prime)$ sampled from agent rollouts.
The first summation is the TD-loss and the second summation is the robustness penalty with regularization $\lambda_{RS}$. $B(a_i)$ is a small set near action $a_i$ (e.g., a $\ell_\infty$ ball of norm 0.05 when action is normalized between 0 to 1).
The inner maximization can be solved using convex relaxation of neural networks as we have done in Section~\ref{sec:sa_ddpg}. Then, we use $Q^\pi_{\theta_{RS}}$ to perform critic-based attacks as in~\eqref{eq:critic_attack}. This attack sometimes significantly outperforms the attack using the critic trained along with the policy network, as its attack strength does not depend on the quality of an existing critic. We give the detailed procedure for RS attack and show the importance of the robust objective in appendix~\ref{sec:appendix_attack}.

\textbf{Maximal Action Difference (MAD) attack.} We propose another simple yet very effective attack which does not depend on a critic. Following our Theorem~\ref{thm:optimal_distance} and~\ref{thm:policy_distance}, we can find an adversarial state $\hat{s}$ by {\it maximizing} $D_\text{KL}\left (\pi(\cdot|s) \| \pi(\cdot|\hat{s})\right )$. For actions parameterized by Gaussian mean $\pi_{\theta_\pi}(s)$ and covariance matrix $\Sigma$ (independent of $s$), we minimize $L_\text{MAD}(\hat{s}):=-D_\text{KL}\left (\pi(\cdot|s) \| \pi(\cdot|\hat{s})\right )$ to find $\hat{s}$:
\begin{equation}
\argmin_{\hat{s} \in B(s)}L_\text{MAD}(\hat{s}) =
\argmax_{\hat{s} \in B(s)} \left ( \pi_{\theta_\pi}(s) - \pi_{\theta_\pi}(\hat{s}) \right )^\top \Sigma^{-1} \left ( \pi_{\theta_\pi}(s) - \pi_{\theta_\pi}(\hat{s}) \right).
\end{equation}
For DDPG we can simply set $\Sigma=I$ . The objective can be optimized using SGLD to find a good $\hat{s}$.
\vspace{-0.3cm}
\section{Experiments}
\label{sec:exp}
\vspace{-5pt}
In our experiments\footnote{Code and pretrained agents available at~\textcolor{blue}{\url{https://github.com/chenhongge/StateAdvDRL}}}, the set of adversarial states $B(s)$ is defined as an $\ell_\infty$ norm ball around $s$ with a radius $\epsilon$: $B(s):=\{\hat{s}:\|s-\hat{s}\|_\infty\leq\epsilon\}$. Here $\epsilon$ is also referred to as the perturbation budget. In MuJoCo environments, the $\ell_\infty$ norm is applied on normalized state representations.
\vspace{-0.3cm}
\paragraph{Evaluation of SA-PPO} We use the PPO implementation from~\citep{engstrom2020implementation}, which conducted hyperparameter search and published the optimal hyperparameters for PPO on three Mujoco environments in OpenAI Gym~\citep{brockman2016openai}. We use their optimal hyperparameters for PPO, and the same set of hyperparameters for SA-PPO without further tuning. We run Walker2d and Hopper $2 \times 10^6$ steps and Humanoid $1 \times 10^7$ steps to ensure convergence. Our vanilla PPO agents achieve similar or better performance than reported in the literature~\citep{engstrom2020implementation,henderson2018deep,hamalainen2018ppo}. Detailed hyperparameters are in Appendix~\ref{sec:ppo_details}. SA-PPO has one additional regularization parameter, $\kappa_\text{PPO}$, for the regularizer $\mathcal{R}_\text{PPO}$, which is chosen in \{0.003, 0.01, 0.03, 0.1, 0.3, 1.0\}. 
We solve the SA-PPO objective using both SGLD and convex relaxation methods.
We include three baselines: vanilla PPO, and adversarially trained PPO~\citep{mandlekar2017adversarially,pattanaik2018robust} with 50\% and 100\% training steps under critic attack~\citep{pattanaik2018robust}. The attack is conducted by finding $\hat{s} \in B(s)$ minimizing $V(\hat{s})$ instead of $Q(s,\pi(\hat{s}))$, as PPO does not learn a $Q$ function during learning. 
We evaluate agents using 5 attacks, including our strong RS and MAD attacks, detailed in Appendix~\ref{sec:appendix_attack}.

\begin{table*}[t]\centering
\caption{Average episode rewards $\pm$ standard deviation over 50 episodes on 3 baselines and SA-PPO. We report natural rewards (no attacks) and rewards under five adversarial attacks. In each row we bold the best (lowest) attack reward over all five attacks. The \colorbox{lightgray}{gray rows} are the most robust agents.}
\resizebox{\linewidth}{!}{
\begin{tabular}{l|c|c|c|c|c|c|c|c|c}
\hline
\multirow{2}{*}{Env.} & \multirow{2}{*}{\shortstack{$\epsilon$}}                                  & \multirow{2}{*}{Method} & \multirow{2}{*}{\shortstack{Natural\\ Reward}} & \multicolumn{5}{c|}{Attack Reward} & \multirow{2}{*}{\shortstack{Best\\ Attack}} \\
                      &                                                              &                         &                                   & Critic           & Random            & MAD                  & RS                    & RS+MAD                &  \\ \hline
                      &                                                              & PPO (vanilla)           &  3167.6$\pm$	541.6  &  1799.0$\pm$	935.2  &   2915.2$\pm$677.7 & 1505.2$\pm$	382.0   & 779.4$\pm$	33.2   & \bf 733.8$\pm$	44.6   &  733  \\
                      &                                                              & PPO (adv. 50\%)    &174$\pm$	146     &69	$\pm$83 & 141$\pm$	128  & \bf 42$\pm$	46 & 49	$\pm$50  &  44$\pm$	43 & 42  \\
                      &                                                              & PPO (adv. 100\%)        &6.1$\pm$	2.6 &  4.4	$\pm$1.8 &  6.1$\pm$	3.2 & 5.8$\pm$	2.7  &  3.8	$\pm$0.9 & \bf 3.6	$\pm$0.5  &3.6\\\rowcolor{lightgray}\cellcolor{white}
                      &                                                      \cellcolor{white}        & SA-PPO (SGLD)           &  3523.1$\pm$329.0  &  3665.5$\pm$	8.2  &  3080.2$\pm$	745.4  & 2996.6$\pm$	786.4   & \bf 1403.3$\pm$	55.0  &  1415.4$\pm$	72.0  &   1403.3 \\

\multirow{-5}{*}{\cellcolor{white} Hopper}  & \multirow{-5}{*}{\cellcolor{white}0.075} 
& SA-PPO (Convex)         &  3704.1$\pm$	2.2  &      3698.4$\pm$	4.4   &  3708.7$\pm$	23.8  &  3443.1$\pm$	466.672  & 1235.8$\pm$	50.2   & \bf 1224.2$\pm$	47.8  & 1224.2 \\ \hline
                      &                                                              & PPO (vanilla)           &  4619.5$\pm$	38.2  & 4589.3$\pm$	12.4   &  4480.0$\pm$465.3  &  4469.1$\pm$715.6  &  \bf 913.7$\pm$	54.3  &  926.8$\pm$66.3  &  913.7 \\
                      &                                                              & PPO (adv. 50\%)         & -11 $\pm$ 0.9                     & -10.6 $\pm$ 0.86  & -10.99 $\pm$ 0.95 & -10.78 $\pm$ 0.89 & \bf -11.55 $\pm$ 0.79 & -11.37 $\pm$ 0.87     & -11.55    \\
                      &                                                              & PPO (adv. 100\%)        & -113 $\pm$ 4.14                   & -111.9 $\pm$ 4.13 & -111 $\pm$ 4.27   & -112 $\pm$ 4.08   & -114.4 $\pm$ 4.0      & \bf -114.5 $\pm$ 4.09 & -114.5    \\
                      \rowcolor{lightgray}\cellcolor{white} &         \cellcolor{white}& SA-PPO (SGLD)           & 4911.8$\pm$	188.9  & 5019.0$\pm$	65.2    & 4894.8$\pm$	139.9  &     4755.7$\pm$	413.1  &   2605.6$\pm$	1255.7&  \bf 2468.4	$\pm$1205  &2468.4    \\

\multirow{-5}{*}{\cellcolor{white}Walker2d} & \multirow{-5}{*}{\cellcolor{white}0.05} & SA-PPO (Convex)         & 4486.6$\pm$	60.7 &  4572.0$\pm$	52.3   &  4475.0$\pm$	48.7  &  4343.4$\pm$	329.4  &  2168.2$\pm$	665.4  &    \bf 2076.1$\pm$	666.7   &    2076.1      \\ \hline
                      &                                                              & PPO (vanilla)           &  5270.6$\pm$1074.3  & 5494.7$\pm$	118.7   &  5648.3$\pm$	86.8  &  1140.3$\pm$	534.8  & 1036.0$\pm$	420.2   &  \bf 884.1$\pm$	356.3  & 884.1   \\
                      &                                                              & PPO (adv. 50\%)         & 234$\pm$ 28                      & 198 $\pm$ 58      & 240 $\pm$ 19.4    & 148 $\pm$ 73      & \bf 98 $\pm$ 69       & 101.5 $\pm$ 66.4      & 98        \\
                      &                                                              & PPO (adv. 100\%)        & 141.4 $\pm$ 20.6                  & 140.25 $\pm$ 16.6 & 142.13 $\pm$ 16   & 140.23 $\pm$ 34.5 &  113.2 $\pm$ 18.5  & \bf 112.6 $\pm$ 13.88     &  112.6     \\
                     \rowcolor{lightgray}\cellcolor{white} &         \cellcolor{white}                                                     & SA-PPO (SGLD)   &  6624.0$\pm$	25.5 &  6587.0$\pm$	23.1   &  6614.1$\pm$	21.4  &  6586.4$\pm$	23.5  &  6200.5$\pm$	818.1  &    \bf 6073.8$\pm$	1108.1   &  6073.8        \\
                     
\multirow{-5}{*}{\cellcolor{white}Humanoid} & \multirow{-5}{*}{\cellcolor{white}0.075} & SA-PPO (Convex)         &  6400.6$\pm$	156.8  &  6397.9	$\pm$35.6 &  6207.9$\pm$	783.3 & 6379.5$\pm$	30.5   &   4707.2$\pm$	1359.1 & \bf 4690.3$\pm$	1244.89   & 4690.3   \\ \hline
\end{tabular}
}
\label{tab:ppo_res}
\end{table*}

In Table~\ref{tab:ppo_res}, naive adversarial training deteriorates performance and does not reliably improve robustness in all three environments. Our RS attack and MAD attacks are very effective in all environments and achieve significantly lower rewards than critic and random attacks; this shows the importance of evaluation using strong attacks. SA-PPO, solved either by SGLD or the convex relaxation objective, \emph{significantly improves robustness} against strong attacks. Additionally, SA-PPO achieves natural performance (without attacks) similar to that of vanilla PPO in Walker2d and Hopper, and \emph{significantly improves the reward in Humanoid environment}. Humanoid has a high state-space dimension (376) and is usually hard to train~\citep{hamalainen2018ppo}, and our results suggest that a robust objective can be helpful even in a non-adversarial setting. Because PPO training can have large performance variance across multiple runs, to show that our SA-PPO can consistently obtain a robust agent, we repeatedly train each environment using SA-PPO and vanilla PPO at least \textbf{15 times} and attack all agents obtained. In Figures~\ref{fig:nat_reproduce} and~\ref{fig:att_reproduce} we show the box plot of the natural and best attack reward for these PPO and SA-PPO agents. We can see that the best attack rewards of most SA-PPO agents are consistently better than PPO agents (in terms of median, 25\% and 75\% percentile rewards over multiple repetitions).
\begin{figure}
\centering
\caption{Box plots of natural and attack rewards for PPO and SA-PPO. Each box is obtained from at least \textbf{15 agents} trained with the same hyperparameters as in agents reported in Table~\ref{tab:ppo_res}.
The red lines inside the boxes are median rewards, and the upper and lower sides of the boxes show 25\% and 75\% percentile rewards of 30 agents. The line segments outside of the boxes show min or max rewards.}
\label{fig:test}
\begin{subfigure}{.5\textwidth}
  \centering
  \includegraphics[width=\linewidth]{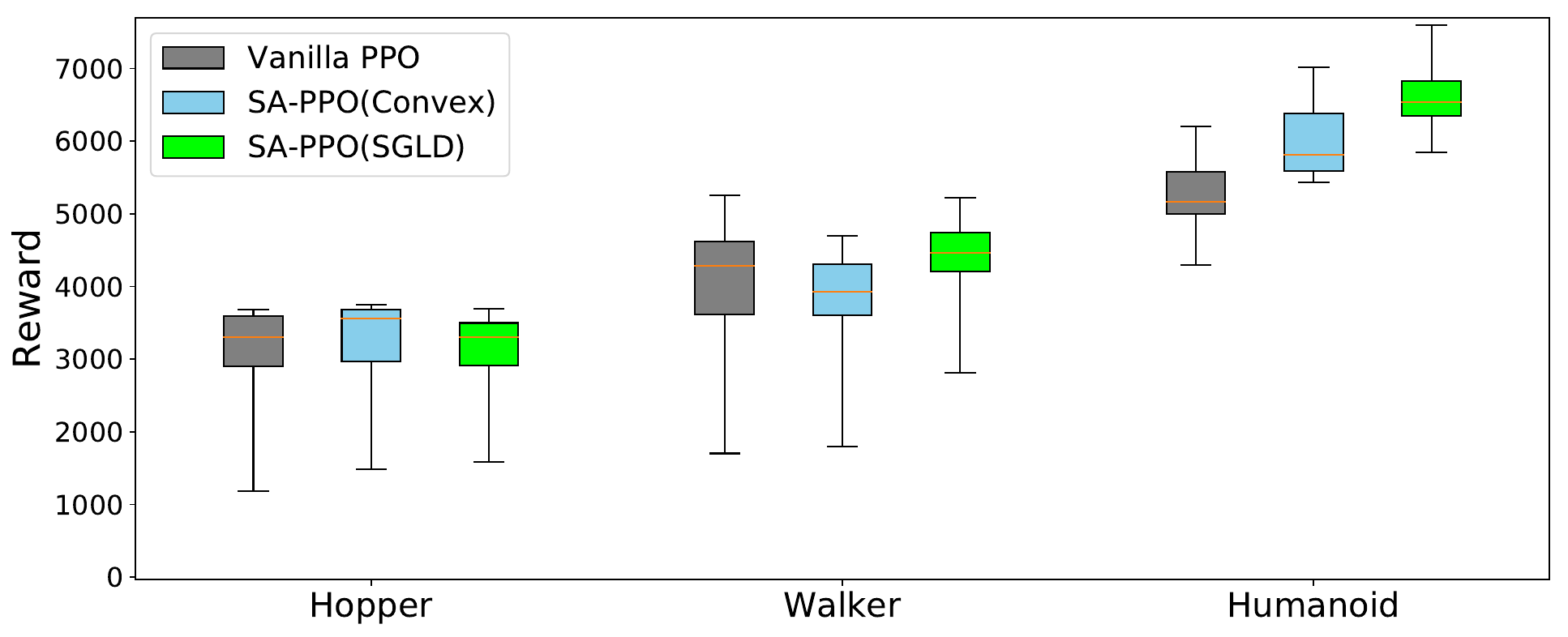}
  \caption{Natural episode rewards (no attacks)}
  \label{fig:nat_reproduce}
\end{subfigure}%
\begin{subfigure}{.5\textwidth}
  \centering
  \includegraphics[width=\linewidth]{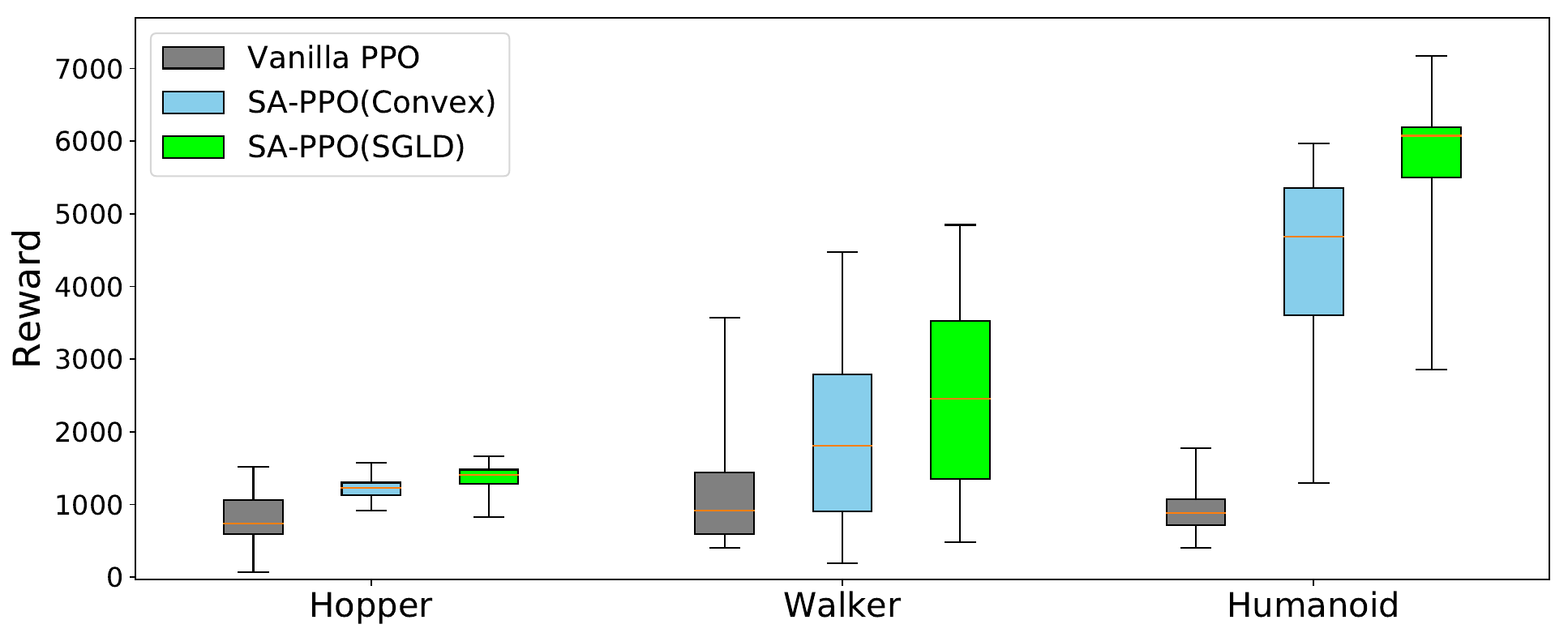}
  \caption{Rewards under the best (strongest) attacks}
  \label{fig:att_reproduce}
\end{subfigure}
\vspace{-0.4cm}
\end{figure}
\vspace{-0.3cm}
\paragraph{Evaluation of SA-DDPG}
We use a high quality DDPG implementation~\citep{deeprl} as our baseline, achieving similar or better performance on five Mujoco environments as in the literature~\citep{lillicrap2015continuous,fujimoto2018addressing}. For SA-DDPG, we use the same set of hyperparameters as in DDPG~\citep{deeprl} (detailed in Appendix~\ref{sec:ddpg_details}), except for the additional regularization term $\kappa_\text{DDPG}$ for $\mathcal{R}_\text{DDPG}$ which is searched in $\{0.1, 0.3, 1.0, 3.0\}$ for InvertedPendulum and Reacher due to their low dimensionality and $\{30, 100, 300, 1000\}$ for other environments. We include vanilla DDPG, adversarially trained DDPG~\citep{pattanaik2018robust} (attacking 50\% or 100\% steps) as baselines. We use the same set of 5 attacks as in~\ref{tab:ppo_res}. In Table~\ref{tab:ddpg_res}, we observe that naive adversarial training is not very effective in many environments. SA-DDPG (solved by SGLD or convex relaxations) significantly improves robustness under strong attacks in all 5 environments. Similar to the observations on SA-PPO, SA-DDPG can improve natural agent performance in environments (Ant and Walker2d) with relatively high dimensional state space (111 and 17).

\begin{table*}[tb]
\centering
\caption{Average episode rewards $\pm$ standard deviation over 50 episodes on DDPG, adversarial training~\citep{pattanaik2018robust} (50\% and 100\% steps) and SA-DDPG. Each number represents an agent with \emph{median} reward under the best attack over 11 training runs with identical hyperparameters. 
Due to large variance in RL, it important to report median metrics. \textbf{Bold} numbers indicate the most robust agents. Full results of all five attacks are in Table~\ref{tab:ddpg_full_res_} and statistics over multiple training runs are in Figure~\ref{fig:ddpg_repo_test}. 
}

\vspace{-0.5em}
\resizebox{0.85\linewidth}{!}{
\begin{tabular}{c|c|c |c |c |c| c}
\toprule
\multicolumn{2}{c|}{Environment}& Ant & Hopper & Inverted Pendulum & Reacher & Walker2d\\\hline
\multicolumn{2}{c|}{$\ell_\infty$ norm perturbation budget $\epsilon$}& 0.2 & 0.075 & 0.3 & 1.5 & 0.05 \\\midrule

\multirow{2}{*}{\parbox{2cm}{\centering DDPG\\(vanilla) }}&Natural Reward &$1487 \pm 850$ & $3302 \pm 762$ & $1000 \pm 0$ & $-4.37 \pm 1.54$ & $1870 \pm 1418$   \\
&Attack Reward (best) & $142 \pm 180$ & $606 \pm 124$ & $92 \pm 1$ & $-27.87 \pm 4.38$ & $790 \pm 985$ \\ \hline

\multirow{2}{*}{\parbox{2cm}{\centering DDPG\\(adv. 50\%)}}&Natural Reward & $1487 \pm 850$ & $3302 \pm 762$ & $1000 \pm 0$ & $-4.37 \pm 1.54$ & $1870 \pm 1418$ \\
&Attack Reward (best) &$31 \pm 179$ & $41 \pm 105$ & $39 \pm 0$ & $-25.81 \pm 6.53$ & $837 \pm 722$ \\\hline

\multirow{2}{*}{\parbox{2cm}{\centering DDPG\\(adv. 100\%)}}&Natural Reward & $1082 \pm 574$ & $973 \pm 0$ & $1000 \pm 0$ & $-5.71 \pm 1.80$ & $462 \pm 569$  \\
&Attack Reward (best) &$-52 \pm 231$ & $24 \pm 15$ & $82 \pm 0$ & $-27.44 \pm 4.05$ & $302 \pm 260$ \\\Xhline{2\arrayrulewidth}

\multirow{2}{*}{\parbox{2cm}{\centering SA-DDPG\\(SGLD)}}&Natural Reward & $2186 \pm 534$ & $3068 \pm 223$ & $1000 \pm 0$ & $-5.38 \pm 1.74$ & $3318 \pm 680$ \\
&Attack Reward (best) &  $\mathbf{2007 \pm 686}$ & $\mathbf{1609 \pm 676}$ & $423 \pm 281$ & $\mathbf{-12.10 \pm 4.58}$ & $1210 \pm 979$ \\\hline

\multirow{2}{*}{\parbox{2cm}{\centering SA-DDPG\\(convex relax)}}&Natural Reward & \cellcolor{white} $2254 \pm 430$ & $3128 \pm 453$ & $1000 \pm 0$ & $-5.24 \pm 2.06$ & $4540 \pm 1562$ \\
&Attack Reward (best) & $1820 \pm 635$ & $1202 \pm 402$ & $\mathbf{1000 \pm 0}$ & $-12.44 \pm 3.77$ & $\mathbf{1986 \pm 1993}$  \\\bottomrule
\end{tabular}
}
\label{tab:ddpg_res}\end{table*}

\begin{table*}\centering
\caption{Average episode rewards $\pm$ std and \emph{action certification rates} over 50 episodes on three baselines and SA-DQN. We report natural rewards (no attacks) and PGD attack rewards (under 10-step or 50-step PGD). Action certification rate is the proportion of the actions during rollout that are guaranteed unchanged by any attacks within the given $\epsilon$. 
Training time is reported in Section~\ref{sec:dqn_details}.
}
\resizebox{0.9\linewidth}{!}{
\begin{tabular}{c|c|c|c|c|c}
\toprule
\multicolumn{2}{c|}{Environment}&Pong&Freeway&BankHeist&RoadRunner\\\hline
\multicolumn{2}{c|}{$\ell_\infty$ norm perturbation budget $\epsilon$}&\multicolumn{4}{c}{1/255}\\\midrule
\multirow{3}{*}{\shortstack{DQN\\(vanilla)}}&Natural Reward&21.0 $\pm$ 0.0&34.0 $\pm$ 0.2&\cellcolor{white}1308.4 $\pm$ 24.1&\cellcolor{white}45534.0 $\pm$ 7066.0\\ 
&PGD Attack Reward (10 steps)&{-21.0$\pm$0.0}&{0.0$\pm$0.0}&{56.4$\pm$21.2}&{0.0$\pm$0.0}\\
&Action Cert. Rate&0.0&0.0&0.0&0.0\\\hline
\multirow{3}{*}{\shortstack{DQN Adv. Training\\(attack 50\% frames)\\ \citet{behzadan2017whatever}}}&Natural Reward&10.1 $\pm$ 6.6&25.4$\pm$0.8&1126.0$\pm$70.9&22944.0$\pm$6532.5\\ 
&PGD Attack Reward (10 steps)&{-21.0 $\pm$ 0.0} &{0.0$\pm$0.0}&{9.4$\pm$13.6}&{14.0$\pm$34.7}\\
&Action Cert. Rate&0.0&0.0&0.0&0.0\\\hline
\multirow{2}{*}{\shortstack{Imitation learning\\\citet{fischer2019online}}}&Natural Reward&19.73&32.93&238.66&12106.67\\
&PGD Attack Reward (4 steps)&18.13&\bf 32.53&{190.67}&5753.33\\\Xhline{2\arrayrulewidth}
\multirow{3}{*}{\shortstack{SA-DQN\\(PGD)}}&Natural Reward&21.0$\pm$0.0&\cellcolor{white} 33.9 $\pm$ 0.4&1245.2$\pm$14.5&34032.0$\pm$3845.0\\
&PGD Attack Reward (10 steps)& 21.0$\pm$0.0&23.7 $\pm$ 2.3& 1006.0$\pm$226.4&  20402.0$\pm$7551.1\\
&Action Cert. Rate&0.0&0.0&0.0&0.0\\\hline
\multirow{4}{*}{\shortstack{SA-DQN\\(convex)}}&Natural Reward&\cellcolor{white} 21.0 $\pm$ 0.0&30.0$\pm$0.0&1235.4$\pm$9.8&44638.0$\pm$7367.0\\
&PGD Attack Reward (10 steps)&\bf 21.0 $\pm$ 0.0&30.0$\pm$0.0&\bf 1232.4$\pm$16.2&\ \bf 44732.0$\pm$8059.5\\
&PGD Attack Reward (50 steps)&\bf 21.0 $\pm$ 0.0&30.0$\pm$0.0&\bf 1234.6$\pm$16.6&\ \bf 44678.0$\pm$6954.0\\
&Action Cert. Rate&1.000&1.000&0.984&0.475\\
\bottomrule
\end{tabular}
}
\label{tab:dqn_res}
\end{table*}


\vspace{-0.3cm}
\paragraph{Evaluation of SA-DQN}
We implement Double DQN~\cite{van2016deep} and Prioritized Experience Replay~\cite{schaul2015prioritized} on four Atari games.
We train Atari agents for 6 million frames for both vanilla DQN and SA-DQN. Detailed parameters and training procedures are in Appendix~\ref{sec:dqn_details}. We normalize the pixel values to $[0,1]$ and we add $\ell_\infty$ adversarial noise with norm $\epsilon=1/255$. 
We include vanilla DQNs and adversarially trained DQNs with 50\% of frames under attack~\citep{behzadan2017whatever} during training time as baselines, and we report results of robust imitation learning~\citep{fischer2019online}. We evaluate all environments under 10-step untargeted PGD attacks, except that results from~\citep{fischer2019online} were evaluated using a weaker 4-step PGD attack. For the most robust Atari agents (SA-DQN convex), we additionally attack them using 50-step PGD attacks, and find that the rewards do not further reduce.
In Table~\ref{tab:dqn_res}, we see that our SA-DQN achieves much higher rewards under attacks in most environments, and naive adversarial training is mostly ineffective under strong attacks. We obtain better rewards than~\citep{fischer2019online} in most environments, as we learn the agents directly rather than using two-step imitation learning.


\textbf{Robustness certificates.}
When our robust policy regularizer is trained using convex relaxations, we can obtain certain robustness certificates under observation perturbations. 
For a simple environment like Pong, we can guarantee actions do not change for all frames during rollouts, thus guarantee the cumulative rewards under perturbation. 
For SA-DDPG, the {\it upper bounds} on the maximal $\ell_2$ difference in action changes is a few times smaller than baselines on all 5 environments (see Appendix~\ref{sec:app_exp}). Unfortunately, for most RL tasks, due to the complexity of environment dynamics and reward process, it is impossible to obtain a ``certified reward'' as the certified test error in supervised learning settings~\citep{wong2018provable,zhang2019towards}. We leave further discussions on these challenges in Appendix~\ref{sec:certificate}.

\section*{Broader Impact}

Reinforcement learning is a central part of modern artificial intelligence and is still under heavy development in recent years. Unlike supervised learning which has been widely deployed in many commercial and industrial applications, reinforcement learning has not been widely accepted and deployed in real-world settings. Thus, the study of reinforcement learning robustness under the adversarial attacks settings receives less attentions than the supervised learning counterparts. 

However, with the recent success of reinforcement learning on many complex games such as Go~\citep{silver2017mastering}, StartCraft~\citep{vinyals2019grandmaster} and Dota 2~\citep{berner2019dota}, we will not be surprised if we will see reinforcement learning (especially, deep reinforcement learning) being used in everyday decision making tasks in near future. The potential social impacts of applying reinforcement learning agents thus must be investigated before its wide deployment. One important aspect is the trustworthiness of an agent, where robustness plays a crucial rule. The robustness considered in our paper is important for many realistic settings such as sensor noise, measurement errors, and man-in-the-middle (MITM) attacks for a DRL system. if the robustness of reinforcement learning can be established, it has the great potential to be applied into many mission-critical tasks such as autonomous driving~\citep{shalev2016safe,sallab2017deep,you2019advanced} to achieve superhuman performance.

On the other hand, one obstacle for applying reinforcement learning to real situations (beyond games like Go and StarCraft) is the ``reality gap'': a well trained reinforcement learning agent in a simulation environment can easily fail in real-world experiments. One reason for this failure is the potential sensing errors in real-world settings; this was discussed as early as in~\citet{brooks1992artificial} in 1992 and still remains an open challenge now.
Although our experiments were done in simulated environments, we believe that a smoothness regularizer like the one proposed in our paper can also benefit agents tested in real-world settings, such as robot hand manipulation~\citep{akkaya2019solving}.


\section*{Acknowledgments and Disclosure of Funding}
We acknowledge the support by NSF IIS-1901527, IIS-2008173, ARL-0011469453, 
and scholarship by IBM.
The authors thank Ge Yang and Xiaocheng Tang for helpful discussions.

\bibliography{ref}
\bibliographystyle{ref}

\newpage
\onecolumn
\appendix
\setlength{\abovedisplayskip}{7.0pt plus 2.0pt minus 5.0pt}
\setlength{\belowdisplayskip}{7.0pt plus 2.0pt minus 5.0pt}
\setlength{\abovedisplayshortskip}{0.0pt plus 3.0pt}
\setlength{\belowdisplayshortskip}{4.0pt plus 3.0pt minus 3.0pt} 

\setlength{\topsep}{4.0pt plus 1.0pt minus 2.0pt}

\renewcommand{\algorithmicrequire}{\textbf{Input:}}
\renewcommand{\algorithmicensure}{\textbf{Output:}}
\renewcommand{\algorithmiccomment}[1]{\hfill$\triangleright$\textcolor{blue}{#1}}

{\Large \bf Appendix}
\begin{itemize}[wide]
\item Readers who are interested in SA-MDP can find an example of SA-MDP in Section~\ref{sec:example_mdp} and complete proofs in Section~\ref{sec:proofs}.

\item Readers who are interested in adversarial attacks can find more details about our new attacks and existing attacks in Section~\ref{sec:appendix_attack}. Especially, we discussed how a robust critic can help in attacking RL, and show experiments on the improvements gained by the robustness objective during attack.

\item Readers who want to know more details of optimization techniques to solve our state-adversarial robust regularizers can refer to Section~\ref{sec:optimization}, including more background on convex relaxations of neural networks in Section~\ref{sec:convex-relaxation}.

\item We provide detailed algorithm and hyperparameters for SA-PPO in Section~\ref{sec:ppo_details}. We provide details for SA-DDPG in Section~\ref{sec:ddpg_details}. We provide details for SA-DQN in Section~\ref{sec:dqn_details}.

\item We provide more empirical results in Section~\ref{sec:app_exp}. To demonstrate the convergence of our algorithm, we repeat each experiment at least 15 times and plot the convergence of rewards during multiple runs. We found that for some environments (like Humanoid) we can  consistently improve baseline performance. 
We also evaluate some settings under multiple perturbation strength $\epsilon$.

\end{itemize}

\section{An example of SA-MDP}
\label{sec:example_mdp}
We first show a simple environment and solve it under different settings of MDP and SA-MDP. The environment has three states $\mathcal{S}=\{S_1, S_2, S_3\}$ and 2 actions $\mathcal{A}=\{A_1, A_2\}$. The transition probabilities and rewards are defined as below (unmentioned probabilities and rewards are 0):
\begin{align*}
\mathrm{Pr}(s'=S_1|s=S_1, a=A_1) &= 1.0 \\
\mathrm{Pr}(s'=S_2|s=S_1, a=A_2) &= 1.0 \\
\mathrm{Pr}(s'=S_2|s=S_2, a=A_2) &= 1.0 \\
\mathrm{Pr}(s'=S_3|s=S_2, a=A_1) &= 1.0 \\
\mathrm{Pr}(s'=S_1|s=S_3, a=A_2) &= 1.0 \\
\mathrm{Pr}(s'=S_2|s=S_3, a=A_1) &= 1.0 \\
R(s=S_1, a=A_2, s'=S_2) &= 1.0 \\
R(s=S_2, a=A_1, s'=S_2) &= 1.0 \\
R(s=S_3, a=A_1, s'=S_3) &= 1.0
\end{align*}
\begin{figure}[htbp]
    \centering
    \includegraphics[width=0.35\linewidth]{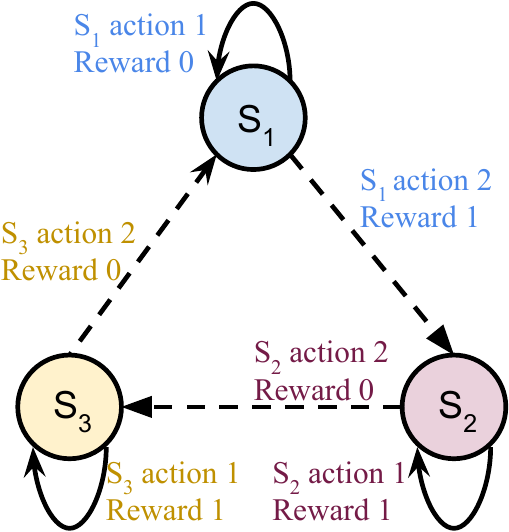}
    \caption{A simple 3-state toy environment.}
    \label{fig:simple_environment}
\end{figure}
The environment is illustrated in Figure~\ref{fig:simple_environment}.
For the power of adversary, we allow $\nu$ to perturb one state to any other two neighbouring states:
\begin{align*}
    B_\nu(S_1) = B_\nu(S_2) = B_\nu(S_3) = \{S_1, S_2, S_3\}
\end{align*}
Now we evaluate various policies for MDP and SA-MDP for this environment. We use $\gamma=0.99$ as the discount factor. A stationary and Markovian policy in this environment can be described by 3 parameters $p_{11}, p_{21}, p_{31}$ where $p_{ij} \in [0,1]$ denotes the probability $\mathrm{Pr}(a = A_j|s=S_i)$. We denote the value function as $V$ for MDP and $\tilde{V}$ for SA-MDP.

\begin{itemize}
    \item \textbf{Optimal Policy for MDP.} For a regular MDP, the optimal solution is $p_{11}=0$, $p_{21}=1$, $p_{31}=1$. We take $A_2$ to receive reward and leave $S_1$, and then keep doing $A_1$ in $S_2$ and $S_3$. The values for each state are $V(S_1)=V(S_2)=V(S_3)=\frac{1}{1-\gamma}=100$, which is optimal. However, this policy obtains $\tilde V(S_1)=\tilde V(S_2)=\tilde V(S_3)=0$ for SA-MDP, because we can set $\nu(S_1)=S_2$, $\nu(S_2)=S_1$, $\nu(S_3)=S_1$ and consequentially we always take the wrong action receiving 0 reward.
    \item \textbf{A Stochastic Policy for MDP and SA-MDP.} We consider a stochastic policy where $p_{11}=p_{21}=p_{31}=0.5$. Under this policy, we randomly stay or move in each state, and has a 50\% probability of receiving a reward. The adversary $\nu$ has no power because $\pi$ is the same for all states. In this situation, $V(S_1)=\tilde V(S_1)=V(S_2)=\tilde V(S_2)=V(S_3)=\tilde V(S_3) =\frac{0.5}{1-0.99}=50$ for both MDP and SA-MDP. This can also be seen as an extreme case of Theorem~\ref{thm:optimal_distance}, where the policy does not change under adversary in all states, so there is no performance loss in SA-MDP.
    \item \textbf{Deterministic Policies for SA-MDP.} Now we consider all $2^3=8$ possible deterministic policies for SA-MDP. Note that if for any state $S_i$ we have $p_{i1}=0$ and another state $S_j$ we have $p_{j1}=1$, we always have $\tilde V(S_1)=\tilde V(S_2)=\tilde V(S_3)=0$. This is because we can set $\nu(S_1)=S_j$, $\nu(S_2)=S_i$ and $\nu(S_3)=S_i$ and always receive a 0 reward. Thus the only two possible other policies are $p_{11}=p_{21}=p_{31}=0$ and $p_{11}=p_{21}=p_{31}=1$, respectively. For $p_{11}=p_{21}=p_{31}=1$ we have $\tilde V(S_1)=0, \tilde V(S_2)=\tilde V(S_3)=100$ as we always take $A_1$ and never transit to other states; for $p_{11}=p_{21}=p_{31}=0$, we circulate through all three states and only receive a reward when we leave $A_1$. We have $\tilde V(S_1)=\frac{1}{1-\gamma^3} \approx 33.67$, $\tilde V(S_2)=\frac{\gamma^2}{1-\gamma^3} \approx 33.00$ and $\tilde V(S_3) = \frac{\gamma}{1-\gamma^3} \approx 33.33$.
\end{itemize}
Figure~\ref{fig:value_p11_eq_0}, \ref{fig:value_p11_eq_05}, \ref{fig:value_p11_eq_1} give the graphs of $\tilde V(S_1)$, $\tilde V(S_2)$ and $\tilde V(S_3)$ under three different settings of $p_{11}$. The figures are generated using Algorithm~\ref{alg:policy_nu_evaluation}.
\begin{figure}[htbp]
    \centering
    \includegraphics[width=0.8\linewidth]{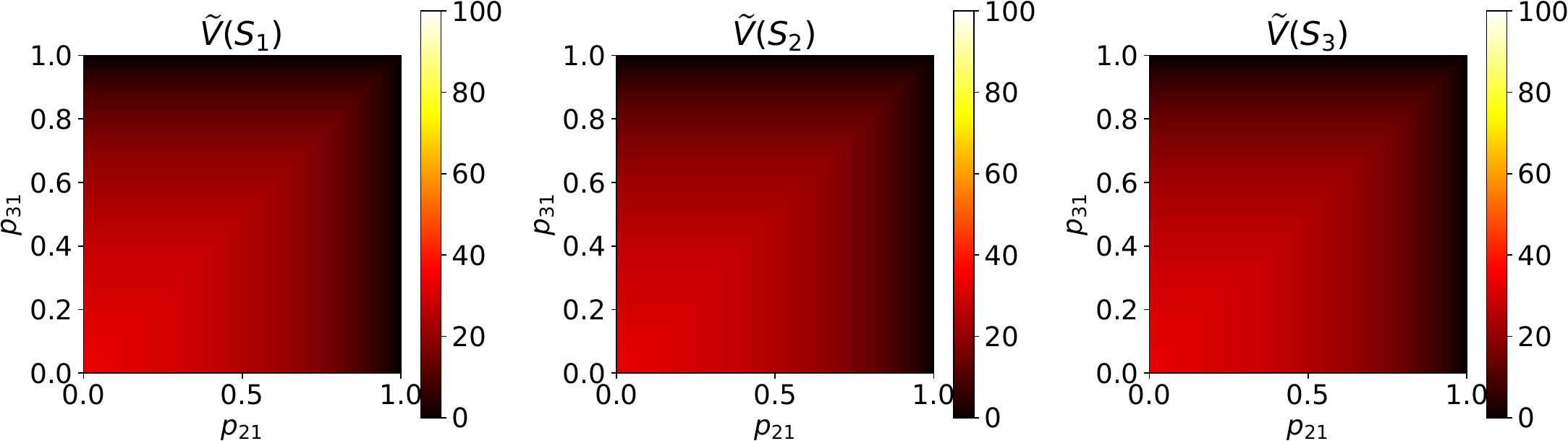}
    \caption{Value functions for SA-MDP when $p_{11}=0$, with $p_{21} \in [0,1]$, $p_{31} \in [0,1]$}
    \label{fig:value_p11_eq_0}
\end{figure}
\begin{figure}[htbp]
    \centering
    \includegraphics[width=0.8\linewidth]{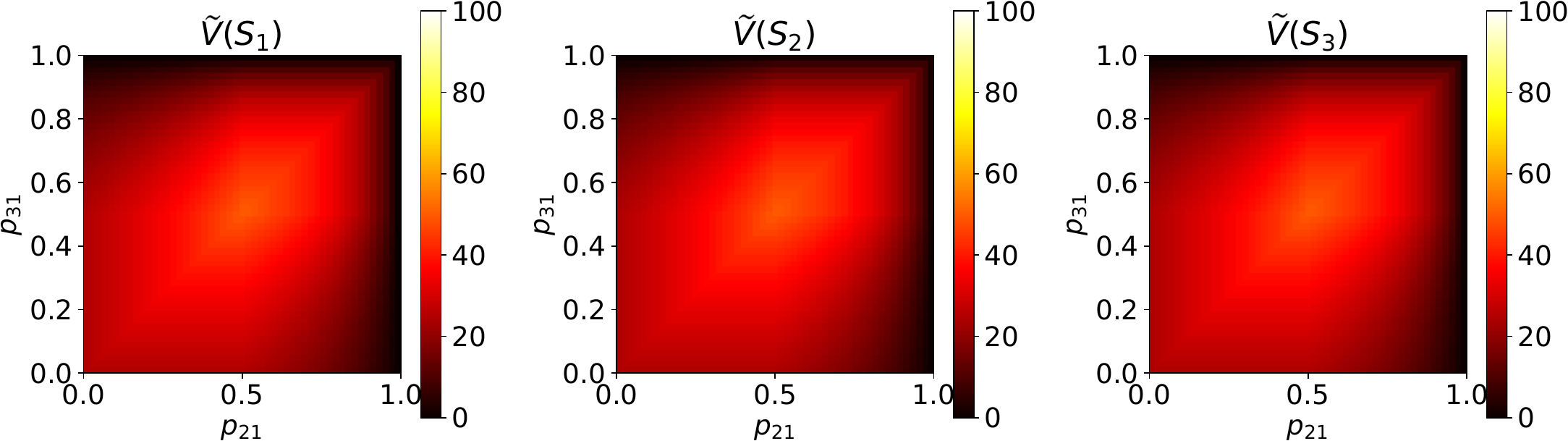}
    \caption{Value functions for SA-MDP when $p_{11}=0.5$, with $p_{21} \in [0,1]$, $p_{31} \in [0,1]$}
    \label{fig:value_p11_eq_05}
\end{figure}
\begin{figure}[htbp]
    \centering
    \includegraphics[width=0.8\linewidth]{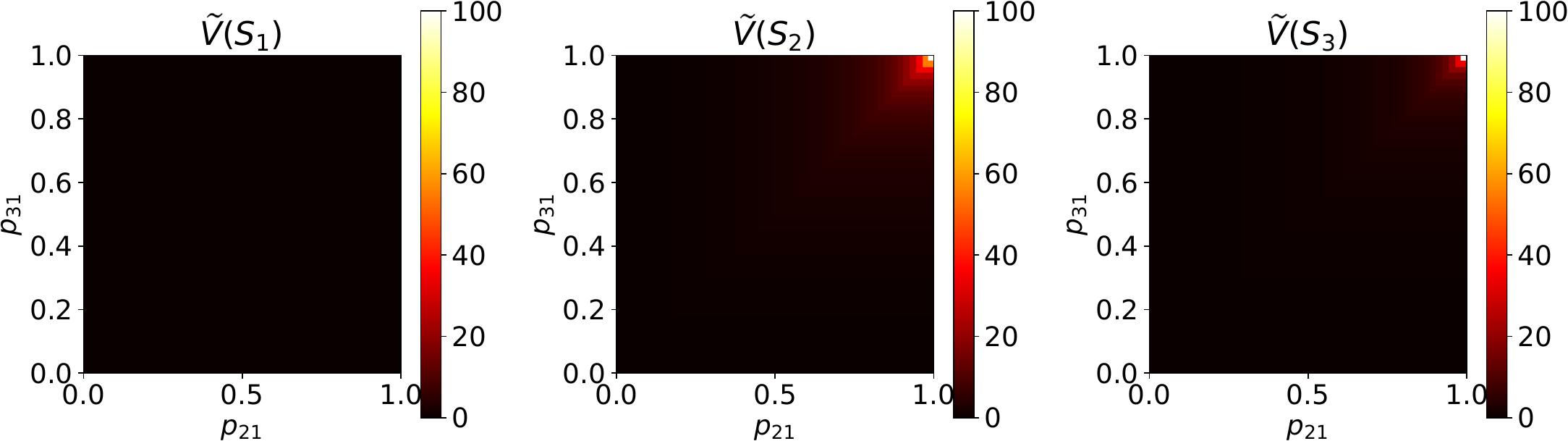}
    \caption{Value functions for SA-MDP when $p_{11}=1.0$, with $p_{21} \in [0,1]$, $p_{31} \in [0,1]$}
    \label{fig:value_p11_eq_1}
\end{figure}

\section{Proofs for State-Adversarial Markov Decision Process}
\label{sec:proofs}
\setcounter{theorem}{0}

\begin{theorem}[Bellman equations for fixed $\pi$ and $\nu$]
Given $\pi: \mathcal{S} \rightarrow \mathcal{P}(\mathcal{A})$ and $\nu: \mathcal{S} \rightarrow \mathcal{S}$, we have
\begin{align*}
\tilde{V}_{\pi\circ\nu}(s) &= \sum_{a \in \mathcal{A}} \pi(a|\nu(s)) \sum_{s' \in \mathcal{S}} p(s'|s,a) \left [R(s, a, s') + \gamma \tilde{V}_{\pi\circ\nu}(s^\prime) \right ]\\
\tilde{Q}_{\pi\circ\nu}(s,a) &= \sum_{s' \in \mathcal{S}} p(s'|s,a) \left [R(s, a, s') + \gamma \sum_{a' \in \mathcal{A}} \pi(a'|\nu(s')) \tilde{Q}_{\pi\circ\nu}(s', a') \right ].
\end{align*}
\vspace{-1em}
\label{thm:bellman_nu}
\end{theorem}

\begin{proof}
Based on the definition of $\tilde{V}_{\pi \circ \nu}(s)$:
\begin{equation}
\begin{aligned}
\tilde{V}_{\pi \circ \nu}(s) &= \E_{\pi\circ\nu} \left [ \sum_{k=0}^{\infty} \gamma^k r_{t+k+1} | s_t = s \right ] \\
&= \E_{\pi\circ\nu} \left [ r_{t+1} + \gamma \sum_{k=0}^{\infty} \gamma^{k} r_{t+k+2} | s_t = s \right ] \\
&= \sum_{a \in \mathcal{A}} \pi(a|\nu(s)) \sum_{s^\prime \in \mathcal{S}} p(s^\prime|s,a) \left [ r_{t+1} + \gamma \E_{\pi\circ\nu} \left [ \sum_{k=0}^{\infty} \gamma^k r_{t+k+2} | s_{t+1} = s^\prime \right ] \right ] \\
&= \sum_{a \in \mathcal{A}} \pi(a|\nu(s)) \sum_{s^\prime \in \mathcal{S}} p(s^\prime|s,a) \left [R(s, a, s^\prime) + \gamma \tilde{V}_{\pi \circ \nu}(s^\prime) \right ]
\end{aligned}
\label{eq:tildeV}
\end{equation}
The recursion for $\tilde{Q}_{\pi \circ \nu}(s,a)$ can be derived similarly. Additionally, we note the following useful relationship between $\tilde{V}_{\pi \circ \nu}(s)$ and $\tilde{Q}_{\pi \circ \nu}(s,a)$:
\begin{equation}
\label{eq:q_and_v}
\tilde{V}_{\pi \circ \nu}(s) = \sum_{a \in \mathcal{A}} \pi(a|\nu(s)) \tilde{Q}_{\pi \circ \nu}(s,a)
\end{equation}
\end{proof}

Before starting to prove Theorem~\ref{thm:optimal_adversary}, first we show that finding the optimal adversary $\nu^*$ given a fixed $\pi$ for a SA-MDP can be cast into the problem of finding an optimal policy in a regular MDP. 

\begin{lemma}[Equivalence of finding optimal adversary in SA-MDP and finding optimal policy in MDP]
\label{lemma:equivalence_optimal}
Given an SA-MDP $M=(\mathcal{S},\mathcal{A},B,R,p,\gamma)$ and a fixed policy $\pi$, there exists a MDP $\hat{M} = (\mathcal{S},\hat{\mathcal{A}},\hat{R},\hat{p}, \gamma)$ such that the optimal policy of $\hat{M}$ is the optimal adversary $\nu$ for SA-MDP given the fixed $\pi$.
\end{lemma}
\begin{proof}
For an SA-MDP $M=(\mathcal{S},\mathcal{A},B,R,p,\gamma)$ and a fixed policy $\pi$, we define a regular MDP $\hat{M} = (\mathcal{S},\hat{\mathcal{A}},\hat{R},\hat{p}, \gamma)$ such that $\hat{\mathcal{A}}=\mathcal{S}$, and $\nu$ is the policy for $\hat{M}$. 
To prove this lemma, we use a slight extension of a stochastic adversary, where $\nu: \mathcal{S} \rightarrow \mathcal{P}(\hat{A})$.
At each state $s$, our policy $\nu$ gives a probability distribution $\nu(\cdot|s)$ indicating that we perturb a state $s$ to $\hat{s}$ with probability $\nu(\hat{s}|s)$ in the SA-MDP $M$.

For $\hat{M}$, the reward function is defined as:
\begin{equation}
\label{eq:eqv_reward}
\hat{R}(s,\hat{a},s^\prime)=
\begin{cases}
-\frac{\sum_{a\in\mathcal{A}}\pi(a|\hat{a})p(s^\prime|s,a)R(s,a,s^\prime)}{\sum_{a\in\mathcal{A}} \pi(a|\hat{a})p(s^\prime|s, a)} &\text{for }s,s^\prime\in\mathcal{S}\text{ and }\hat{a}\in B(s)\subset\hat{\mathcal{A}}=\mathcal{S},\\
C&\text{for }s,s^\prime\in\mathcal{S}\text{ and }\hat{a}\notin B(s).
\end{cases}\end{equation}
The transition probability $\hat{p}$ is defined as
$$\hat{p}(s^\prime|s,\hat{a})=\sum_{a\in\mathcal{A}}\pi(a|\hat{a})p(s^\prime|s,a)\quad\text{for }s,s^\prime\in\mathcal{S}\text{ and }\hat{a}\in\hat{\mathcal{A}}=\mathcal{S}.$$
For the case of $\hat{a}\in B(s)$, the above reward function definition is based on the intuition that when the agent receives a reward $r$ at a time step given $s, a, s^\prime$, the adversary's reward is $\hat{r}=-r$. Note that we consider $r$ as a random variable given $s, a, s^\prime$. To give the distribution of rewards for adversary $p(\hat{r}|s, \hat{a}, s^\prime)$, we follow the conditional probability which marginalizes $\pi$:
\begin{align*}
    p(\hat{r}|s, \hat{a}, s^\prime) &= \frac{p(\hat{r}, s^\prime|s, \hat{a})}{p(s^\prime|s, \hat{a})}  \\
    &= \frac{\sum_a p(\hat{r}, s^\prime | a, s, \hat{a}) \pi(a|s, \hat{a})}{\sum_a p(s^\prime|a, s, \hat{a}) \pi(a|s, \hat{a})} \\
    &= \frac{\sum_a p(\hat{r}, s^\prime | a, s) \pi(a|\hat{a})}{\sum_a p(s^\prime|a, s) \pi(a|\hat{a})} \\
    &= \frac{\sum_a p(\hat{r} |s^\prime, a, s) p(s^\prime | a, s) \pi(a|\hat{a})}{\sum_a p(s^\prime|a, s) \pi(a|\hat{a})} \numberthis\label{eq:expectation_r}
\end{align*}
Considering that $R(s,a,s^\prime) := \E [r |s^\prime, a, s] = -\E [\hat{r} |s^\prime, a, s]$, and taking an expectation in Eq.~\eqref{eq:expectation_r} over $\hat{r}$ yield the first case in~\eqref{eq:eqv_reward}:
\begin{align*}
\hat{R}(s,\hat{a},s^\prime) & \coloneqq \E [\hat{r} | s, \hat{a}, s^\prime]\\
&= \sum_{\hat{r}} \hat{r} \frac{\sum_a p(\hat{r} |s^\prime, a, s) p(s^\prime | a, s) \pi(a|\hat{a})}{\sum_a p(s^\prime|a, s) \pi(a|\hat{a})}\\
&= \frac{\sum_a \left [ \sum_{\hat{r}} \hat{r} p(\hat{r} |s^\prime, a, s) \right ] p(s^\prime | a, s) \pi(a|\hat{a})}{\sum_a p(s^\prime|a, s) \pi(a|\hat{a})}\\
&= \frac{\sum_a \E [\hat{r} |s^\prime, a, s] p(s^\prime | a, s) \pi(a|\hat{a})}{\sum_a p(s^\prime|a, s) \pi(a|\hat{a})} \\
&= -\frac{\sum_a R(s,a,s^\prime) p(s^\prime | a, s) \pi(a|\hat{a})}{\sum_a p(s^\prime|a, s) \pi(a|\hat{a})} \numberthis\label{eq:adv_reward}
\end{align*}

The reward for adversary's actions outside $B(s)$ is a constant $C$ such that
$$C<\min\big\{-\overline{M},\quad\frac{\gamma}{(1-\gamma)}\underline{M}-\frac{1}{(1-\gamma)}\overline{M}\big\},$$ where $\underline{M}:=\min_{s,a,s^\prime}R(s,a,s^\prime)$ and $\overline{M}:=\max_{s,a,s^\prime}R(s,a,s^\prime)$. We have for $\forall(s, \hat{a},s^\prime)$, $$C<\hat{R}(s,\hat{a},s^\prime)\leq-\underline{M},$$
and for $\forall\hat{a}\in B(s)$, according to Eq.~\eqref{eq:adv_reward},
$$-\overline{M}\leq\hat{R}(s,\hat{a},s^\prime)\leq-\underline{M}.$$

According basic properties of MDP~\citep{puterman2014markov,sutton1998introduction}, we know that the $\hat{M}$ has an optimal policy $\nu^*$, which satisfies
$\hat{V}_{\pi \circ \nu^*}(s)\geq \hat{V}_{\pi \circ \nu}(s)$ for $\forall s$, $\forall \nu$. We also know that this $\nu^*$ is deterministic and assigns a unit mass probability for the optimal action of each $s$. 

We define $\mathfrak{N}:=\{\nu:\ \forall s,\ \exists\hat{a}\in B(s),\ \nu(\hat{a}|s)=1\}$ which restricts the adversary from taking an action not in $B(s)$, and claim that $\nu^*\in \mathfrak{N}$. If this is not true for a state $s^0$, we have 
\begin{align*}
    \hat{V}_{\pi \circ \nu^*}(s^0)&=\E_{\hat{p},\nu^*}\Big[\sum_{k=0}^\infty\gamma^k\hat{r}_{t+k+1}|s_t=s^0\Big]\\
    &=C+\E_{\hat{p},\nu^*} \left [ \sum_{k=1}^{\infty} \gamma^k \hat{r}_{t+k+1} | s_t = s^0 \right ]\\
    &\leq C-\frac{\gamma}{1-\gamma}\underline{M}\\
    &<-\frac{1}{1-\gamma}\overline{M}\\
    &\leq \E_{\hat{p},\nu^\prime}\Big[\sum_{k=0}^\infty\gamma^k\hat{r}_{t+k+1}|s_t=s^0\Big]=\hat{V}_{\pi \circ \nu^\prime}(s^0),
\end{align*}
where the second equality holds because $\nu^*$ is deterministic, and the last inequality holds for any $\nu^\prime\in\mathfrak{N}$. This contradicts the assumption that $\nu^*$ is optimal. So from now on in this proof we only study policies in $\mathfrak{N}.$

For any policy $\nu\in \mathfrak{N}$ :
\begin{align*}
    \hat{V}_{\pi \circ \nu}(s)&=\E_{\hat{p},\nu}\Big[\sum_{k=0}^\infty\gamma^k\hat{r}_{t+k+1}|s_t=s\Big]\\
    &=\E_{\hat{p},\nu} \left [ \hat{r}_{t+1} + \gamma \sum_{k=0}^{\infty} \gamma^k \hat{r}_{t+k+2} | s_t = s \right ]\\
    &=\sum_{\hat{a} \in \mathcal{S}} \nu(\hat{a}|s) \sum_{s^\prime \in \mathcal{S}} \hat{p}(s^\prime|s,\hat{a}) \left [ \hat{R}(s,\hat{a},s^\prime) + \gamma \E_{\hat{p},\nu} \left [ \sum_{k=0}^{\infty} \gamma^k \hat{r}_{t+k+2} | s_{t+1} = s^\prime \right ] \right ] \\
    &=\sum_{\hat{a} \in \mathcal{S}} \nu(\hat{a}|s) \sum_{s^\prime \in \mathcal{S}} \hat{p}(s^\prime|s,\hat{a}) \left [ \hat{R}(s,\hat{a},s^\prime) + \gamma \hat{V}_{\pi \circ \nu}(s^\prime) \right ]\numberthis
    \end{align*}

 Note that all policies in $\mathfrak{N}$ are deterministic and this class of policies consists $\nu^*$. Also, $\mathfrak{N}$ is consistent with the class of policies studied in Theorem~\ref{thm:bellman_nu}. We denote the deterministic action $\hat{a}$ chosen by a $\nu\in \mathfrak{N}$ at $s$ as $\nu(s)$. Then for $\forall \nu\in\mathfrak{N}$, we have 
\begin{align*}
\hat{V}_{\pi \circ \nu}(s)&=\sum_{s^\prime \in \mathcal{S}} \hat{p}(s^\prime|s,\nu(s)) \left [ \hat{R}(s,\hat{a},s^\prime) + \gamma \hat{V}_{\pi \circ \nu}(s^\prime) \right ] \\
&=\sum_{s^\prime \in \mathcal{S}} \sum_{a\in\mathcal{A}}\pi(a|\hat{a})p(s^\prime|s,a) \left [ -\frac{\sum_{a\in\mathcal{A}}\pi(a|\hat{a})p(s^\prime|s,a)R(s,a,s^\prime)}{\sum_{a\in\mathcal{A}} \pi(a|\hat{a})p(s^\prime|s, a)} + \gamma \hat{V}_{\pi \circ \nu}(s^\prime) \right ]\\
&=\sum_{a\in\mathcal{A}}\pi(a|\nu(s))\sum_{s^\prime \in \mathcal{S}}p(s^\prime|s,a)\left [-R(s,a,s^\prime) + \gamma \hat{V}_{\pi \circ \nu}(s^\prime) \right ],\numberthis
\end{align*}or
\begin{align}
    -\hat{V}_{\pi \circ \nu}(s)&= \sum_{a\in\mathcal{A}}\pi(a|\nu(s))\sum_{s^\prime \in \mathcal{S}}p(s^\prime|s,a)\left [R(s,a,s^\prime) + \gamma (-\hat{V}_{\pi \circ \nu}(s^\prime)) \right ].\label{eq:hatV}
\end{align}
Comparing \eqref{eq:hatV} and \eqref{eq:tildeV}, we know that $-\hat{V}_{\pi \circ \nu}=\tilde{V}_{\pi \circ \nu}$ for any $\nu\in\mathfrak{N}$. The optimal value function $\hat{V}_{\pi \circ \nu^*}$ satisfies:
\begin{align*}
    \hat{V}_{\pi \circ \nu^*}(s)&=\max_{\hat{a}\in B(s)}\sum_{s^\prime \in \mathcal{S}} \hat{p}(s^\prime|s,\hat{a}) \left [ \hat{R}(s,\hat{a},s^\prime) + \gamma \hat{V}_{\pi \circ \nu}(s^\prime) \right ]\\
    &=\max_{s_\nu\in B(s)}\sum_{a\in\mathcal{A}}\pi(a|s_\nu)\sum_{s^\prime \in \mathcal{S}}p(s^\prime|s,a)\left [- R(s,a,s^\prime) + \gamma \hat{V}_{\pi \circ \nu^*}(s^\prime) \right ],\numberthis
\end{align*}
where we denote the action $\hat{a}$ taken at $s$ as $s_\nu$. So for $\nu^*$, since $-\hat{V}_{\pi \circ \nu^*}=\tilde{V}_{\pi \circ \nu^*}$, we have 
\begin{align}
    \tilde{V}_{\pi \circ \nu^*}(s)=\min_{\hat{a}\in B(s)}\sum_{a\in\mathcal{A}}\pi(a|\hat{a})\sum_{s^\prime \in \mathcal{S}}p(s^\prime|s,a)\left [ R(s,a,s^\prime) + \gamma \tilde{V}_{\pi \circ \nu^*}(s^\prime) \right ],
\end{align}
and $\tilde{V}_{\pi \circ \nu^*}(s)\leq \tilde{V}_{\pi \circ \nu}(s)$ for $\forall s$, $\forall \nu\in\mathfrak{N}$. Hence $\nu^*$ is also the optimal $\nu$ for $\tilde{V}_{\pi \circ \nu}$.
\end{proof}

Lemma~\ref{lemma:equivalence_optimal} gives many good properties for the optimal adversary. First, an optimal adversary always exists under the regularity conditions where an optimal policy exists for a MDP. Second, we do not need to consider stochastic adversaries as there always exists an optimal deterministic adversary. Additionally, showing Bellman contraction for finding the optimal adversary can be done similarly as in obtaining the optimal policy in a regular MDP, as shown in the proof of Theorem~\ref{thm:optimal_adversary}.

\reptext{theorem_bellman_optimal}
\begin{proof}
Based on Lemma~\ref{lemma:equivalence_optimal}, this proof is technically similar to the proof of ``optimal Bellman equation'' in regular MDPs, where $\max$ over $\pi$ is replaced by $\min$ over $\nu$. By the definition of $\tilde{V}_{\pi \circ \nu^*}(s)$,
\begin{align*}
\tilde{V}_{\pi \circ \nu^*}(s) &= \min_{\nu} \tilde{V}_{\pi \circ \nu}(s) \\
&=\min_{\nu} \E_{\pi\circ\nu} \left [ \sum_{k=0}^{\infty} \gamma^k r_{t+k+1} | s_t = s \right ] \\
&=\min_{\nu} \E_{\pi\circ\nu} \left [ r_{t+1} + \gamma \sum_{k=0}^{\infty} \gamma^k r_{t+k+2} | s_t = s \right ] \\
&=\min_{\nu}\sum_{a \in \mathcal{A}} \pi(a|\nu(s)) \sum_{s^\prime \in \mathcal{S}} p(s^\prime|s,a) \left [ r_{t+1} + \gamma \E_{\pi\circ\nu} \left [ \sum_{k=0}^{\infty} \gamma^k r_{t+k+2} | s_{t+1} = s^\prime \right ] \right ] \\
&=\min_{s_\nu \in B_\nu(s)}\sum_{a \in \mathcal{A}} \pi(a|s_\nu) \sum_{s^\prime \in \mathcal{S}} p(s^\prime|s,a) \left [ r_{t+1} + \gamma \min_{\nu} \E_{\pi\circ\nu} \left [ \sum_{k=0}^{\infty} \gamma^k r_{t+k+2} | s_{t+1} = s^\prime \right ] \right ] \\
&=\min_{s_\nu \in B_\nu(s)}\sum_{a \in \mathcal{A}} \pi(a|s_\nu) \sum_{s^\prime \in \mathcal{S}} p(s^\prime|s,a) \left [ r_{t+1} + \gamma \tilde{V}_{\pi \circ \nu^*}(s^\prime) \right ]
\end{align*}
This is the Bellman equation for the optimal adversary $\nu^*$; $\nu^*$ is a fixed point of the Bellman operator $\mathscr{L}$.

Now we show the Bellman operator is a contraction. We have, if $\mathscr{L}\tilde{V}_{\pi \circ \nu_1}(s)\geq\mathscr{L}\tilde{V}_{\pi \circ \nu_2}(s)$,
\begin{align*}
&\mathscr{L}\tilde{V}_{\pi \circ \nu_1}(s)-\mathscr{L}\tilde{V}_{\pi \circ \nu_2}(s)\\
&\leq\max_{s_\nu\in B_\nu(s)}\Big\{\sum_{a \in \mathcal{A}} \pi(a|s_\nu) \sum_{s^\prime \in \mathcal{S}} p(s^\prime|s,a) \left [R(s, a, s^\prime) + \gamma \tilde{V}_{\pi\circ\nu_1}(s^\prime) \right ]\\
&-\sum_{a \in \mathcal{A}} \pi(a|s_\nu) \sum_{s^\prime \in \mathcal{S}} p(s^\prime|s,a) \left [R(s, a, s^\prime) + \gamma \tilde{V}_{\pi\circ\nu_2}(s^\prime) \right ]\Big\}\\
&=\gamma\max_{s_\nu\in B_\nu(s)}\sum_{a \in \mathcal{A}} \pi(a|s_\nu) \sum_{s^\prime \in \mathcal{S}} p(s^\prime|s,a)[\tilde{V}_{\pi\circ\nu_1}(s^\prime)-\tilde{V}_{\pi\circ\nu_2}(s^\prime)]\\
&\leq\gamma\max_{s_\nu\in B_\nu(s)}\sum_{a \in \mathcal{A}} \pi(a|s_\nu) \sum_{s^\prime \in \mathcal{S}} p(s^\prime|s,a)\lVert\tilde{V}_{\pi\circ\nu_1}-\tilde{V}_{\pi\circ\nu_2}\rVert_\infty\\
&=\gamma\lVert\tilde{V}_{\pi\circ\nu_1}-\tilde{V}_{\pi\circ\nu_2}\rVert_\infty
\end{align*}
The first inequality comes from the fact that $$\min_{x_1} f(x_1) - \min_{x_2} g(x_2)\leq f(x_2^*)-g(x_2^*) \leq \max_x (f(x) - g(x)),$$
where $x_2^*=\argmin_{x_2}g(x_2)$. Similarly, we can prove $\mathscr{L}\tilde{V}_{\pi \circ \nu_2}(s)-\mathscr{L}\tilde{V}_{\pi \circ \nu_1}(s)\leq \lVert\tilde{V}_{\pi\circ\nu_1}-\tilde{V}_{\pi\circ\nu_2}\rVert_\infty$ if $\mathscr{L}\tilde{V}_{\pi \circ \nu_2}(s)>\mathscr{L}\tilde{V}_{\pi \circ \nu_1}(s).$ Hence $$\lVert\mathscr{L}\tilde{V}_{\pi \circ \nu_1}(s)-\mathscr{L}\tilde{V}_{\pi \circ \nu_2}(s)\rVert_\infty=\max_s|\mathscr{L}\tilde{V}_{\pi \circ \nu_1}(s)-\mathscr{L}\tilde{V}_{\pi \circ \nu_2}(s)|\leq\gamma\lVert\tilde{V}_{\pi\circ\nu_1}-\tilde{V}_{\pi\circ\nu_2}\rVert_\infty.$$ Then according to the Banach fixed-point theorem, since $0<\gamma<1$, $\tilde{V}_{\pi \circ \nu}$ converges to a unique fixed point, and this fixed point is $\tilde{V}_{\pi \circ \nu^*}$.

\end{proof}

\begin{algorithm}[htbp]
\caption{Policy Evaluation for an SA-MDP $(\mathcal{S},\mathcal{A},B,R,p,\gamma)$}
\label{alg:policy_nu_evaluation}
\begin{algorithmic}
\REQUIRE{Policy $\pi$, convergence threshold $\varepsilon$}
\ENSURE{Values for policy $\pi$, detnoted as $\tilde V_{\pi \circ \nu^*}(s)$}
\STATE{Initialize array $V(s) \leftarrow 0$ for all $s \in \mathcal{S}$}
\REPEAT{} 
\STATE{$\Delta \leftarrow 0$}
\FORALL{$s \in \mathcal{S}$}
\STATE{$v \leftarrow \infty, v_0 \leftarrow V(s)$}
\FORALL{$s_\nu \in B(s)$}
\STATE{$v^\prime \leftarrow \sum_{a \in \mathcal{A}} \pi(a|s_\nu) \sum_{s^\prime \in \mathcal{S}} p(s^\prime|s,a) \cdot\left [R(s, a, s^\prime) + \gamma {V}(s^\prime) \right ]$}
\STATE{$v \leftarrow \min(v, v^\prime)$}
\ENDFOR
\STATE{$V(s) \leftarrow v$}
\STATE{$\Delta \leftarrow \max(\Delta, |v_0 - V(s)|)$}
\ENDFOR
\UNTIL{$\Delta < \varepsilon$}
\STATE{$\tilde V_{\pi \circ \nu^*}(s)\leftarrow V(s)$}
\end{algorithmic}
\end{algorithm}

A direct consequence of Theorem~\ref{thm:optimal_adversary} is the policy evaluation algorithm (Algorithm~\ref{alg:policy_nu_evaluation}) for SA-MDP, which obtains the values for each state under \emph{optimal} adversary for a fixed policy $\pi$. 
For both Lemma~\ref{lemma:equivalence_optimal} and Theorem~\ref{thm:optimal_adversary}, we only consider a fixed policy $\pi$, and in this setting finding an optimal adversary is not difficult. However, finding an optimal $\pi$ under the optimal adversary is more challenging, as we can see in Section~\ref{sec:example_mdp}, given the white-box attack setting where the adversary knows $\pi$ and can choose optimal perturbations accordingly, an optimal policy for MDP can only receive zero rewards under optimal adversary. We now show two intriguing properties for optimal policies in SA-MDP:

\reptext{theorem_stochastic_better}
\begin{proof}
Proof by giving a counter example that no deterministic policy can be better than a random policy. The SA-MDP example in section~\ref{sec:example_mdp} provided such a counter example: all 8 possible deterministic policies are no better than the stochastic policy $p_{11}=p_{21}=p_{31}=0.5$.
\end{proof}


\reptext{theorem_non_optimal}
\begin{proof}
We will show that the SA-MDP example in section~\ref{sec:example_mdp} does not have an optimal policy. First, for $\pi_1$ where $p_{11}=p_{21}=p_{31}=1$ we have $\tilde V_{\pi_1 \circ \nu^*(\pi_1)}(S_1)=0, \tilde V_{\pi_1 \circ \nu^*(\pi_1)}(S_2)=\tilde V_{\pi_1 \circ \nu^*(\pi_1)}(S_3)=100$. This policy is not an optimal policy since we have $\pi_2$ where $p_{11}=p_{21}=p_{31}=0.5$ that can achieve $\tilde V_{\pi_2 \circ \nu^*(\pi_2)}(S_1)=\tilde V_{\pi_2  \circ \nu^*(\pi_2)}(S_2)=\tilde V_{\pi_2  \circ \nu^*(\pi_2)}(S_3)=50$ and $\tilde V_{\pi_2  \circ \nu^*(\pi_2)}(S_1) > \tilde V_{\pi_1 \circ \nu^*(\pi_1)}(S_1)$.

An optimal policy $\pi$, if exists, must be better than $\pi_1$ and have $\tilde V_{\pi \circ \nu^*(\pi)}(S_1)>0, V_{\pi \circ \nu^*(\pi)}(S_2)=V_{\pi \circ \nu^*(\pi)}(S_3)=100$. In order to achieve $V_{\pi \circ \nu^*(\pi)}(S_2)=V_{\pi \circ \nu^*(\pi)}(S_3)=100$, we must set $p_{21}=p_{31}=1$ since it is the only possible way to start from $S_2$ and $S_3$ and receive +1 reward for every step. We can still change $p_{11}$ to probabilities other than 1, however if $p_{11}<1$ the adversary can set $\nu(S_2)=\nu(S_3)=S_1$ and reduce $V_{\pi \circ \nu^*(\pi)}(S_2)$ and $V_{\pi \circ \nu^*(\pi)}(S_3)$. Thus, no policy better than $\pi_1$ exists, and since $\pi_1$ is not an optimal policy, no optimal policy exists.
\end{proof}

Theorem~\ref{eq:stochastic_can_be_better} and Theorem~\ref{eq:non_exist_optimal} show that the classic definition of optimality is probably not suitable for SA-MDP. Further works can study how to obtain optimal policies for SA-MDP under some alternative definition of optimality, or using a more complex policy class (e.g., history dependent policies).

\begin{theorem}
\label{thm:optimal_distance}
Given a policy $\pi$ for a non-adversarial MDP and its value function is $V_\pi(s)$. Under the optimal adversary $\nu$ in SA-MDP, for all $s \in \mathcal{S}$ we have
\begin{equation}
\label{eq:difference_tv}
\max_{s\in\mathcal{S}}\big\{V_\pi(s) - \tilde{V}_{\pi \circ \nu^* (\pi)}(s) \big\}\leq \alpha \max_{s\in\mathcal{S}}\max_{\hat{s} \in B(s)} \mathrm{D}_{\mathrm{TV}}(\pi(\cdot|s),\pi(\cdot|\hat{s}))
\end{equation}
where $\mathrm{D}_{\mathrm{TV}}(\pi(\cdot|s),\pi(\cdot|\hat{s}))$ is the total variation distance between $\pi(\cdot|s)$ and $\pi(\cdot|\hat{s})$, and $\alpha:=2[1+\frac{\gamma}{(1-\gamma)^2}]\max_{(s,a,s^\prime)\in \mathcal{S}\times\mathcal{A}\times\mathcal{S}}|R(s,a,s^\prime)|$ is a constant that does not depend on $\pi$.
\end{theorem}
\begin{proof}
Our proof is based on Theorem 1 in~\citet{achiam2017constrained}. In fact, many works in the literature have proved similar results under different scenarios~\cite{kakade2002approximately,pirotta2013safe}. For an arbitrary starting state $s_0$ and two arbitrary policies $\pi$ and $\pi^\prime$, Theorem 1 in~\citet{achiam2017constrained} gives
an upper bound of $V_\pi(s_0)-V_{\pi^\prime}(s_0)$. The bound is given by
\begin{align}
\begin{split}
V_\pi(s_0)-V_{\pi^\prime}(s_0)&\leq-\mathbb{E}_{\substack{s\sim d^\pi_{s_0}\\a\sim\pi(\cdot|s)\\s'\sim p(\cdot|a,s)}}\Big[\big(\frac{\pi^\prime(a|s)}{\pi(a|s)}-1\big)R(s,a,s^\prime)\Big]\\&+\frac{2\gamma}{(1-\gamma)^2}\max_s\Big\{\mathbb{E}_{\substack{a\sim\pi^\prime(\cdot|s)\\s^\prime\sim p(\cdot|a,s)}}\big[R(s,a,s^\prime)\big]\Big\}\mathbb{E}_{s\sim d_{s_0}^\pi}\big[\mathrm{D}_{TV}(\pi(\cdot|s),\pi^\prime(\cdot|s))\big],
\label{eq:cite_bound}
\end{split}
\end{align}
where $d_{s_0}^\pi$ is the discounted future state distribution from $s_0$, defined as
\begin{align}
    d_{s_0}^\pi(s):=(1-\gamma)\sum_{t=0}^\infty\gamma^t\mathrm{Pr}(s_t=s|\pi,s_0).
\end{align}


Note that in Theorem 1 of~\citet{achiam2017constrained}, the author proved a general form with an arbitrary function $f$ and we assume $f\equiv 0$ in our proof. We also assume the starting state is deterministic, so $J^\pi$ in~\citet{achiam2017constrained} is replaced by $V^\pi(s_0).$ Then we simply need to bound both terms on the right hand side of~\eqref{eq:cite_bound}.

For the first term we know that 
\begin{align}
\begin{split}
    -\mathbb{E}_{\substack{s\sim d^\pi_{s_0}\\a\sim\pi(\cdot|s)\\s'\sim p(\cdot|a,s)}}\Big[\big(\frac{\pi^\prime(a|s)}{\pi(a|s)}-1\big)R(s,a,s^\prime)\Big]&=\sum_{s}d_{s_0}^\pi(s)\sum_a\big[\pi(a|s)-\pi^\prime(a|s)\big]\sum_{s^\prime}p(s^\prime|s,a)R(s,a,s^\prime)\\
    &\leq \sum_{s}d_{s_0}^\pi(s)\sum_{a}\big|\pi(a|s)-\pi^\prime(a|s)\big|\big|\sum_{s^\prime}p(s^\prime|s,a)R(s,a,s^\prime)\big|\\
    &\leq \max_{s,a,s^\prime}|R(s,a,s^\prime)|\max_s\big\{\sum_{a}\big|\pi(a|s)-\pi^\prime(a|s)\big|\big\}\\
    &=2\max_{s,a,s^\prime}|R(s,a,s^\prime)|\max_s\mathrm{D}_{TV}(\pi(\cdot|s),\pi^\prime(\cdot|s))
\end{split}
\end{align}

The second term is bounded by 
\begin{align}
\begin{split}
\frac{2\gamma}{(1-\gamma)^2}\max_s\Big\{\mathbb{E}_{\substack{a\sim\pi^\prime(\cdot|s)\\s^\prime\sim p(\cdot|a,s)}}\big[R(s,a,s^\prime)\big]\Big\}\mathbb{E}_{s\sim d_{s_0}^\pi}\big[\mathrm{D}_{TV}(\pi(\cdot|s),\pi^\prime(\cdot|s))\big]&\\\leq\frac{2\gamma}{(1-\gamma)^2}\max_{s,a,s^\prime}|R(s,a,s^\prime)|
\max_s\mathrm{D}_{TV}(\pi(\cdot|s),\pi^\prime(\cdot|s))&
\end{split}
\end{align} 
Therefore, the RHS of~\eqref{eq:cite_bound} is bounded by $\alpha\max_s\mathrm{D}_{TV}(\pi(\cdot|s),\pi^\prime(\cdot|s))$, where
\begin{align}
    \alpha=2[1+\frac{\gamma}{(1-\gamma)^2}]\max_{s,a,s^\prime}|R(s,a,s^\prime)|
\end{align}

Finally, we simply let $\pi^\prime(\cdot|s):=\pi(\cdot|\nu^*(s))$ and the proof is complete.
\end{proof}

Before proving Theorem~\ref{thm:policy_distance} we first give a technical lemma about the total variation distance between two multi-variate Gaussian distributions with the same variance.

\begin{lemma}
\label{lemma:normal_distance}
Given two multi-variate Gaussian distributions $X_1 \sim \mathcal{N}(\mu_1, \sigma^2 I_n)$ and $X_2 \sim \mathcal{N}(\mu_2, \sigma^2 I_n)$, $\mu_1, \mu_2 \in \R^n$, define $d = \|\mu_2 - \mu_1\|_2$. We have $\mathrm{D}_{TV}(X_1, X_2) = \sqrt{\frac{2}{\pi}}\frac{d}{\sigma} + O(d^3)$.
\end{lemma}

\begin{proof}
Denote probability density of $X_1$ and $X_2$ as $f_1$ and $f_2$, and denote $a = \frac{\mu_2 - \mu_1}{d}$ as the normal vector of the perpendicular bisector line between $\mu_1$ and $\mu_2$. Due to the symmetry of Gaussian distribution, $f_1(x) - f_2(x)$ is positive for all $x$ where $a^\top x - a^\top \mu_1 - \frac{d}{2} >0$ and negative for all $x$ on the other symmetric side. When $a^\top x - a^\top \mu_1 - \frac{d}{2} >0$, $\int_{x \in \R^n} [f_1(x) - f_2(x)] \mathrm{d}x = \Phi(\frac{d}{2\sigma}) - (1 - \Phi(\frac{d}{2\sigma})) = 2 \Phi(\frac{d}{2\sigma}) - 1$. Thus,
\begin{align*}
D_{TV}(X_1, X_2) &= \int_{x \in \R^n} | f_1(x) - f_2(x) | \mathrm{d}x \\
&= 2 \int_{a^\top x - a^\top \mu_1 - \frac{d}{2} >0} (f_1(x) - f_2(x)) \mathrm{d}x \\
&= 2(\Phi(\frac{d}{2\sigma}) - (1-\Phi(\frac{d}{2\sigma}))) \\
&= 2(2\Phi(\frac{d}{2\sigma}) - 1)
\end{align*}

Then we use the Taylor series for $\Phi(x)$ at $x=0$:

\begin{equation*}
    \Phi(x) = \frac{1}{2} + \frac{1}{\sqrt{2\pi}}\sum_{n=0}^\infty \frac{(-1)^n x^{2n+1}}{2^n n! (2n+1)}
\end{equation*}

Since we consider the case where $d$ is small, we only keep the first order term and obtain:

\begin{equation*}
    D_{TV}(X_1, X_2) = \sqrt{\frac{2}{\pi}} \frac{d}{\sigma} + O(d^3)
\end{equation*}

\end{proof}

\reptext{theorem_policy_distance}
\begin{proof}
This theorem is a special case of Lemma~\ref{lemma:normal_distance} where $X_1=\bar{\pi}(\cdot | s)$, $X_2=\bar{\pi}(\cdot | s')$ and $X_1 \sim \mathcal{N}(\pi(s), \sigma^2 I)$, $X_2 \sim \mathcal{N}(\pi(s'), \sigma^2 I)$.
\end{proof}

\section{Optimization Techniques}
\label{sec:optimization}

\subsection{More Backgrounds for Convex Relaxation of Neural Networks}
\label{sec:convex-relaxation}

In our work, we frequently need to solve a minimax problem:
\begin{equation}
\min_\theta \max_{\phi \in \mathbb{S}} g(\theta, \phi)
\label{eq:minimax}
\end{equation}
One approach we will discuss is to first solve the inner maximization problem (approximately) using an optimizer like SGLD. However, due to the non-convexity of $\pi_\theta$, we cannot solve the inner maximization to global maxima, and the gap between local maxima and global maxima can be large. Using convex relaxations of neural networks, we can instead find an upper bound of $\max_{\phi \in \mathbb{S}} g(\theta, \phi)$:
\[
\overline{g}(\theta) \geq \max_{\phi \in \mathbb{S}} g(\theta, \phi)
\]
Thus we can minimize an upper bound instead, which can guarantee the original objective~\eqref{eq:minimax} is minimized.

As an illustration on how to find $\overline{g}(\theta)$ using convex relaxations, following~\citet{salman2019convex} we consider a simple $L$-layer MLP network $f(\theta, x)$ with parameters $\theta = \{(W^{(i)}, b^{(i)}), i \in \{1, \cdots, L\}\}$ and activation function $\sigma$. We denote $x^{(0)}=x$ as the input, $x^{(i)}$ as the post-activation value for layer $i$, $z^{(i)}$ as the pre-activation value for layer $i$. $i \in \{1, \cdots, L\}$. The output of the network $f(\theta, x)$ is $z^{(L)}$. Then, we consider the following optimization problem:
\[
\max_{x \in \mathbb{S}} f(\theta, x), \quad \text{where $\mathbb{S}$ is the set of perturbations}
\]
which is equivalent to the following optimization problem:
\begin{equation}
\begin{aligned}
    \max \quad & z^{(L)} \\
    \text{s.t.}\quad & z^{(l)} = W^{(l)} x^{(l-1)} + b^{(l)}, l \in [L], \\
    & x^{(l)}=\sigma(z^{(l)}), l \in [L-1], \\
    & x^{(0)} \in \mathbb{S}
\end{aligned}
\label{eq:constraint_verification_opt}
\end{equation}

In this constrained optimization problem~\eqref{eq:constraint_verification_opt}, assuming $\mathbb{S}$ is a convex set, the constraint on $z^{(l)}$ is convex (linear) and the only non-convex constraints are those for $x^{(l)}, l=\{1, \cdots, L-1\}$, where a non-linear activation function is involved. Note that activation function $\sigma(z)$ itself can be a convex function, but when used as an equality constraint, the feasible solution is constrained to the \emph{graph} of $\sigma(z)$, which is non-convex.

Previous works~\citep{wong2018provable,zhang2018efficient,salman2019convex} propose to use convex relaxations of non-linear units to relax the non-convex constraint $x^{(l)}=\sigma(z^{(l)})$ with a convex one, $x^{(l)}=\text{convex}(\sigma(z^{(l)}))$, such that~\eqref{eq:constraint_verification_opt} can be solved efficiently. We can then obtain an \emph{upper bound} of $f(\theta, x)$ since the constraints are relaxed.

\citet{zhang2018efficient} gave several concrete examples (e.g., ReLU, tanh, sigmoid) on how these relaxations are formed. In the special case where linear relaxations are used, \eqref{eq:constraint_verification_opt} can be solved efficiently and automatically (without manual derivation and implementation) for general computational graphs~\citep{xu2020automatic}. Generally, using the framework from~\citet{xu2020automatic} we can access an oracle function ConvexRelaxUB defined as below:

\begin{definition}
Given a neural network function $f(\mathbf{X})$ where $\mathbf{X}$ is any input for this function, and $\mathbf{X} \in \mathbb{S}$ where $\mathbb{S}$ is the set of perturbations, the oracle function \emph{ConvexRelaxUB} provided by an automatic neural network convex relaxation tool returns an upper bound $\overline{f}$, which satisfies:
\[
\overline{f} \geq \max_{\mathbf{X} \in \mathbb{S}} f(\mathbf{X})
\]
\end{definition}

Note that in the above definition, $\mathbf{X}$ can by \emph{any} input for this computation (e.g., $\mathbf{X}$ can be $s$, $a$, or $\theta$ for a $Q_\theta(s, a)$ function). In the special case of our paper, for simplicity we define the notation $\text{ConvexRelaxUB}(f, \theta, s\!\in\!B(s))$ which returns an upper bound function $\overline{f}(\theta)$ for $\max_{s \in B(s)}f(\theta, s)$.

\paragraph{Computational cost} Many kinds of convex relaxation based methods exist~\citep{salman2019convex}, where the expensive ones (which give a tighter upper bound) can be a few magnitudes slower than forward propagation. The cheapest method is interval bound propagation (IBP), which only incurs twice more costs as forward propagation; however, IBP base training has been reported unstable and hard to reproduce as its bounds are very loose~\citep{zhang2019towards,balunovic2019adversarial}. To avoid potential issues with IBP, in all our environments, we use the IBP+Backward relaxation scheme following~\citep{zhang2019towards,xu2020automatic}, which produces considerably tighter bounds, while being only a few times slower than forward propagation (e.g., 3 times slower than forward propagation when loss fusion~\citep{xu2020automatic} is implemented). In fact, \citet{xu2020automatic} used the same relaxation for training downscaled ImageNet dataset on very large vision models. For DRL the policy neural networks are typically small and can be handled quite efficiently. In our paper, we use convex relaxation as a blackbox tool (provided by the \texttt{auto\_LiRPA} library~\citep{xu2020automatic}), and any new development for improving its efficiency can benefit us.

\subsection{Solving the Robust Policy Regularizer using SGLD}
\label{sec:optimizing_sgld}
Stochastic gradient Langevin dynamics (SGLD)~\cite{gelfand1991recursive} can escape saddle points and shallow local optima in non-convex optimization problems~\citep{raginsky2017non,zhang2017hitting,bubeck2015finite,xu2018global}, and can be used to solve the inner maximization with zero gradient at $\hat{s}=s$. SGLD uses the following update rule to find $\hat{s}^K$ to maximize $\mathcal{R}_s(\hat{s},\theta_\mu)$:
\[
\hat{s}^{k+1} \leftarrow \text{proj}\left ( \hat{s}^{k} - \eta_k \nabla_{\hat{s}^k} \mathcal{R}_s(\hat{s}^k, \theta_\mu) + \sqrt{{2 \eta_k}/{\beta_k}} \xi \right ), \quad \hat{s}^{0}=s, \quad k=0,\cdots,K-1
\]
where $\eta_k$ is step size, $\xi$ is an i.i.d. standard Gaussian random variable in $\R^{|\mathcal{S}|}$, $\beta_k$ is an inverse temperature hyperparameter, and $\text{proj}(\cdot)$ projects the update back into $B(s)$.  
We find that SGLD is sufficient to escape the stationary point at $\hat{s}=s$.  However, due to the non-convexity of $\mu_{\theta_\mu}(\hat{s},\theta_\mu)$,
this approach only provides a lower bound $\mathcal{R}_s(\hat{s}^K,\theta_\mu)$ of $\max_{\hat{s} \in B(s)} \mathcal{R}_s(\hat{s},\theta_\mu)$. Unlike the convex relaxation based approach, minimizing this lower bound does not guarantee to minimize~\eqref{eq:ppo_regularizer}, as the gap between $\max_{\hat{s} \in B(s)} \mathcal{R}_s(\hat{s},\theta_\mu)$ and $\mathcal{R}_s(\hat{s}^K,\theta_\mu)$ can be large.

\paragraph{Computational Cost} In SGLD, we first need to solve the inner maximization problem (such as Eq.~\eqref{eq:ppo_regularizer}). The additional time cost depends on the number of SGLD steps. In our experiments for PPO and DDPG, we find that using 10 steps are sufficient. However, the total training cost does not grow by 10 times, as in many environments the majority of time was spent on environment simulation steps, rather than optimizing a small policy network.

\section{Additional details for adversarial attacks on state observations}
\label{sec:appendix_attack}

\subsection{More details on the Critic based attack}
In Section~\ref{sec:attack} we discuss the critic based attack~\citep{pattanaik2018robust} as a baseline. This attack requires a $Q$ function $Q(s,a)$ to find the best perturbed state. In Algorithm~\ref{alg:critic_attack} we present our ``corrected'' critic based attack based on~\citep{pattanaik2018robust}:

\begin{algorithm}[htbp]
	\caption{Critic based attack~\citep{pattanaik2018robust}}
	\label{alg:critic_attack}
	\begin{algorithmic}
		\REQUIRE{A policy function $\pi$ under attack, a corresponding $Q(s,a)$ network, and a initial state $s^0$, $K$ is the number of attack steps, $\eta$ is the step size, $\underline{s}$ and $\overline{s}$ are valid lower and upper range of $s$ (assuming a $\ell_\infty$ norm-like threat model).}
		\FOR{$k=1$ to $K$}
		\STATE $g^k = \nabla_{s^{k-1}} Q(s_0,\pi(s^{k-1}))=\frac{\partial Q}{\partial \pi}\frac{\partial \pi}{\partial s^{k-1}}$
		\STATE $g^k \leftarrow \mathrm{proj}(g^k)$ \COMMENT{project $g^k$ according to norm constraint of $s$; for $\ell_\infty$ norm simply take the sign}
		\STATE $s^k \leftarrow s^{k-1} - \eta g^k$
		\STATE $s^k \leftarrow \min(\max(s^k, \underline{s}), \overline{s})$ \COMMENT{only needed for $\ell_\infty$ norm threat model}
		\ENDFOR
		\ENSURE{An adversarial state $\hat{s}:=s^{K}$}
	\end{algorithmic}
\end{algorithm}

Note that in Algorithm 4 of~\citep{pattanaik2018robust}, given a state $s^0$ under attack, they use the gradient $\nabla_s Q(s,\pi(s)) = \frac{\partial Q}{\partial s} + \frac{\partial Q}{\partial \pi}\frac{\partial \pi}{\partial s}$ which essentially attempts to minimize $Q(\hat{s}, \pi(\hat{s}))$, but they then sample randomly along this gradient direction to find the best $\hat{s}$ that minimizes $Q(s^0, \pi(\hat{s}))$. Our corrected formulation directly minimizes $Q(s^0, \pi(\hat{s}))$ using this gradient instead $\nabla_s Q(s^0,\pi(s))=\frac{\partial Q}{\partial \pi}\frac{\partial \pi}{\partial s}$.

For PPO, since there is no $Q(s,a)$ available during training, we extend~\citep{pattanaik2018robust} to perform attack relying on $V(s)$: we find a state $\hat{s}$ that minimizes $V(\hat{s})$. Unfortunately, it does not match our setting of perturbing state observations; it looks for a state $\hat{s}$ that has the worst value (i.e., taking action $\pi(\hat{s})$ in state $\hat{s}$ is bad), but taking the action $\pi(\hat{s})$ at state $s^0$ does not necessarily trigger a low reward action, because $V(\hat{s})=\max_a Q(\hat{s},a) \neq \max_a Q(s^0, a)$. Thus, in Table~\ref{tab:ppo_res} we can observe that critic based attack typically does not work very well for PPO agents.

\subsection{More details on the Maximal Action Difference (MAD) attack}
We present the full algorithm of MAD attack in Algorithm~\ref{alg:mad_attack}. It is a relatively simple attack by directly maximizing a KL-divergence using SGLD, yet it usually outperforms random attack and critic attack on many environments (e.g., see Figure~\ref{fig:attack_eps_sweep}).

\begin{algorithm}[htbp]
\caption{Maximal Action Difference (MAD) Attack (a critic-independent attack)}
\label{alg:mad_attack}
\begin{algorithmic}
\REQUIRE{A policy function $\pi$ under attack, and a initial state $s_0$, $T$ is the number of attack steps, $\eta$ is the step size, $\beta$ is the (inverse) temperature parameter for SGLD, $\underline{s}$ and $\overline{s}$ are valid lower and upper range of $s$.}
\STATE Define loss function $L_\text{MAD}(s)=-D_\text{KL}(\pi(\cdot|s_0)\|\pi(\cdot|s))$
\FOR{$t=1$ to $T$}
\STATE Sample $\xi \sim \mathcal{N}(0,1)$
\STATE $g_t = \nabla L_\text{MAD}(s_{t-1}) + \sqrt{\frac{2}{\beta \eta}}\xi$
\STATE $g_t \leftarrow \mathrm{proj}(g_t)$ \COMMENT{project $g_t$ according to norm constraint of $s$; for $\ell_\infty$ norm simply take the sign}
\STATE $s_t \leftarrow s_{t-1} - \eta g_t$
\STATE $s_t \leftarrow \min(\max(s_t, \underline{s}), \overline{s})$
\ENDFOR
\ENSURE{An adversarial state $\hat{s}:=s_T$}
\end{algorithmic}
\end{algorithm}

\subsection{More details on the Robust Sarsa attack}
Algorithm~\ref{alg:sarsa_attack} gives the full procedure of the Robust Sarsa attack. We collect trajectories of the agents and then optimize the ordinary temporal difference (TD) loss along with a robust objective $L_\text{robust}(\theta)$. $L_\text{robust}(\theta)$ constrains that when an input action $a$ is slightly changed, the value $Q^\pi_\text{RS}(s,a)$ should not change significantly. We set the perturbation set $B_p(a, \epsilon)$ to be a $\ell_p$ norm ball with radius $\epsilon$ around an action $a$. We gradually increase $\epsilon$ from 0 to $\epsilon_\text{max}$ during training to learn a critic that is increasingly more robust. The inner maximization of $L_\text{robust}(\theta)$ is upper bounded by convex relaxations of neural networks, which we introduced in section~\ref{sec:convex-relaxation}. Once the inner maximization is eliminated, we solve the final objective using regular first order optimization methods. In our attacks to DDPG and PPO, we try multiple regularization parameter $\lambda_\text{RS}$ to find the best Sarsa model that achieves \emph{lowest} attack rewards.

\begin{algorithm}[htbp]
\caption{Train a robust value function for critic-independent attack (Robust Sarsa attack)}
\label{alg:sarsa_attack}
\begin{algorithmic}
\REQUIRE{Any policy function $\pi$ under attack, $T$ is the number of training steps, and an epsilon schedule $\epsilon_t$}
\STATE Initialize $Q^\pi_\text{RS}(s,a)$ to be a random network
\FOR{$t=1$ to $T$}
\STATE Run the agent with policy $\pi$ and collect a batch of $N$ steps: $\{s_i, a_i, r_i, s_i^\prime, a_i^\prime\}, i \in [N]$
\STATE $L_\text{TD}(\theta) = \sum_{i \in [N]} \left [ r_i + \gamma Q^\pi_\text{RS}(s_i^\prime, a_i^\prime) - Q^\pi_\text{RS}(s_i, a_i) \right ]^2$
\STATE $L_\text{robust}(\theta) = \sum_{i \in [N]} \max_{\hat{a} \in B_p (a_i, \epsilon_t)} (Q^\pi_\text{RS}(s_i, \hat{a}) -  Q^\pi_\text{RS}(s_i, a_i))^2$
\STATE $\overline{L}_\text{robust}=\text{ConvexRelaxUB}(L_\text{robust}, \theta, B_p (a_i, \epsilon_t))$, where $L_\text{robust}(\theta) \leq \overline{L}_\text{robust}(\theta)$ \COMMENT{Solving the inner maximization by upper bounding $L_\text{robust}$ using an automatic NN convex relaxation tool}
\STATE Minimize $L_\text{RS}(\theta)=L_\text{TD}(\theta) + \lambda_\text{RS}\overline{L}_\text{robust}(\theta)$ using any gradient based optimizer (e.g., Adam)
\ENDFOR
\ENSURE A robust critic function $Q^\pi_\text{RS}$ that can be used for Algorithm~\ref{alg:critic_attack}.
\end{algorithmic}
\end{algorithm}

Although it is beyond the scope of this paper, RS attack can also be used as a blackbox attack when perturbing the actions rather than state observations, as $Q^\pi_{\theta_{RS}}$ can be learned by observing the environment and the agent without any internal information of the agent. Then, using the robust critic we learned, black-box attacks can be performed on action space by solving $\min Q^\pi_{\theta_{RS}}(s, a)$ with a norm constrained $a$.

For a practical implementation, to improve convergence and reduce instability, two $Q^\pi_\text{RS}(s,a)$ functions can be also used similarly as in double Q learning~\citep{hasselt2010double}. In our case, since the policy is not being updated and stable, we find that using a single Q function is also sufficient for most settings and usually converges faster.

We provide some empirical justifications for the necessity of using a robust objective. For both PPO and DDPG, we conduct attacks using a Sarsa network trained with and without the robustness objective, in Table~\ref{tab:robust_sarsa_ppo} and Table~\ref{tab:robust_sarsa_ddpg}, respectively. We observe that the robust objective can decrease reward further more in most settings.

\begin{table*}[tbh]\centering
\caption{Comparison between Non-robust Sarsa attack (without the robustness objective $L_\text{robust}(\theta)$) and robust Sarsa attack on PPO and SA-PPO agents in Table~\ref{tab:ppo_res}. The Robust Sarsa Attack Reward column is the same result presented in RS column of Table~\ref{tab:ppo_res}. We report mean reward $\pm$ standard deviation over 50 attack episodes.}
\resizebox{0.8\linewidth}{!}{
\begin{tabular}{l|c|c|c|c}
\hline
\multirow{2}{*}{Env.} & \multirow{2}{*}{\shortstack{$\ell_\infty$ norm perturb-\\ation budget $\epsilon$}}                                  & \multirow{2}{*}{Method}  & \multirow{2}{*}{\shortstack{Non-robust Sarsa\\  Attack Reward}}&\multirow{2}{*}{\shortstack{Robust Sarsa\\Attack Reward}}\\
                      &                                                              &                         &                                   & \\ \hline
                      &                                                              & PPO (vanilla)           &2757.0$\pm$604.2  &  \bf 779.4$\pm$33.2     \\
                      &                                                              & PPO (adv. 50\%)         &276 $\pm$140& \bf 49 $\pm$ 50       \\
                      &                                                              & PPO (adv. 100\%)        &14.4$\pm$ 4.20& \bf 3.8 $\pm$ 0.9   \\
                      &                                                              & SA-PPO (SGLD)           &  3642.9$\pm$4.0 &\bf1403.3$\pm$55.0\\

\multirow{-5}{*}{\cellcolor{white} Hopper}  & \multirow{-5}{*}{\cellcolor{white}0.05} 
& SA-PPO (Convex)         & 3014.9$\pm$656.1 &\bf1235.8$\pm$50.2\\ \hline
                      &                                                              & PPO (vanilla)           &  2224.7$\pm$1438.7&\bf 913.7$\pm$54.3\\
                      &                                                              & PPO (adv. 50\%)         &-10.79 $\pm$ 0.93 & \bf -11.55 $\pm$ 0.79 \\
                      &                                                              & PPO (adv. 100\%)        &-111.9$\pm$ 4.5& \bf -114.4 $\pm$ 4.0      \\
                      &                                                              & SA-PPO (SGLD)           &  4777.1$\pm$305.5&\bf 2605.6$\pm$1255.7\\

\multirow{-5}{*}{\cellcolor{white}Walker2d} & \multirow{-5}{*}{\cellcolor{white}0.05} & SA-PPO (Convex)         & 3701.1$\pm$1013.3 &\bf2168.2$\pm$	665.4\\ \hline
                      &                                                              & PPO (vanilla)           &  \bf 716.4$\pm$166.1&1036.0$\pm$420.2\\
                      &                                                              & PPO (adv. 50\%)         & 166$\pm$ 78& \bf 98 $\pm$ 69       \\
                      &                                                              & PPO (adv. 100\%)        &  122.6$\pm$ 15.9&\bf 113.2 $\pm$ 18.5  \\
                      &                                                              & SA-PPO (SGLD)           &  6115.4$\pm$783.2&\bf 6200.5$\pm$818.1\\
                      
\multirow{-5}{*}{\cellcolor{white}Humanoid} & \multirow{-5}{*}{\cellcolor{white}0.075} & SA-PPO (Convex)        &6241.2$\pm$540.8  &\bf 4707.2$\pm$1359.1\\ \hline
\end{tabular}
}
\label{tab:robust_sarsa_ppo}
\end{table*}

\begin{table*}[tbh]\centering
\caption{Comparison between Non-robust Sarsa attack (without the robustness objective) and robust Sarsa attack on DDPG and SA-DDPG agents in Table~\ref{tab:ddpg_res}. The Robust Sarsa Attack Reward column presents the same results as presented in the RS attack rows of Table~\ref{tab:ddpg_full_res_}. We report mean reward $\pm$ standard deviation over 50 attack episodes.}
\resizebox{0.8\linewidth}{!}{
\begin{tabular}{l|c|c|c|c}
\hline
\multirow{2}{*}{Env.} & \multirow{2}{*}{\shortstack{$\ell_\infty$ norm perturb-\\ation budget $\epsilon$}}                                  & \multirow{2}{*}{Method}  & \multirow{2}{*}{\shortstack{Non-robust Sarsa\\  Attack Reward}}&\multirow{2}{*}{\shortstack{Robust Sarsa\\Attack Reward}}\\
                      &                                                              &                         &                                   & \\ \hline
   &  & DDPG (vanilla) & $700 \pm 305$  & $\bf 336 \pm 283$ \\
   \multirow{-2}{*}{\cellcolor{white} Ant}  &  \multirow{-2}{*}{0.2}& SA-DDPG (Convex) & $2380 \pm 142$  & $\bf 1820 \pm 635$  \\
\hline
  &  & DDPG (vanilla) & $1362 \pm 1468$  & $\bf 606\pm 124$ \\ \multirow{-2}{*}{\cellcolor{white} Hopper}  &  \multirow{-2}{*}{0.075}& SA-DDPG (Convex) & $1323 \pm 491$  & $\bf 1258 \pm 561$ \\
\hline
  &  & DDPG (vanilla) &$1000 \pm 0$  & $\bf 92 \pm 1$  \\ \multirow{-2}{*}{\cellcolor{white} InvertedPendulum}  &\multirow{-2}{*}{0.3}  & SA-DDPG (Convex) & $1000 \pm 0$  & $1000 \pm 0$  \\
\hline
  &  & DDPG (vanilla) & $ -24.11 \pm 7.19$  & $\bf -21.74\pm5.14$ \\ \multirow{-2}{*}{\cellcolor{white} Reacher}  & \multirow{-2}{*}{1.5}& SA-DDPG (Convex) &$\bf -11.67 \pm 3.57$  & $-11.40 \pm 3.56$  \\
\hline
  &  & DDPG (vanilla) & $ 951 \pm 1146$  & $959 \pm 1001$ \\ \multirow{-2}{*}{\cellcolor{white} Walker2d}  &\multirow{-2}{*}{0.05}  & SA-DDPG (Convex) &$3200 \pm 1939$  & $\bf 1986 \pm 1993$ \\
\hline

\end{tabular}
}
\label{tab:robust_sarsa_ddpg}
\end{table*}


\subsection{Hybrid RS+MAD attack}
We find that RS and MAD attack can achieve the best results (lowest attack reward) in many cases. We also consider combining them to form a hybrid attack, which minimizes the robust critic predicted value and in the meanwhile maximizes action differences. It can be conducted by minimizing this combined loss function to find an adversarial state $\hat{s} \in B(s)$:
\[
L_{\text{Hybrid}}(\hat{s}) = \alpha_\text{RS-MAD} Q_{\theta_Q}(s, \pi_{\theta_{RS}}(\hat{s})) + (1 - \alpha_\text{RS-MAD}) L_\text{MAD}(\hat{s})
\]
For a practical implementation, it is important to choose $\alpha_\text{RS-MAD}$ so that the two parts of the loss are roughly balanced. The value of $Q_{\theta_Q}$ depends on environment reward (if reward is not normalized), and might be much larger in magnitudes than $\text{RS-MAD}$, so typically $\alpha_\text{RS-MAD}$ is close to 1.

We try different values of $\alpha_\text{RS-MAD}$ and report the lowest reward as the final reward under this attack.
\subsection{Projected Gradient Decent (PGD) Attack for DQN}
\label{sec:dqn_attack}
For DQN, we use the regular untargeted Projected Gradient Decent (PGD) attack in the  literature~\cite{lin2017tactics,pattanaik2018robust,xiao2019characterizing}. The untargeted PGD attack with $K$ iterations updates the state $K$ times as follows:
\begin{equation}
\begin{gathered}
    s^{k+1} = s^k+ \eta \text{proj} [\nabla_{s^k}\mathcal{H}(Q_\theta(s^k,\cdot),a^*)],\\s^0=s,\quad k=0,\dots,K-1
    \end{gathered}
\end{equation}
where $\mathcal{H}(Q_\theta(s^k,\cdot),a^*)$ is the cross-entropy loss between the output logits of $Q_\theta(s^k,\cdot)$ and the onehot-encoded distribution of $a^*:=\argmax_aQ_\theta(s,a)$. $\text{proj}[\cdot]$ is a projection operator depending on the norm constraint of $B(s)$ and $\eta$ is the learning rate. A successful untargeted PGD attack will then perturb the state to lead the Q network to output an action other than the optimal action $a^*$ chosen at the original state $s$. To guarantee that the final state obtained by the attack is within an $\ell_\infty$ ball around $s$ ($B_\epsilon(s)=\{\hat{s}:s-\epsilon\leq \hat{s}\leq s+\epsilon\}$),  the projection $\text{proj}[\cdot]$ is a sign operator and $\eta$ is typically set to $\eta=\frac{\epsilon}{K}$.

\section{Robustness Certificates for Deep Reinforcement Learning}
\label{sec:certificate}
If we use the convex relaxation in Section~\ref{sec:convex-relaxation} to train our networks, it can produce robustness certificates~\citep{wong2018provable,mirman2018differentiable,zhang2019towards} for our task. However in some RL tasks the certificates have interpretations different from classification tasks, as discussed in detail below.

\textbf{Robustness Certificates for DQN.}
In DQN, the action space is finite, so we have a robustness certificate on the actions taken at each state. More specifically, at each state $s$, policy $\pi$'s action is certified if its corresponding Q function satisfies
\begin{align}
    \argmax_a Q_\theta(s,a)=\argmax_{a}Q_\theta(\hat{s},a)=a^*, \text{for all $\hat{s} \in B(s).$}
\end{align}
Given a states $s$, we can use neural network convex relaxations to compute an upper bound $u_{Q_\theta, a^*,a}(s)$ such that $$Q_\theta(\hat{s},a)-Q_\theta(\hat{s},a^*)\leq u_{Q_\theta, a^*,a}(s)$$ holds for all $ \hat{s} \in B(s)$. So if $u_{Q_\theta, a^*,a}(s)\leq 0$ for all $a\in\mathcal{A}$, we have
\begin{align}
    Q_\theta(\hat{s},a)-Q_\theta(\hat{s},a^*)\leq0
\end{align}
is guaranteed for all $ \hat{s} \in B(s)$, which means that the agent's action will not change when the state observation is in $B(s)$. When the agent's action is not changed under an adversarial perturbation, its reward and transition at current step will not change in the DQN setting, either.

In some settings, we find that 100\% of the actions are guaranteed to be unchanged (e.g., the Pong environment in Table~\ref{tab:dqn_res}). In that case, we can in fact also certify that the accumulated reward is not changed given the specific initial conditions for testing. 
Otherwise, if some steps during the roll-out do not have this certificate, or have a weaker certificate that more than one actions are possible given $\hat{s} \in B(s)$, all the possible actions have to be explored as the next action input to the environment. When there are $n$ states which are not certified to have unchanged actions, each with $m$ possible actions, we need to run $n^m$ trajectories to find the worst case cumulative reward. This is impractical for typical settings.

However, even in the 100\% certificate rate setting like Pong, it can still be challenging to certify that the agent is robust under \emph{any} starting condition. Since the agent is started with a random initialization, it is impractical to enumerate all possible initializations and guarantee all generated trajectories are certified. Similarly, in the classification setting, many existing certified defenses~\citep{wong2018scaling,mirman2018differentiable,gowal2018effectiveness,zhang2019towards} can only practically guarantee robustness on a specific test set (by computing a ``verified test error''), rather than on \emph{any} input image.

\textbf{Robustness Certificates for PPO and DDPG.} In DDPG and PPO, the action space is continuous, hence it is not possible to certify that actions do not change under adversary. We instead seek for a different type of guarantee, where we can upper bound the change in action given a norm bounded input perturbation:
\begin{equation}
U_s \geq \max_{\hat{s} \in B(s)} \| \pi_{\theta_\pi}(\hat{s}) - \pi_{\theta_\pi}(s) \|
\label{eq:action_certificate}
\end{equation}
Given a state $s$, we can use convex relaxations to compute an upper bound $U_s$. Generally speaking, if $B(s)$ is small, a robust policy desires to have a small $U_s$, otherwise it can be possible to find an adversarial state perturbation that greatly changes $\pi_{\theta_\pi}(\hat{s})$ and causes the agent to misbehave. However, giving certificates on cumulative rewards is still challenging, as it requires to bound reward $r(s,a)$ given a fixed state $s$, and a perturbed and bounded action $a$ (bounded via~\eqref{eq:action_certificate}). Since the environment dynamics can be quite complex in practice (except for the simplest environment like InvertedPendulum), it is hard to bound reward changes given a bounded action. We leave this part as a future direction for exploration and we believe the robustness certificates in~\eqref{eq:action_certificate} can be useful for future works.


\section{Additional details for SA-PPO}
\label{sec:ppo_details}

\paragraph{Algorithm} We present the full SA-PPO algorithm in Algorithm~\ref{alg:sa_ppo}. Compared to vanilla PPO, we add a robust state-adversarial regularizer which constrains the KL divergence on state perturbations. We highlighted these changes in Algorithm~\ref{alg:sa_ppo}. The regularizer $\mathcal{R}_\text{PPO}(\theta_\pi)$ can be solved using SGLD or convex relaxations of neural networks. We define the perturbation set $B(s)$ to be an $\ell_p$ norm ball around state $s$ with radius $\epsilon$: $B_p (s, \epsilon) := \{s^\prime| \| s^\prime - s\|_p \leq \epsilon \}$. We use a $\epsilon$-schedule during training, where the perturbation budget is slowly increasing dduring each epoch $t$ as $\epsilon_t$ until reaching $\epsilon$.

\begin{algorithm}[htbp]
\caption{State-Adversarial Proximal Policy Optimization (SA-PPO). We highlight its differences compared to vanilla PPO in \textcolor{brown}{brown}.}
\label{alg:sa_ppo}
\begin{algorithmic}[1]
\REQUIRE Number of iterations $T$, a $\epsilon$ schedule $\epsilon_t$
\STATE{Initialize actor network $\pi(a|s)$ and critic network $V(s)$ with parameter $\theta_\pi$ and $\theta_V$},
\FOR{$t=1$ to $T$}
\STATE Run $\pi_{\theta_\pi}$ to collect a set of trajectories $\mathcal{D}=\{\tau_k\}$ containing $|\mathcal{D}|$ episodes, each $\tau_k$ is a trajectory contain $|\tau_k|$ samples, $\tau_{k} := \{(s_{k,i}, a_{k,i}, r_{k,i}, s_{k, i+1})\}$, $i \in [|\tau_k|]$
\STATE Compute cumulative reward $\hat{R}_{k,i}$ for each step $i$ in every episode $k$ using the trajectories and discount factor $\gamma$
\STATE Update value function by minimizing the mean-square error:
\[
\theta_V \leftarrow \argmin_{\theta_V} \frac{1}{\sum_k |\tau_k|} \sum_{\tau_k \in D} \sum_{i=0}^{|\tau_k|} \left (V(s_{k,i}) - \hat{R}_{k,i} \right )^2
\]
\STATE Estimate advantage $\hat{A}_{k,i}$ for each step $i$ in every episode $k$ using generalized advantage estimation (GAE) and value function $V_{\theta_V}(s)$
\STATE \textcolor{brown}{Define the state-adversarial policy regularier:}
\[
\mathcal{R}_\text{PPO}(\theta_\pi):=\sum_{\tau_k \in D} \sum_{i=0}^{|\tau_k|} \max_{\bar{s}_{k,i} \in B_p(s_{k,i}, \epsilon_t)} \mathrm{D}_{\mathrm{KL}}\left (\pi(a|s_{k,i})\|\pi(a|\bar{s}_{k,i}) \right)
\]
\STATE \textcolor{brown}{Option 1: Solve $\mathcal{R}_\text{PPO}(\theta_\pi)$ using SGLD:} 
\STATE \qquad find $\hat{s}_{k,i} = \argmax_{\bar{s}_{k,i} \in B_p(s_{k,i}, \epsilon_t)}  \mathrm{D}_{\mathrm{KL}}(\pi(a|s_{k,i})\|\pi(a|\bar{s}_{k,i}))$ using SGLD optimization for all $k, i$ (the objective can be solved in a batch)
\STATE \qquad set $\overline{\mathcal{R}}_{\text{PPO}}(\theta_\pi) := \sum_{\tau_k \in D} \sum_{i=0}^{|\tau_k|} \mathrm{D}_{\mathrm{KL}}(\pi(a|s_{k,i})\|\pi(a|\hat{s}_{k,i}))$
\STATE \textcolor{brown}{Option 2: Solve $\mathcal{R}_\text{PPO}(\theta_\pi)$ using convex relaxations:}
\STATE \qquad $\overline{\mathcal{R}}_{\text{PPO}}(\theta_\pi) := \text{ConvexRelaxUB}({\mathcal{R}_\text{PPO}}, \theta_\pi, \bar{s}_{k,i} \in B_p(s_{k,i}, \epsilon_t))$
\STATE Update the policy by minimizing the SA-PPO objective (the minimization is solved using ADAM):
\[
\theta_\pi \leftarrow \argmin_{\theta_\pi^\prime} \frac{1}{\sum_k |\tau_k|} \left [ \sum_{\tau_k \in D} \sum_{i=0}^{|\tau_k|} \min \left ( r_{\theta_\pi^\prime}(a_{k,i} | s_{k,i}) \hat{A}_{k,i}, g(r_{\theta_\pi^\prime}(a_{k,i} | s_{k,i})) \hat{A}_{k,i} \right ) \mathcolor{brown}{+ \kappa_\text{PPO} \overline{\mathcal{R}}_{\text{PPO}}(\theta_\pi^\prime)} \right ]
\]
where $r_{\theta_\pi^\prime}(a_{k,i} | s_{k,i}) := \frac{\pi_{\theta_\pi^\prime}(a_{k,i} |s_{k,i})}{\pi_{\theta_\pi} (a_{k,i} |s_{k,i})}$, $g(r):=\mathrm{clip}(r_{\theta_\pi^\prime}(a_{k,i} | s_{k,i}), 1 - \epsilon_\text{clip}, 1 + \epsilon_\text{clip})$
\ENDFOR
\end{algorithmic}
\end{algorithm}

\paragraph{Hyperparameters for Regular PPO Training} We use the optimal hyperparameters in~\citep{engstrom2020implementation} which were found using a grid search for vanilla PPO. However, we found that their parameters are not optimal for Humanoid and achieves a cumulative reward of only about 2000 after $1 \times 10^7$ steps. Thus we redo hyperparameter search on Humanoid and change learning rate for actor to $5 \times 10^{-5}$ and critic to $1 \times 10^{-5}$. This new set of hyperemeters allows us to obtain Humanoid reward about 5000 for vanilla PPO. Note that even under the original, non-optimal set of hyperemeters by~\citep{engstrom2020implementation}, our SA-PPO variants still achieve high rewards similarly to those reported in our paper. Our hyperparameter change only significantly improves the performance of vanilla PPO baseline.

We run 2048 simulation steps per iteration, and run policy optimization of 10 epochs with a minibatch size of 64 using Adam optimizer with learning rate $3 \times 10^{-4}$, $4 \times 10^{-4}$ and $5 \times 10^{-5}$ for Walker, Hopper and Humanoid, respectively. The value network is also trained in 10 epochs per iteration with a minibatch size of 64, using Adam optimizer with learning rate 0.00025, $3 \times 10^{-4}$, and $1 \times 10^{-5}$ for Walker, Hopper and Humanoid environments, respectively (the same as in~\citep{engstrom2020implementation} without further tuning, except for Humanoid as discussed above). Both networks are 3-layer MLPs with $[64,64]$ hidden neurons. The clipping value $\epsilon$ for PPO is 0.2. We clip rewards to $[-10, 10]$ and states to $[-10,10]$. The discount factor $\gamma$ for reward is 0.99 and the discount factor used in generalized advantage estimation (GAE) is 0.95. We found that in~\citep{engstrom2020implementation} the agent rewards are still improving when training finishes, thus in our experiments we run the agents longer for better convergence: we run Walker2d and Hopper $2 \times 10^6$ steps (976 iterations) and Humanoid $1 \times 10^7$ steps (4882 iterations) to ensure convergence.

\paragraph{Hyperparameter for SA-PPO Training} For SA-PPO, we use the same set of hyperparameters as in PPO. Note that the hyperparameters are tuned for PPO but not specifically for SA-PPO. The additional regularization parameter $\kappa_\text{PPO}$ for the regularizer $\mathcal{R}_\text{PPO}$ is chosen in $\{0.003, 0.01, 0.03, 0.1, 0.3, 1.0\}$. We linearly increase $\epsilon_t$, the norm of $\ell_\infty$ perturbation on normalized states, from 0 to the target value ($\epsilon$ for evaluation, reported in Table~\ref{tab:ppo_res}) during the first $3/4$ iterations, and keep $\epsilon_t=\epsilon$ for the reset iterations. The same $\epsilon$ schedule is used for both SGLD and convex relaxation training. For SGLD, we run 10 iterations with step size $\frac{\epsilon_t}{10}$ and set the temperature parameter $\beta=1 \times 10^{-5}$. For convex relaxations, we use the efficient IBP+Backward scheme~\citep{xu2020automatic}, and we use a training schedule similar to~\citep{zhang2019towards} by mixing the IBP bounds and backward mode perturbation analysis bounds.

\section{Additional Details for SA-DDPG}
\label{sec:ddpg_details}
\paragraph{Algorithm}
We present the SA-DDPG training algorithm in Algorithm~\ref{alg:sa_ddpg}. The main difference between DDPG and SA-DDPG is the additional loss term $\mathcal{R}_{\text{DDPG}}(\theta_\pi)$, which provides an upper bound on $\max_{s \in B(s_i)} \| \pi(s) - \pi(s_i) \|_2^2$. We highlighted these changes in Algorithm~\ref{alg:sa_ddpg}. We define the perturbation set $B(s)$ to be a $\ell_p$ norm ball around $s$ with radius $\epsilon$: $B_p (s, \epsilon) := \{s^\prime| \| s^\prime - s\|_p \leq \epsilon \}$. We use a $\epsilon$-schedule during training, where the perturbation budget is slowly increasing during training as $\epsilon_t$ until reaching $\epsilon$.

\begin{algorithm}[htbp]
\caption{State-Adversarial Deep Deterministic Policy Gradient (SA-DDPG). We highlight its differences compared to vanilla DDPG in \textcolor{brown}{brown}.}
\label{alg:sa_ddpg}
\begin{algorithmic}
\STATE{Initialize actor network $\pi(s)$ and critic network $Q(s,a)$ with parameter $\theta_\pi$ and $\theta_Q$}
\STATE{Initialize target network $\pi^\prime(s)$ and critic network $Q^\prime(s,a)$ with weights $\theta_{\pi^\prime} \leftarrow \theta_\pi$ and $\theta_{Q^\prime} \leftarrow \theta_Q$}
\STATE{Initial replay buffer $\mathcal{B}$}
\FOR{$t=1$ to $T$}
\STATE{Initial a random process $\mathcal{N}$ for action exploration}
\STATE{Choose action $a_t \sim \pi(s_t) + \epsilon$, $\epsilon \sim \mathcal{N}$}
\STATE{Observe reward $r_t$, next state $s_{t+1}$ from environment}
\STATE{Store transition $\{s_t,a_t,r_t,s_{t+1}\}$ into $\mathcal{B}$}
\STATE{Sample a mini-batch of $N$ samples $\{s_i,a_i,r_i,s^\prime_i\}$ from $\mathcal{B}$}
\STATE{$y_i \leftarrow r_i + \gamma Q^\prime(s^\prime_i, \pi^\prime(s^\prime_i))$ for all $i \in [N]$}
\STATE{Update $\theta_Q$ by minimizing loss $L(\theta_Q) = \frac{1}{N}\sum_i \left (y_i - Q(s_i,a_i) \right)^2$}
\STATE \textcolor{brown}{$\mathcal{R}_\text{DDPG}(\theta_\pi, \bar{s}_i):=\sum_i\max_{\bar{s}_i \in B_p(s_i, \epsilon_t)} \| \pi_{\theta_\pi} (s_i) - \pi_{\theta_\pi} (\bar{s}_i) \|_2$}
\STATE \textcolor{brown}{Option 1: Solve $\mathcal{R}_{\text{DDPG}}(\theta_\pi)$ using SGLD:}
\STATE \qquad find $\hat{s}_i = \argmax_{\bar{s}_i \in B_p(s_i, \epsilon_t)} \| \pi_{\theta_\pi} (s_i) - \pi_{\theta_\pi} (\bar{s}_i) \|_2$ for all $i$ (solved in a batch using SGLD)
\STATE \qquad set $\overline{\mathcal{R}}_\text{DDPG}(\theta_\pi) := \sum_i \| \pi_{\theta_\pi} (s_i) - \pi_{\theta_\pi} (\hat{s}_i) \|_2$
\STATE \textcolor{brown}{Option 2: Solve $\mathcal{R}_\text{DDPG}(\theta_\pi)$ using convex relaxations:}
\STATE \qquad $\overline{\mathcal{R}}_\text{DDPG}(\theta_\pi) := \text{ConvexRelaxUB}(\mathcal{R}_{\text{DDPG}}, \theta_\pi, \bar{s}_i \in B_p(s_i, \epsilon_t))$
\STATE{Update $\theta_\pi$ using deterministic policy gradient and gradient of $\overline{\mathcal{R}}_\text{DDPG}$:}
\STATE{$\nabla_{\theta_\pi} J(\theta_\pi) = \frac{1}{N}\sum_{i} \left [ \nabla_{a} Q(s,a) \rvert_{s=s_i, a=\pi(s_i)} \nabla_{\theta_\pi} \pi(s) \rvert_{s=s_i} \mathcolor{brown}{+ \kappa_\text{DDPG} \nabla_{\theta_\pi} \overline{\mathcal{R}}_\text{DDPG}} \right ] $}
\STATE{Update Target Network:}
\STATE{$\theta_{Q^\prime} \leftarrow \tau \theta_Q + (1-\tau) \theta_{Q^\prime}$} 
\STATE{$\theta_{\pi^\prime} \leftarrow \tau \theta_\pi + (1-\tau) \theta_{\pi^\prime}$}
\ENDFOR
\end{algorithmic}
\end{algorithm}

\paragraph{Hyperparameters for Regular DDPG Training.} Our hyperparameters are from~\citep{deeprl}. Both actor and critic networks are 3-layer MLPs with $[400, 300]$ hidden neurons. We run each environment for $2 \times 10^6$ steps. Actor network learning rate is $1 \times 10^{-4}$ and critic network learning rate is $1 \times 10^{-3}$ (except that for Hopper-v2 and Ant-v2 the critic learning rate is reduced to $1 \times 10^{-4}$ due to the larger values of rewards); both networks are optimized using Adam optimizer. No reward scaling is used, and discount factor is set to $0.99$. We use a replay buffer with a capacity of $1 \times 10^6$ items and we do not use prioritized replay buffer sampling. For the random process $\mathcal{N}$ used for exploration, we use a Ornstein-Uhlenbeck process with $\theta=0.15$ and $\sigma=0.2$. The mixing parameter of current and target actor and critic networks is set to $\tau=0.001$.

\paragraph{Hyperparameters for SA-DDPG Training.} SA-DDPG uses the same hyperparameters as in DDPG training. For the additional regularization parameter $\kappa$ for $\pi(s)$, we choose in $\{0.1, 0.3, 1.0, 3.0\}$ for InvertedPendulum and Reacher due to their low dimensionality and $\{30, 100, 300, 1000\}$ for other environments.. We train the actor network without state-adversarial regularization for the first $1 \times 10^6$ steps, then increase $\epsilon_t$ from 0 to the target value in $5 \times 10^5$ steps, and then keep training at the target $\epsilon$ for $5 \times 10^5$ steps. The same $\epsilon$ schedule is used for both SGLD and convex relaxation. For SGLD, we run 5 iterations with step size $\frac{\epsilon_t}{5}$ and set the temperature parameter $\beta=1 \times 10^{-5}$. For convex relaxations, we use the efficient IBP+Backward scheme~\citep{xu2020automatic}, and  a training schedule similar to~\citep{zhang2019towards} by mixing the IBP bounds and backward mode perturbation analysis bounds. The total number of training steps is thus $2 \times 10^6$, which is the same as the regular DDPG training. The target $\epsilon$ values for each task is the same as $\epsilon$ listed in Table~\ref{tab:ddpg_res} for evaluation. Note that we apply perturbation on normalized environment states. The normalization factors are the standard deviations calculated using data collected on the baseline policy (vanilla DDPG) without adversaries.

\section{Additional Details for SA-DQN}

\label{sec:dqn_details}
\paragraph{Algorithm} We present the SA-DQN training algorithm in Algorithm~\ref{alg:sa_dqn}. The main difference between SA-DQN and DQN is the additional state-adversarial regularizer $\mathcal{R}_\text{DQN}(\theta)$, which encourages the network not to change its output under perturbations on the state observation. We highlighted these changes in Algorithm~\ref{alg:sa_dqn}. Note that the use of hinge loss is not required; other loss functions (e.g., cross-entropy loss) may also be used.
\begin{algorithm}[htbp]
\caption{State-Adversarial Deep Q-Learning (SA-DQN). We highlight its differences compared to vanilla DQN in \textcolor{brown}{brown}.}
\label{alg:sa_dqn}
\begin{algorithmic}[1]
\STATE{Initialize current Q network $Q(s,a)$ with parameters $\theta$.}
\STATE{Initialize target Q network $Q^\prime(s,a)$ with parameters $\theta' \leftarrow \theta$.}
\STATE{Initial replay buffer $\mathcal{B}$}
\FOR{$t=1$ to $T$}
\STATE{With probability $\epsilon_t$ select a random action at $a_t$, otherwise select $a_t=\argmax_aQ_\theta(s_t,a;\theta)$}
\STATE{Execute action $a_t$ in environment and observe reward $r_t$ and state $s_{t+1}$}
\STATE{Store transition $\{s_t,a_t,r_t,s_{t+1}\}$ in $\mathcal{B}$.}
\STATE{Randomly sample a minibatch of $N$ samples $\{s_i,a_i,r_i,s^\prime_i\}$ from $\mathcal{B}$.}
\STATE{For all $s_i$, compute $a_i^*=\argmax_aQ_\theta(s_i,a;\theta)$.}
\STATE{Set $y_i=r_i+\gamma\max_{a^\prime}Q'_{\theta'}(s^\prime_i,a^\prime;\theta)$ for non-terminal $s_i$, and $y_i=r_i$ for terminal $s_i$.}
\STATE{Compute TD-loss for each transition: $\text{TD-}L(s_i,a_i,s^\prime_i;\theta)=\text{Huber}( y_i-Q_\theta(s_i,a_i;\theta))$}
\STATE{\textcolor{brown}{Define $\mathcal{R}_\text{DQN}(\theta):=\sum_i\max\big\{\max_{\hat{s}_i\in B(s)}\max_{a\neq a_i^*}Q_\theta(\hat{s}_i,a;\theta)-Q_\theta(\hat{s}_i,a_i^*;\theta),-c\big\}$.}}
\STATE{\textcolor{brown}{Option 1: Use projected gradient descent (PGD) to solve $\mathcal{R}_\text{DQN}(\theta)$.}}
\STATE{\qquad Run PGD to solve: $\hat{s}_i = \argmax_{\hat{s}_i\in B(s_i)}\max_{a\neq a_i^*}Q_\theta(\hat{s}_i,a;\theta)-Q_\theta(\hat{s}_i,a_i^*;\theta)$.}
\STATE{\qquad Compute the sum of hinge loss of each $s_i$:\\
\qquad $\overline{\mathcal{R}}_{\text{DQN}}(\theta)=\sum_i\max\{\max_{a\neq a_i^*}Q_\theta(\hat{s}_i,a;\theta)-Q_\theta(\hat{s}_i,a_i^*),-c\}$.}
\STATE{\textcolor{brown}{Option 2: Use convex relaxations of neural networks to solve a surrogate loss of $\mathcal{R}_\text{DQN}(\theta)$.}}
\STATE{\qquad For all $s_i$ and all $a\neq a_i^*$, obtain upper bounds on $Q_\theta(s,a;\theta)-Q_\theta(s,a_i^*;\theta)$:\\
$\qquad u_{a_i^*, a}(s_i; \theta)=\text{ConvexRelaxUB}(Q_\theta(s,a;\theta)-Q_\theta(s,a_i^*;\theta), \theta, s\in B(s_i))$}
\STATE{\qquad Compute a surrogate loss for the hinge loss:\\ \qquad$\overline{\mathcal{R}}_{\text{DQN}}(\theta)=\sum_i\max\big\{\max_{a\neq a_i^*}\{u_{a_i^*,a}(s_i)\},-c\big\}$}
\STATE{Perform a gradient descent step to minimize $\frac{1}{N}[\sum_i\text{TD-}L(s_i,a_i,s^\prime_i;\theta) \mathcolor{brown}{ + \kappa_\text{DQN}\overline{\mathcal{R}}_{\text{DQN}}(\theta)}]$.}
\STATE{Update Target Network every $M$ steps: $\theta'\leftarrow\theta.$}
\ENDFOR
\end{algorithmic}
\end{algorithm}

\paragraph{Hyperparameters for Vanilla DQN training.}
For Atari games, the deep Q networks have 3 CNN layers followed by 2 fully connected layers (following~\citep{wang2016dueling}). The first CNN layer has 32 channels, a kernel size of 8, and stride 4. The second CNN layer has 64 channels, a kernel size of 4, and stride 2. The third CNN layer has 64 channels, a kernel size of 3, and stride 1. The fully connected layers have $512$ hidden neurons for both value and advantage heads. We run each environment for $6\times 10^6$ steps without framestack. We set learning rate as $6.25\times 10^{-5}$ (following~\citep{hessel2017rainbow}) for Pong, Freeway and RoadRunner; for BankHeist our implementation cannot reliably converge within 6 million steps, so we reduce learning rate to $1 \times 10^{-5}$. For all Atari environments, we clip reward to $-1, +1$ (following~\citep{mnih2015human}) and use a replay buffer with a capacity of $2\times 10^5$.


We set discount factor set to 0.99. Prioritized replay buffer sampling is used with $\alpha=0.5$ and $\beta$ increased from 0.4 to 1 linearly through the end of training. A batch size of 32 is used in training. Same as in~\citep{mnih2015human}, we choose Huber loss as the TD-loss. We update the target network every 2k steps for all environments.

\paragraph{Hyperparameters for SA-DQN training.}
SA-DQN uses the same network structure and hyperparameters as in DQN training. The total number of SA-DQN training steps in all environments are the same as those in DQN (6 million). We update the target network every 2k steps for all environments except that the target network is updated every 32k steps for RoadRunner's SA-DQN, which improves convergence for our short training schedule of 6 million frames. For the additional state-adversarial regularization parameter $\kappa$ for robustness, we choose $\kappa\in\{0.005, 0.01,0.02\}$. For all 4 Atari environments, we train the Q network without regularization for the first $1.5\times 10^6$ steps, then increase $\epsilon$ from 0 to the target value in $4\times 10^6$ steps, and then keep training at the target $\epsilon$ for the rest $5\times 10^5$ steps. 

\paragraph{Training Time}
As Atari training is expensive, we train DQN and SA-DQN only 6 million frames; the rewards reported in most DQN paper (e.g.,~\citep{mnih2015human,wang2016dueling,hessel2017rainbow}) are obtained by training 20 million frames. Thus, the rewards (without attacks) reported maybe lower than some baselines. The training time for vanilla DQN, SA-DQN (SGLD) and SA-DQN (convex) are roughly 15 hours, 40 hours and 50 hours on a single 1080 Ti GPU, respectively. The training time of each environment varies but is very close.

Note that the training time for convex relaxation based method can be further reduced when using an more efficient relaxation. The fastest relaxation is interval bound propagation (IBP), however it is too inaccurate and can make training unstable and hard to tune~\citep{zhang2019towards}. We use the tighter IBP+Backward relaxation, and its complexity can be further improved to the same level as IBP with the recently developed loss fusion technique~\citep{xu2020automatic}, while providing a much better relaxation than IBP. Our work simply uses convex relaxations as a blackbox tool and we leave further improvements on convex relaxation based methods as a future work.

\section{Additional Experimental Results}
\label{sec:app_exp}

\subsection{More results on SA-PPO} 

\paragraph{Box plots of rewards for SA-PPO agents}

In Table~\ref{tab:ppo_res}, we report the mean and standard deviation of rewards for agents under attack. However, since the distribution of cumulative rewards can be non-Gaussian, in this section we include box plots of rewards for each task in Figure~\ref{fig:test}. We can observe that the rewards (median, 25\% and 75\% percentiles) under the strongest attacks (Figure~\ref{fig:att_box}) significantly improve. 

\begin{figure}
\centering
\begin{subfigure}{.5\textwidth}
  \centering
  \includegraphics[width=\linewidth]{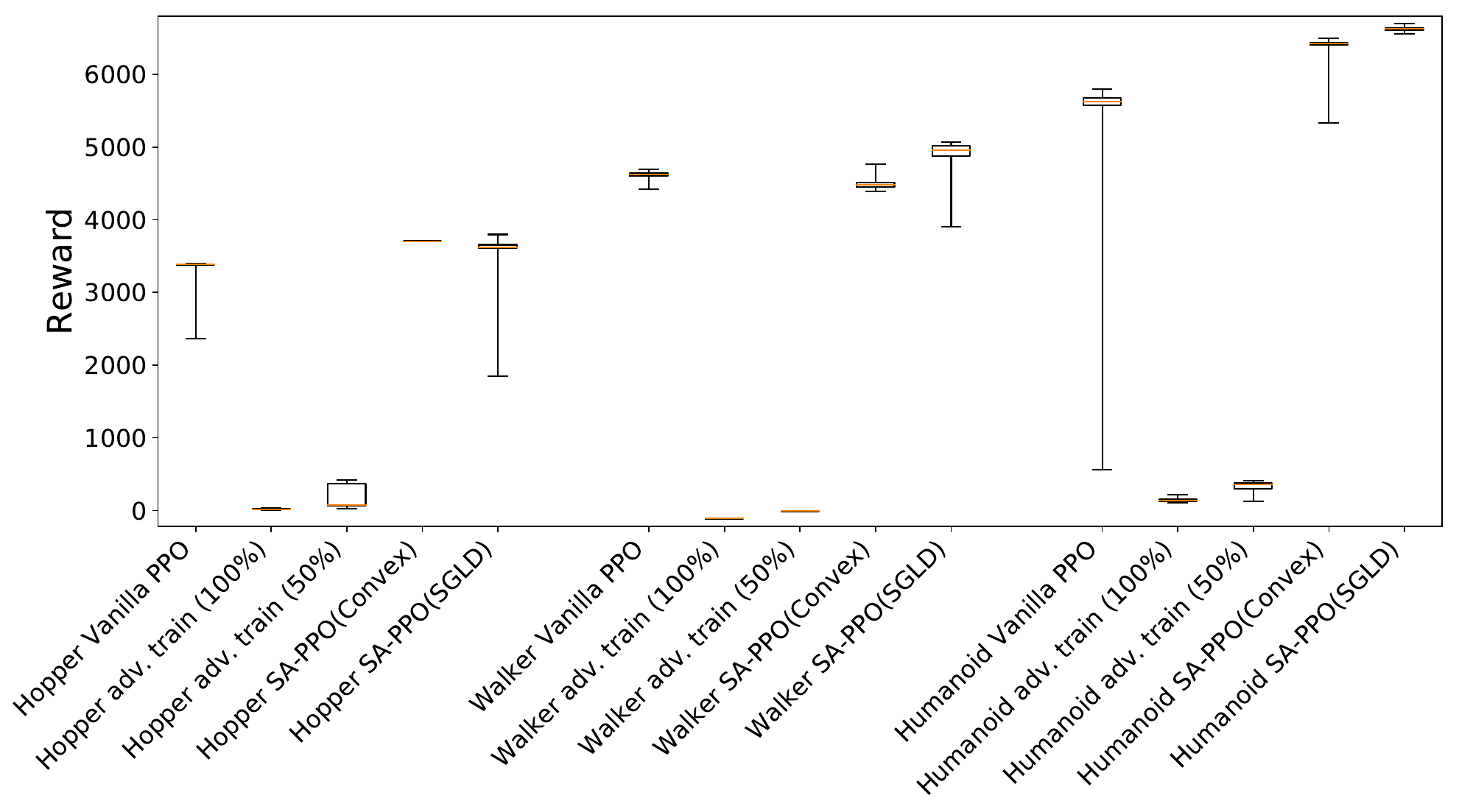}
  \caption{Natural episode rewards (no attacks)}
  \label{fig:nat_box}
\end{subfigure}%
\begin{subfigure}{.5\textwidth}
  \centering
  \includegraphics[width=\linewidth]{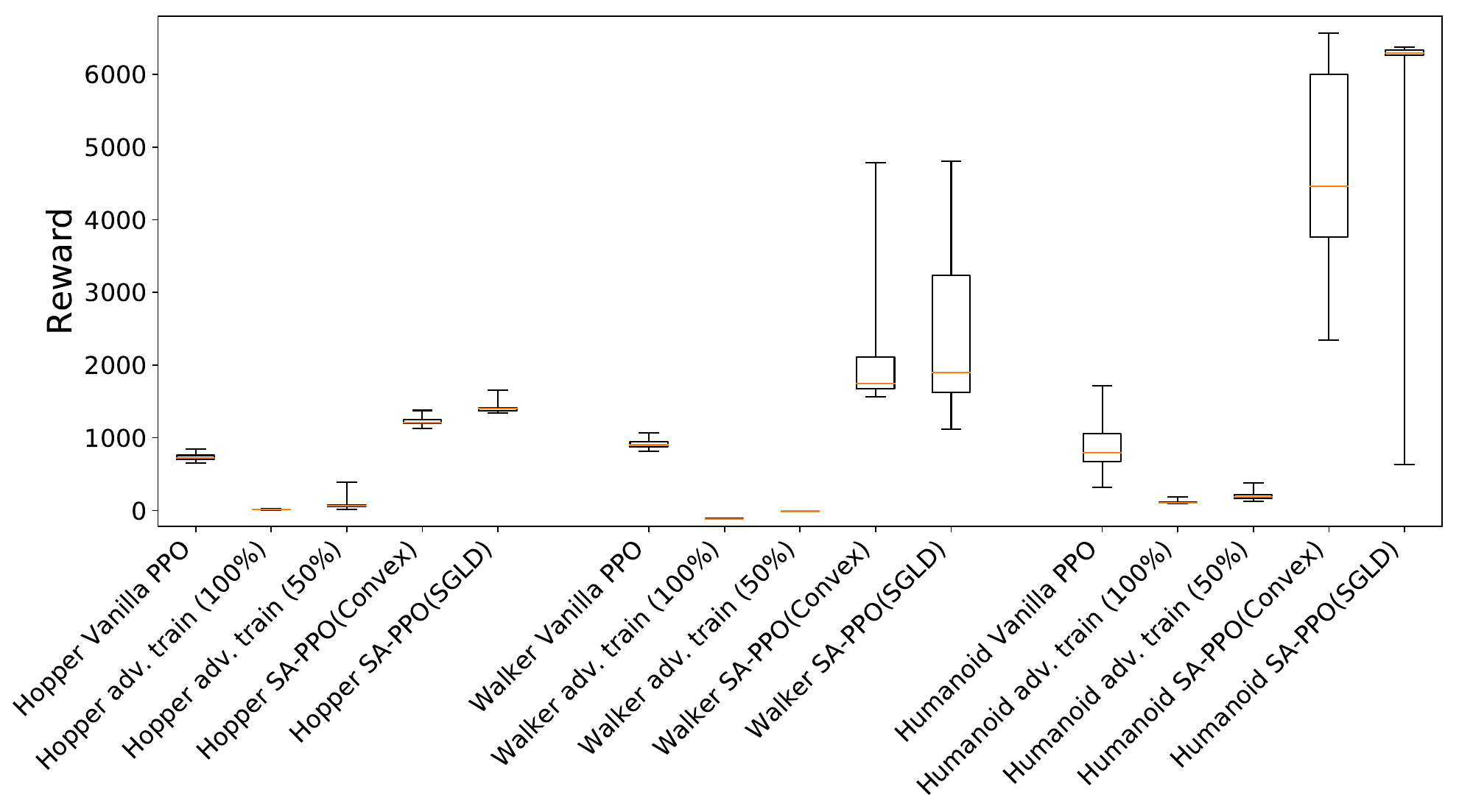}
  \caption{Rewards under the best (strongest) attacks}
  \label{fig:att_box}
\end{subfigure}
\caption{Box plots of natural rewards and rewards under the strongest (best) attacks for PPO, adversarially trained PPO and SA-PPO agents corresponding to the results presented in Table~\ref{tab:ppo_res} (Table~\ref{tab:ppo_res} only reports mean and standard deviation). Each box shows the distribution of cumulated rewards collected from 50 episodes of a single agent. The red lines inside the boxes are median rewards, and the upper and lower sides of the boxes show 25\% and 75\% percentile rewards of 50 episodes. The line segments outside of the boxes show min or max rewards.}
\label{fig:test}
\end{figure}


\begin{figure}
\begin{center}
     
     \begin{tabular}{cc}
     Hopper&\begin{minipage}{.95\textwidth}
      \includegraphics[width=\linewidth]{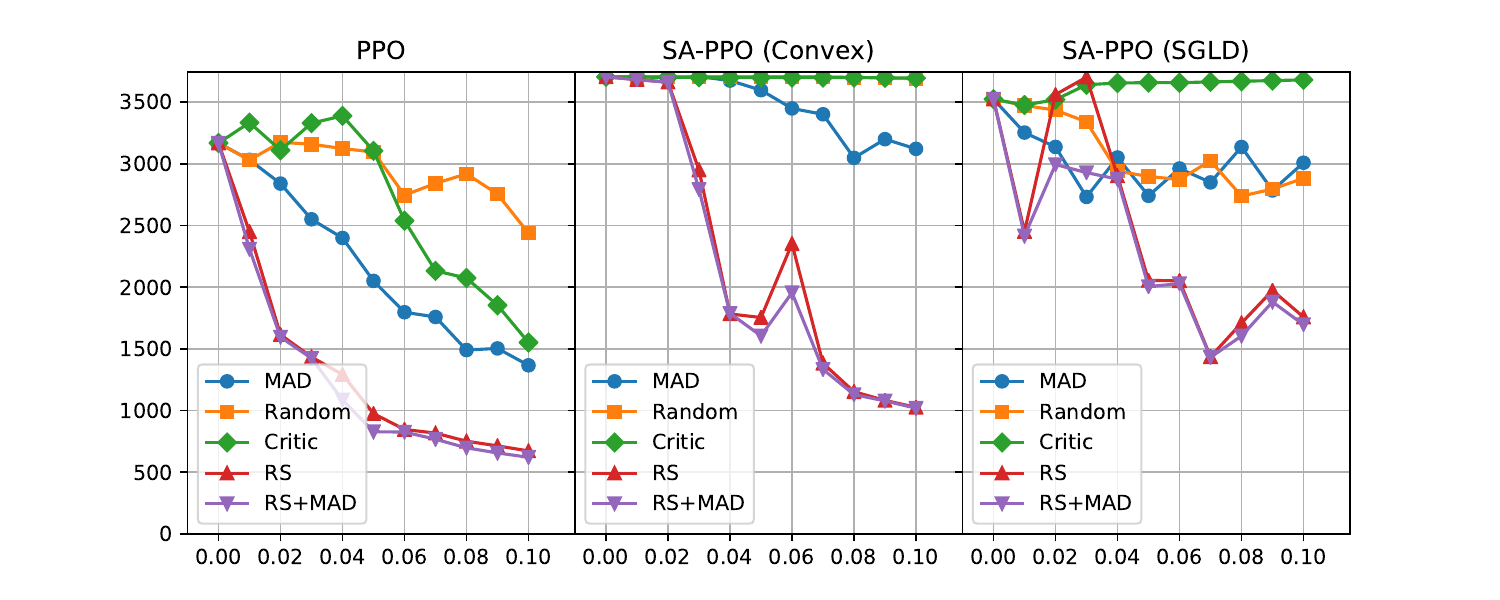}
    \end{minipage}\\
    Walker&\begin{minipage}{.95\textwidth}
      \includegraphics[width=\linewidth]{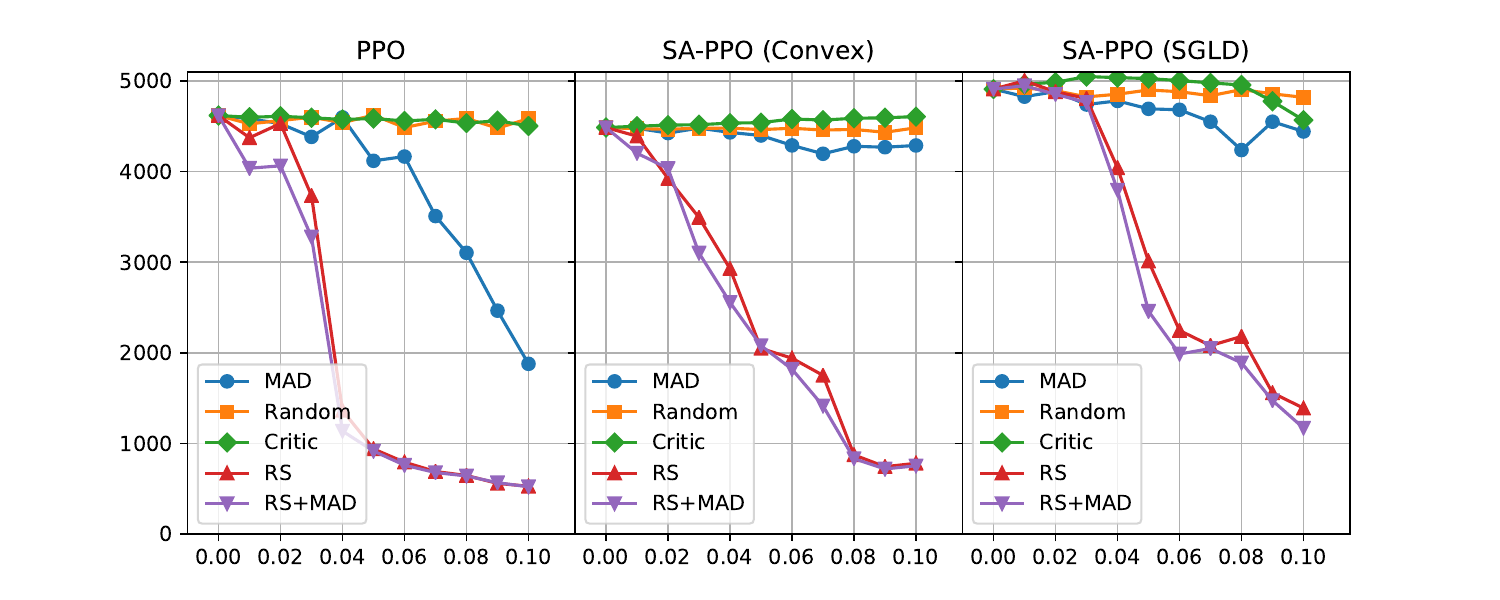}
    \end{minipage}\\
    Humanoid&\begin{minipage}{.95\textwidth}
      \includegraphics[width=\linewidth]{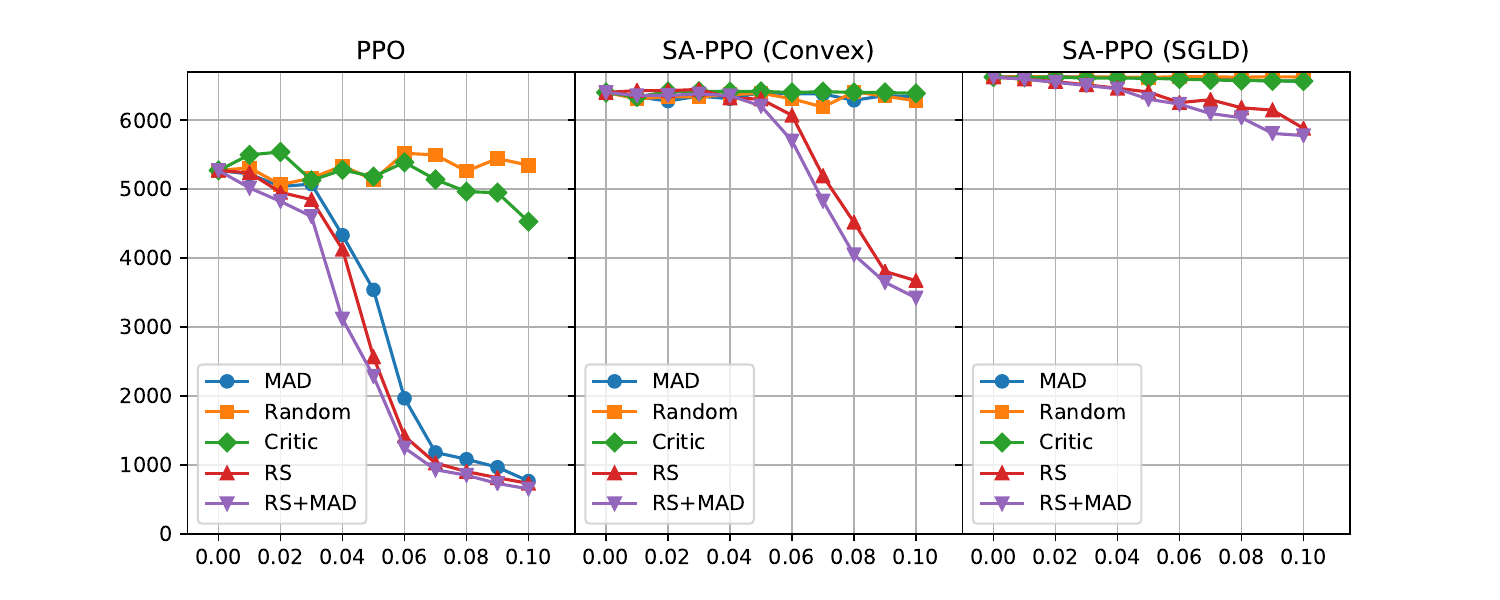}
    \end{minipage}\\
     \end{tabular}
  \caption{Attacking PPO agents under different $\epsilon$ values using 5 attacks. Each data point reported in this figure is an average of 50 episodes.}
  \label{fig:attack_eps_sweep}
\end{center}
\end{figure}

\paragraph{Evaluation using multiple $\epsilon$}
In Figure~\ref{fig:attack_eps_sweep} we show the attack rewards of PPO and SA-PPO agents with different perturbation budget $\epsilon$. We can see that the lowest attack rewards of SA-PPO agents are higher than those of PPO under all $\epsilon$ values. Additionally, Robust Sarsa (RS) attacks and RS+MAD attacks are typically stronger than other attacks. On vanilla PPO agents, the MAD attack is also competitive.

\paragraph{Convergence of PPO and SA-PPO agents}

We want to confirm that our better performing Humanoid agents under state-adversarial regularization are not just by chance. We train each environment using SA-PPO and PPO \emph{at least 15 times}, and collect rewards during training. 
We plot the median, 25\% and 75\% percentile of rewards during the training process for all these runs in Figure~\ref{fig:convergence}. 

We can see that our SA-PPO agents consistently outperform vanilla PPO agents in Humanoid. Since we also present the 25\% and 75\% percentile of the rewards among 15 agents, we believe this improvement is not because of cherry-picking. For Hopper and Walker environments, 
SA-PPO has almost no performance drop compared to vanilla PPO.

\begin{figure}
\centering
\begin{subfigure}{.8\textwidth}
  \centering
  \includegraphics[width=\linewidth]{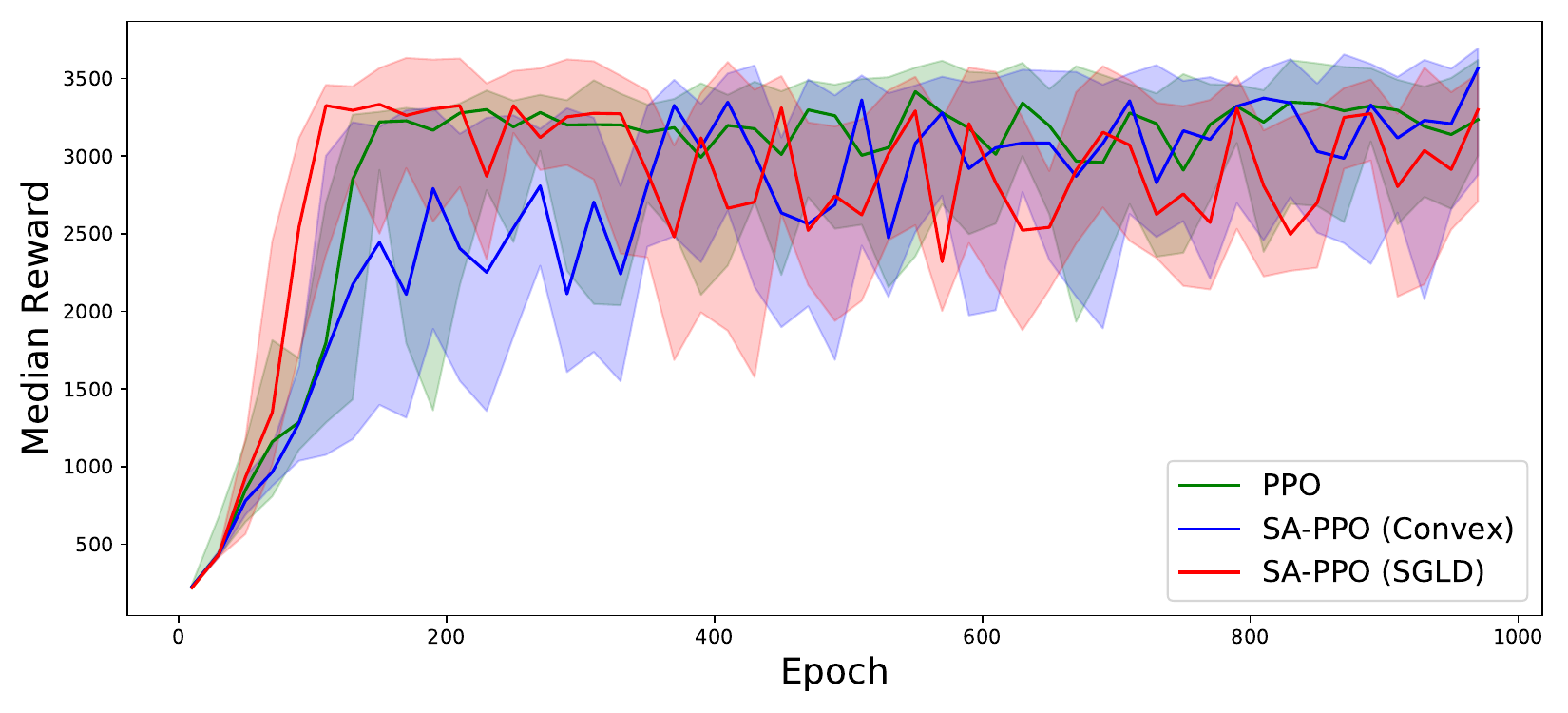}
  \caption{Hopper}
  \label{fig:hopper_cov}
\end{subfigure}\\
\begin{subfigure}{.8\textwidth}
  \centering
  \includegraphics[width=\linewidth]{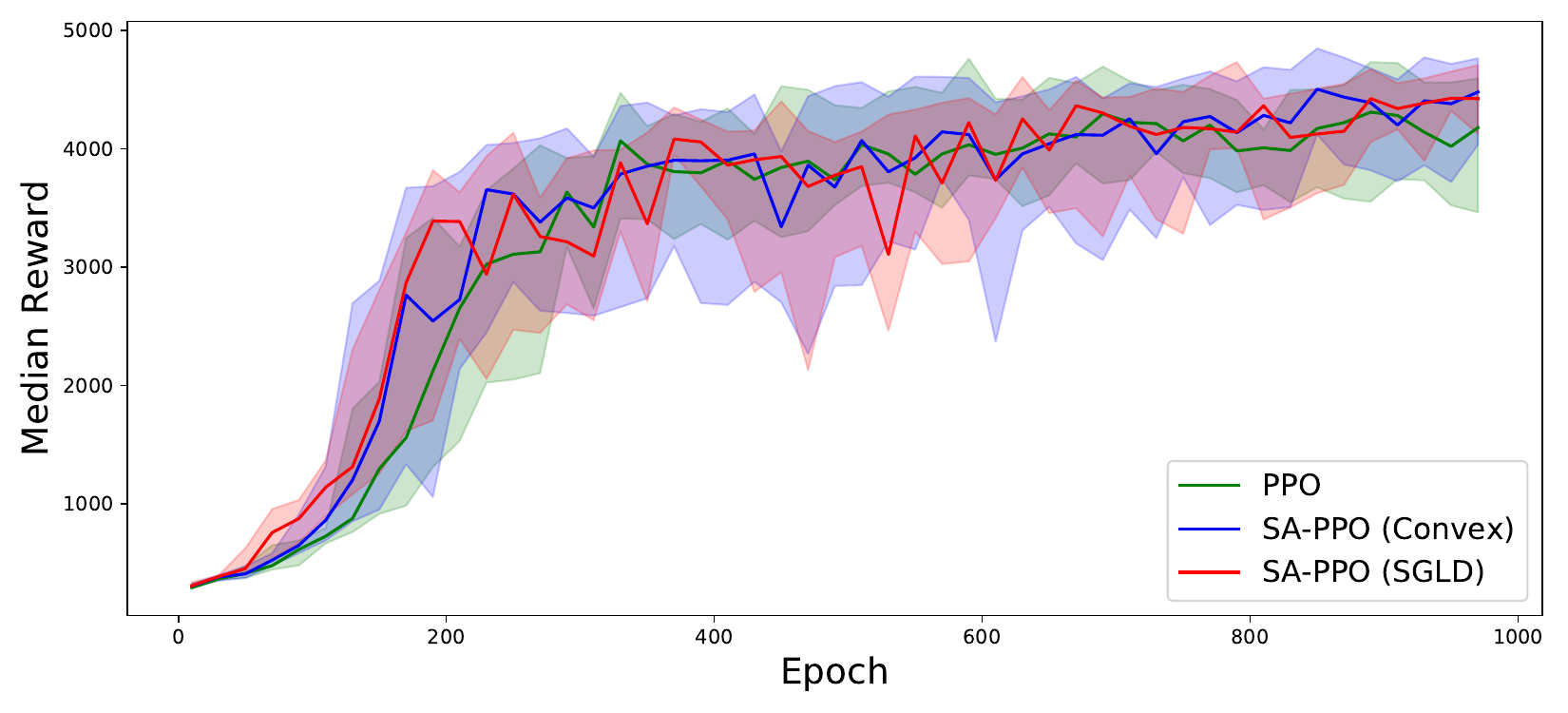}
  \caption{Walker}
  \label{fig:walker_cov}
\end{subfigure}\\
\begin{subfigure}{.8\textwidth}
  \centering
  \includegraphics[width=\linewidth]{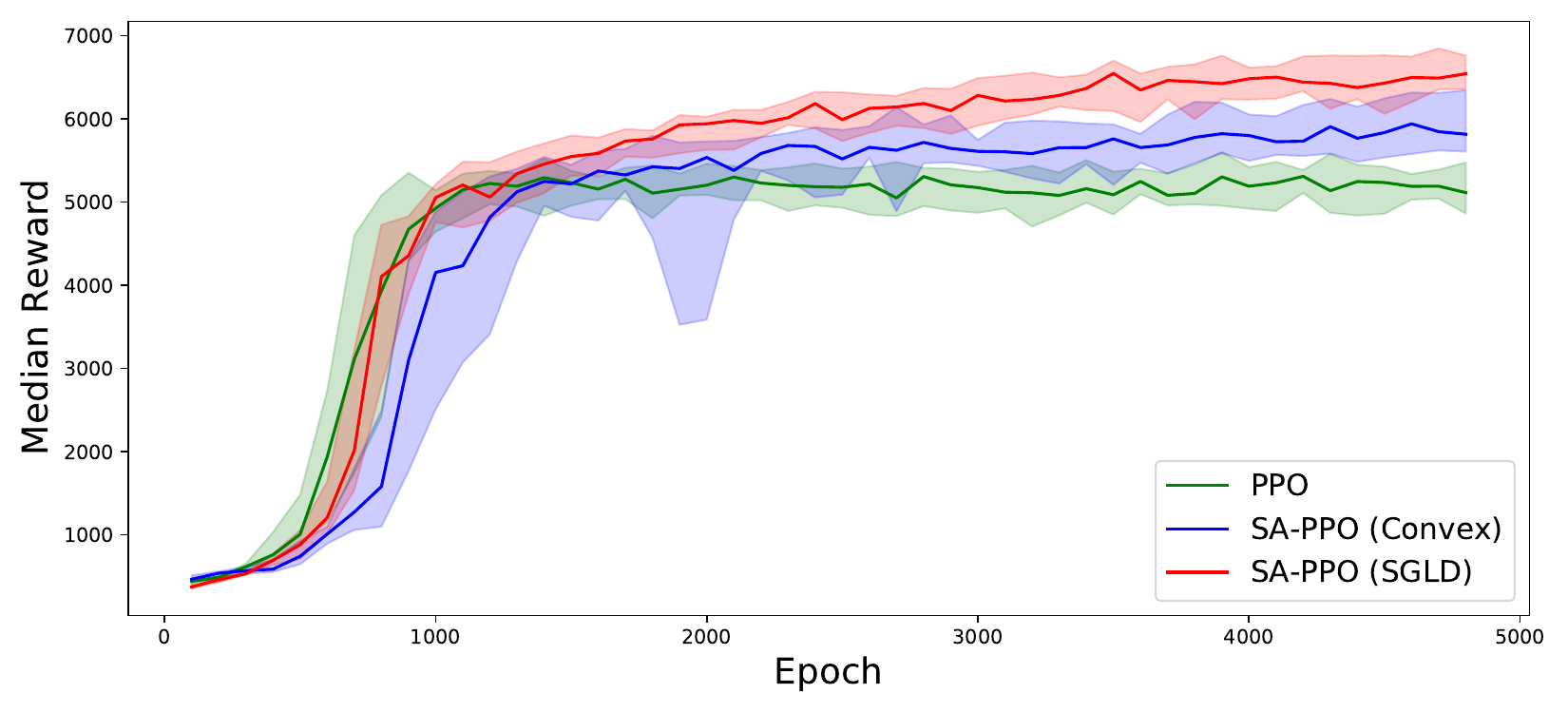}
  \caption{Humanoid}
  \label{fig:humanoid_conv}
\end{subfigure}

\caption{The median, 25\% and 75\% percentile episode reward of at least 15 PPO and 15 SA-PPO agents during training. 
The region of the shaded colors (light blue: SA-PPO solved with SGLD; light green: SA-PPO solved with convex relaxations; light red: vanilla PPO) represent the interval between 25\% and 75\% percentile rewards over these 15 different training runs, and the solid line is the median rewards over these runs.}
\label{fig:convergence}
\end{figure}


\subsection{More results on SA-DDPG}

\paragraph{Reproducibility over multiple training runs.} To show that our SA-DDPG can consistently obtain a robust agent and we do not cherry-pick good results, we repeatedly train all 5 environments using SA-DDPG and DDPG \textbf{11 times} each and attack all agents. We report the median, minimum, 25\% and 75\% rewards of 11 agents in box plots. The results are shown in Figure~\ref{fig:ddpg_repo_test}. We can observe that SA-DDPG is able to consistently improve the robustness: the median, 25\% and 75\% percentile rewards under attacks are significantly and consistently better than vanilla DDPG over all 5 environments.

\begin{figure}[!tb]
\centering
\begin{subfigure}{.5\textwidth}
  \centering
 \includegraphics[width=\linewidth]{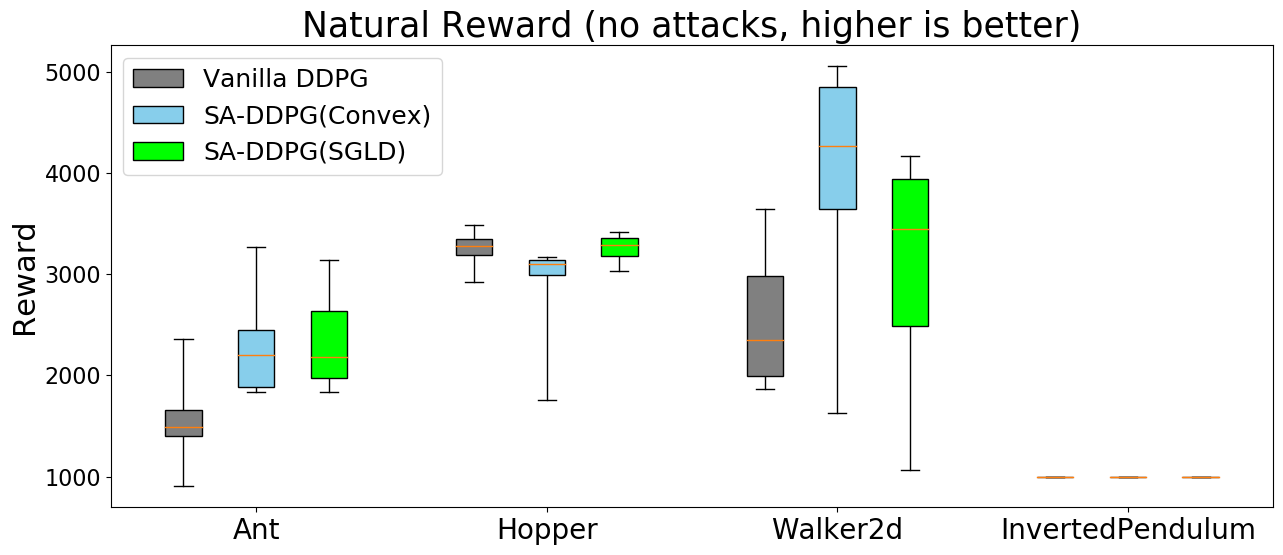}
  \caption{Natural episode rewards (no attacks)}
  \label{fig:ddpg_nat_reproduce}
\end{subfigure}%
\begin{subfigure}{.5\textwidth}
  \centering
 \includegraphics[width=\linewidth]{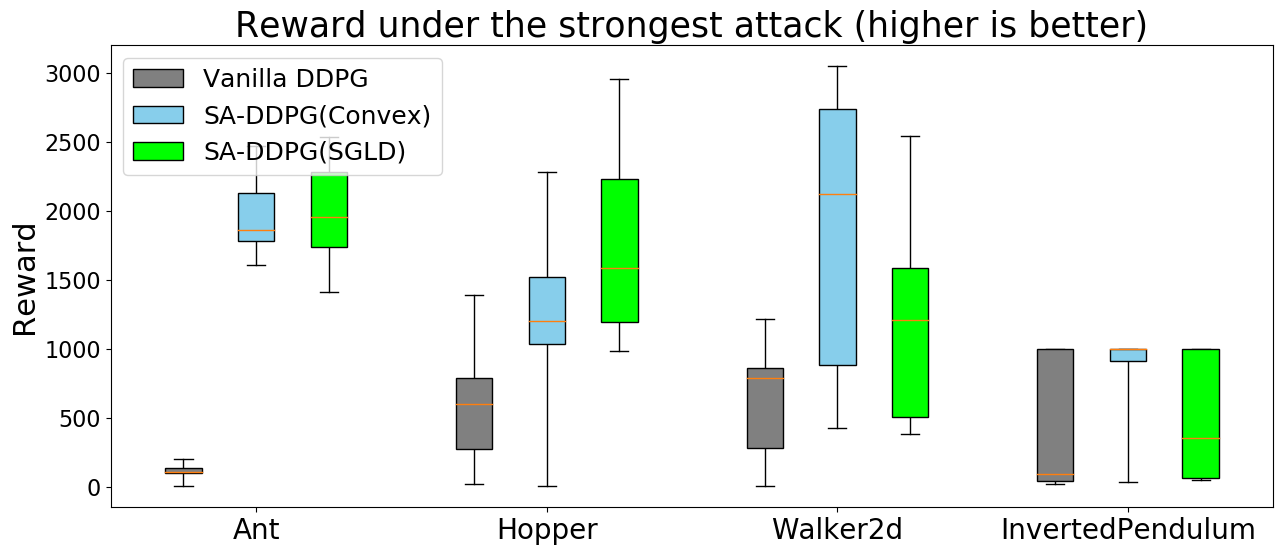}
  \caption{Rewards under the best (strongest) attacks}
  \label{fig:ddpg_att_reproduce}
\end{subfigure}
\caption{Box plots of natural and attack rewards for DDPG and SA-DDPG. Each box is obtained from \textbf{11 agents} trained with the same hyerparameters as the agents reported in Table~\ref{tab:ddpg_res} and tested for 50 episodes (each sample of the box is an average reward over 50 episodes). The red lines inside the boxes are median rewards, and the upper and lower sides of the boxes show 25\% and 75\% percentile rewards. The line segments outside of the boxes show min or max rewards.}
\label{fig:ddpg_repo_test}
\end{figure}

\paragraph{Full attack results}

In Table~\ref{tab:ddpg_full_res_} we present attack rewards on all of our DDPG agents. In the main text, we only report the strongest (lowest) attack rewards since the lowest reward determines the true agent robustness.

\begin{table*}\centering
\resizebox{\linewidth}{!}{
\begin{tabular}{c|c|c c c c c c}
\toprule
\multicolumn{2}{c|}{Environment}& Ant  & Hopper & Inverted Pendulum & Reacher & Walker2d\\\hline
\multicolumn{2}{c|}{$\epsilon$}& 0.2 & 0.075 & 0.3 & 1.5 & 0.05 \\\hline
\multicolumn{2}{c|}{State Space}& 111 & 11 & 4 & 11 & 17 \\\hline

\multirow{7}{*}{\parbox{2cm}{\centering Vanilla\\DDPG}}&Natural Reward &$1487 \pm 850$ & $3302 \pm 762$ & $1000 \pm 0$ & $-4.37 \pm 1.54$ & $1870 \pm 1418$  \\
&Critic Attack &$187 \pm 157$ & $2504 \pm 1207$ & $1000 \pm 0$ & $-24.35 \pm 5.10$ & $1301 \pm 1229$ \\

&Random Attack & $1473 \pm 795$ & $3086 \pm 1006$ & $1000 \pm 0$ & $-8.71 \pm 2.42$ & $1828 \pm 1456$ \\
&MAD Attack &  $180 \pm 200$ & $2745 \pm 1073$ & $1000 \pm 0$ & $-27.67 \pm 5.32$ & $1564 \pm 1405$ \\
&RS Attack & $336\pm 283$ & $606 \pm 124$ & $92 \pm 1$ & $-21.74 \pm 5.14$ & $959 \pm 1001$ \\
&RS+MAD &$142 \pm 180$ & $2056 \pm 1225$ & $1000 \pm 0$ & $-27.87 \pm 4.38$ & $790 \pm 985$ \\
&Best Attack & 142  & 606  & 92  &-27.87 & 790\\\hline

\multirow{7}{*}{\parbox{2cm}{\centering DDPG with\\adv. training\\(50\% steps)\\\citet{pattanaik2018robust}}}&Natural Reward &$1522 \pm 831$ & $2694 \pm 497$ & $1000 \pm 0$ & $-5.20 \pm 1.70$ & $1818 \pm 1187$ \\
&Critic Attack & $222 \pm 299$ & $1789 \pm 1143$ & $703 \pm 373$ & $-23.88 \pm 5.05$ & $1391 \pm 1083$ \\
&Random Attack &$1389 \pm 785$ & $2316 \pm 741$ & $1000 \pm 0$ & $-9.09 \pm 2.42$ & $1793 \pm 955$ \\
&MAD Attack & $92 \pm 240$ & $1497 \pm 839$ & $238 \pm 240$ & $-25.81 \pm 6.53$ & $1680 \pm 1106$ \\
&RS Attack &$129 \pm 156$ & $41 \pm 105$ & $39 \pm 0$ & $-25.45 \pm 6.70$ & $837 \pm 722$ \\
&RS+MAD &$31 \pm 179$ & $1503 \pm 851$ & $116 \pm 90$ & $-25.81 \pm 6.53$ & $1120 \pm 859$ \\
&Best Attack & 31 & 41 & 39  & -25.81  & 837   \\\hline

\multirow{7}{*}{\parbox{2cm}{\centering DDPG with\\adv. training\\(100\% steps)\\\citet{pattanaik2018robust}}}&Natural Reward & $1082 \pm 574$ & $973 \pm 0$ & $1000 \pm 0$ & $-5.71 \pm 1.80$ & $462 \pm 569$ \\
&Critic Attack & $126 \pm 148$ & $62 \pm 34$ & $174 \pm 66$ & $-21.91 \pm 3.52$ & $809 \pm 525$ \\
&Random Attack & $832 \pm 545$ & $577 \pm 431$ & $998 \pm 5$ & $-9.60 \pm 2.56$ & $751 \pm 568$ \\
&MAD Attack &$43 \pm 165$ & $56 \pm 50$ & $121 \pm 19$ & $-26.47 \pm 4.19$ & $699 \pm 484$ \\
&RS Attack & $115 \pm 286$ & $24 \pm 15$ & $82 \pm 0$ & $-22.17 \pm 4.46$ & $302 \pm 260$ \\
&RS+MAD & $-52 \pm 231$ & $56 \pm 50$ & $110 \pm 26$ & $-27.44 \pm 4.05$ & $488 \pm 406$ \\
&Best Attack & $-52 $ & $24$ & $82 $ & $-27.44$ & $302 $ \\\hline

\multirow{7}{*}{\parbox{2cm}{\centering SA-DDPG\\solved by\\SGLD}}&Natural Reward & $2186 \pm 534$ & $3068 \pm 223$ & $1000 \pm 0$ & $-5 \pm 1$ & $3318 \pm 680$ \\
&Critic Attack & $2076 \pm 556$ & $2899 \pm 439$ & $423 \pm 281$ & $-12.10 \pm 4.58$ & $1210 \pm 979$ \\
&Random Attack & $2162 \pm 524$ & $3071 \pm 196$ & $1000 \pm 0$ & $-11.41 \pm 4.96$ & $3058 \pm 848$ \\
&MAD Attack & $2128 \pm 482$ & $3093 \pm 17$ & $733 \pm 284$ & $-11.94 \pm 4.79$ & $3252 \pm 689$  \\
&RS Attack & $2038 \pm 401$ & $1729 \pm 792$ & $832 \pm 328$ & $-11.69 \pm 4.80$ & $2224 \pm 1050$  \\
&RS+MAD & $2007\pm 686$ & $1609 \pm 676$ & $724 \pm 322$ & $-12.01 \pm 4.84$ & $1933 \pm 1055$ \\
&Best Attack & $\bf 2007$  & $\bf 1609 $ & $423$ & $\bf-12.10$ & $1210$ \\\hline

\multirow{7}{*}{\parbox{2cm}{\centering SA-DDPG\\solved by\\convex relaxations}}&Natural Reward & $2254 \pm 430$ & $3128 \pm 453$ & $1000 \pm 0$ & $-5.24 \pm 2.06$ & $4540 \pm 1562$  \\
&Critic Attack &$1826 \pm 568$ & $2546 \pm 843$ & $1000 \pm 0$ & $-11.51 \pm 3.80$ & $2245 \pm 1881$  \\
&Random Attack &$2249 \pm 491$ & $3036 \pm 593$ & $1000 \pm 0$ & $-9.87 \pm 3.95$ & $4216 \pm 1616$  \\
&MAD Attack & $2106 \pm 573$ & $2959 \pm 663$ & $1000 \pm 0$ & $-12.43 \pm 3.76$ & $4135 \pm 1884$  \\
&RS Attack & $1820 \pm 635$ & $1258 \pm 561$ & $1000 \pm 0$ & $-11.40 \pm 3.56$ & $1986 \pm 1993$  \\
&RS+MAD & $2005 \pm 699$ & $1202 \pm 402$ & $1000 \pm 0$ & $-12.44 \pm 3.77$ & $2315 \pm 2127$ \\
&Best Attack &$1820$ & $1202$ & $ 1000$ & $-12.44$ &  $\bf 1986$ \\\hline
\bottomrule
\end{tabular}
}
\caption{Average episode rewards on 5 MuJoCo environments using policies trained by DDPG and SA-DDPG. Natural reward is the reward in clean environment without adversarial attacks. The ``Best Attack'' rows report the lowest reward over all five attacks (representing the strongest attack), and this lowest reward is used for robustness evaluation.}
\label{tab:ddpg_full_res_}
\end{table*}

\subsection{Robustness Certificates}

We report robustness certificates for SA-DQN in Table~\ref{tab:dqn_res}. As discussed in section~\ref{sec:certificate}, for DQN we can guarantee that an action does not change under bounded adversarial noise. In Table~\ref{tab:dqn_res}, the ``Action Cert. Rate'' is the ratio of actions that does not change under any $\ell_\infty$ norm bounded noise. In some settings, we find that 100\% of the actions are guaranteed to be unchanged (e.g., the Pong environment in Table~\ref{tab:dqn_res}). In that case, we can in fact also certify that the cumulative reward is not changed given the specific initial conditions for testing.

In SA-DDPG, we can obtain robustness certificates that give bounds on actions in the presence of bounded perturbation on state inputs. Given an input state $s$, we use convex relaxations of neural networks to obtain the upper and lower bounds for each action: $l_i(s) \leq \pi_i(\hat{s}) \leq u_i(s), \forall \hat{s} \in B(s)$. We consider the following certificates on $\pi(s)$: the average output range $\frac{\|u(s) - l(s)\|_1}{|\mathcal{A}|}$ which reflect the tightness of bounds, and the $\ell_2$ distance. 
Note that bounds on other $\ell_p$ norms can also be computed given $l_i(s)$ and $u_i(s)$. Since the action space is normalized within $[-1,1]$, the worst case output range is 2. We report both certificates for all five environments in Table~\ref{tab:new_ddpg_certificate}. DDPG without our robust regularizer usually cannot obtain non-vacuous certificates (range is close to 2). SA-DDPG can provide robustness certificates (bounded inputs guarantee bounded outputs). We include some discussions on these certificates in Section~\ref{sec:certificate}.

For SA-PPO, since the action follows a Gaussian policy, we can upper bound its KL-divergence under state perturbations. The results are shown in Table~\ref{tab:ppo_certificate}. Note that, by increasing the regularization parameter $\kappa$, it is possible to obtain an even tighter certificate at the cost of model performance.

The robustness certificates for SA-DDPG and SA-PPO are computed using interval bound propagation (IBP). For vanilla DDPG and PPO, we use CROWN~\citep{zhang2018efficient}, a much tighter convex relaxation to obtain the certificates, but they are often still vacuous.

\begin{table}[htbp]
\caption{Robustness certificates on bounded action changes under bounded state perturbations for DDPG agents. Results are averaged over 50 episodes. A smaller number is better. A vanilla DDPG agent typically cannot provide non-vacuous robustness guarantees.}
\label{tab:new_ddpg_certificate}
\resizebox{\linewidth}{!}{
\begin{tabular}{|c|c|c|c|c|c|c|}
\hline
\multicolumn{2}{|c|}{Settings}                                            & Ant    & Hopper & InvertedPendulum & Reacher & Walker2d \\ \hline
\multirow{2}{*}{Certificates ($\ell_2$ upper bound)}      & SA-DDPG (Convex) &0.181 & 0.050 & 0.787 & 0.202 & 0.169     \\ \cline{2-7} 
                                              & DDPG (vanilla)         & 3.972 & 2.612 & 0.992 & 1.491 & 2.484 \\ \hline
\multirow{2}{*}{Certificates ($\ell_1$ upper bound)}      & SA-DDPG (Convex) &0.454 & 0.074 & 0.787 & 0.283 & 0.301    \\ \cline{2-7} 
                                              & DDPG (vanilla)         &11.087 & 4.345 & 0.992 & 2.107 & 4.923   \\ \hline
\multirow{2}{*}{Certificates ($\ell_\infty$ upper bound)} & SA-DDPG (Convex) & 0.104 & 0.041 & 0.787 & 0.157 & 0.131     \\ \cline{2-7} 
                                              & DDPG (vanilla)         &1.734 & 1.794 & 0.992 & 1.073 & 1.570    \\ \hline
\multirow{2}{*}{Certificates (Range)}         & SA-DDPG (Convex) & 0.057 & 0.025 & 0.787 & 0.142 & 0.050     \\ \cline{2-7} 
                                              & DDPG (vanilla)         & 1.386 & 1.448 & 0.992 & 1.054 & 0.821 \\ \hline
\end{tabular}}
\end{table}

\begin{table}[htbp]
\caption{Upper bound on KL-divergence $\mathrm{D_{KL}}(\pi(a|s)\|\pi(a|\hat{s}))$ for three PPO environments. A smaller number is better. SA-PPO can reduce this upper bound significantly especially for high dimensional environments like Humanoid.}
\label{tab:ppo_certificate}
\centering
\resizebox{0.8\linewidth}{!}{
\begin{tabular}{|c|c|c|c|c|}
\hline
\multicolumn{2}{|c|}{Settings}                                & Hopper & Walker2d & Humanoid \\ \hline
\multirow{2}{*}{Certificates (KL upper bound)} & SA-PPO (Convex) & 0.1232 & 0.09831   & 3.529   \\ \cline{2-5} 
                                   & PPO (vanilla)      & 32.16 & 31.56   & 925140   \\ \hline
\end{tabular}
}
\end{table}

\end{document}